%% file: main.tex
\documentclass[twoside,10pt]{article}

\usepackage{jmlr2e}

\jmlrheading{1}{2000}{1-48}{4/00}{10/00}{Chuanhao Li and Hongning Wang}

\ShortHeadings{Asynchronous Algorithms for Federated Linear Bandit}{Li and Wang}
\firstpageno{1}

\usepackage[utf8]{inputenc} 
\usepackage[T1]{fontenc}    
\usepackage{hyperref}       
\usepackage{url}            
\usepackage{booktabs}       
\usepackage{amsfonts}       
\usepackage{nicefrac}       
\usepackage{microtype}      
\usepackage{graphicx}
\usepackage{algorithmic}
\usepackage{algorithm}
\usepackage{wrapfig}

\usepackage{amsmath}
\usepackage{subfigure}
\usepackage{multirow}
\usepackage[section]{placeins}
\usepackage{enumitem}
\usepackage{xcolor}
\newtheorem{assumption}{Assumption}
\DeclareMathOperator*{\argmax}{arg\,max}
\DeclareMathOperator*{\argmin}{arg\,min}

\def \cF {\mathcal{F}}

\def \bx {\mathbf{x}}
\def \bX {\mathbf{X}}

\def \by {\mathbf{y}}

\def \cA {\mathcal{A}}

\def \bR {\mathbb{R}}
\def \bE {\mathbb{E}}
\def \bbE {\mathbf{E}}
\def \bV {\mathbb{V}}

\def \cB {\mathcal{B}}
\def \cS {\mathcal{S}}

\newcommand{\algone}{{FedGLB-UCB}}

\begin{document}

\title{Communication Efficient Federated Learning for \\Generalized Linear Bandits}

\author{\name Chuanhao Li \email cl5ev@virginia.edu \\
       \addr Department of Computer Science\\
       University of Virginia\\
       Charlottesville, VA 22903, USA\\
       \AND
       \name Hongning Wang \email hw5x@virginia.edu \\
       \addr Department of Computer Science\\
       University of Virginia\\
       Charlottesville, VA 22903, USA
       }


\maketitle

\begin{abstract}
Contextual bandit algorithms have been recently studied under the federated learning setting to satisfy the demand of keeping data decentralized and pushing the learning of bandit models to the client side. But limited by the required communication efficiency, existing solutions are restricted to linear models to exploit their closed-form solutions for parameter estimation. Such a restricted model choice greatly hampers these algorithms' practical utility. 
In this paper, we take the first step to addressing this challenge by studying generalized linear bandit models under the federated learning setting. We propose a communication-efficient solution framework that employs online regression for local update and offline regression for global update. We rigorously proved, though the setting is more general and challenging, our algorithm can attain sub-linear rate in both regret and communication cost, which is also validated by our extensive empirical evaluations.
\end{abstract}


\input{intro}
\input{related_work}
\input{preliminaries}
\input{method}
\input{exp}

\section{Conclusion}
In this paper, we take the first step to address the new challenges in communication efficient federated bandit learning beyond linear models, where closed-form solutions do not exist, and propose a solution framework for federated GLB that employs online regression for local update and offline regression for global update. For arm selection, we propose a novel confidence ellipsoid construction based on the sequence of \emph{offline-and-online} model estimations. We rigorously prove that the proposed algorithm attains sub-linear rate for both regret and communication cost,
and also analyze the impact of each component of our algorithm via theoretical comparison with different variants.
In addition, extensive empirical evaluations are performed to validate the effectiveness of our algorithm.

An important further direction of this work is the lower bound analysis for the communication cost, 
analogous to the communication lower bound for standard distributed optimization by Arjevani and Shamir \cite{arjevani2015communication}.
Moreover, in our algorithm, 
clients' locally updated models 
are not utilized for global model update,
so that another interesting direction is to investigate whether using such knowledge, e.g., by model aggregation, can further improve communication efficiency.

\section*{Acknowledgement}
This work is supported by NSF grants IIS-2213700, IIS-2128019 and IIS-1838615.

\bibliography{bibfile}

\clearpage
\appendix

\input{appendix}

\end{document}

%% file: intro.tex
\section{Introduction}
As a classic model for sequential decision making problems, contextual bandit has been widely used for a variety of real-world applications, including recommender systems \citep{li2010contextual}, display advertisement \citep{li2010exploitation} and clinical trials \citep{durand2018contextual}. While most existing bandit solutions are designed under a centralized setting (i.e., data is readily available at a central server), in response to the increasing application scale and public concerns of privacy, there is 
increasing research effort on federated bandit learning lately \citep{wang2019distributed,dubey2020differentially,shi2021federated,huang2021federated,li2022asynchronous},
where $N$ clients collaborate with limited communication bandwidth to minimize the overall cumulative regret incurred over a finite time horizon $T$, while keeping each client's raw data local.
Compared with standard federated learning \citep{mcmahan2017communication,kairouz2019advances} that works with fixed datasets, federated bandit learning is characterized by its online interactions with the environment, which continuously provides new data samples to the clients over time. 
This brings in new challenges in addressing the conflict between the need of timely data/model aggregation for regret minimization and the need of communication efficiency with decentralized data. 
A carefully designed 
model update method and communication strategy become vital to strike this balance.



Existing federated bandit learning solutions only partially addressed this challenge by considering simple bandit models, like context-free bandit \citep{shi2021federated} and contextual linear bandit \citep{wang2019distributed,dubey2020differentially,li2022asynchronous}, 
where closed-form solution for both local and global model update exists. Therefore, efficient communication for global bandit model update is realized by directly aggregating local sufficient statistics,
such that the only concern left is how to control the communication frequency over time horizon $T$. 
However, such a solution framework does not apply to the more complicated bandit models that are often preferred in practice, such as generalized linear bandit (GLB) \citep{filippi2010parametric} or neural bandit \citep{zhou2020neural}, where only iterative solutions exist for parameter estimation (e.g., gradient-based optimization). 
To enable joint model estimation, now the learning system needs to solve distributed optimization for multiple times as new data is collected from the environment, and each requires iterative gradient/model aggregation among clients.
This is much more expensive compared with linear models, and it naturally leads to the question: whether a communication efficient solution to this challenging problem is still possible?

In this paper, we answer this question affirmatively by proposing the first provably communication efficient algorithm for federated GLB that only requires $\tilde{O}(\sqrt{T})$ communication cost, while still attaining the optimal order of regret.
Our proposed algorithm employs a combination of online and offline regression, with online regression adjusting each client's model using its newly collected data, and offline (distributed) regression occasionally soliciting local gradients from all $N$ clients for joint model estimation when sufficient amount of new data has been accumulated.
In order to balance exploration and exploitation in arm selection, we propose a novel way to construct the confidence set based on the sequence of \textit{offline-and-online} model updates that each client has received. The initialization of online regression with offline regression introduces dependencies that break the standard martingale argument, which requires proof techniques unique to this paper.

We also explored other non-trivial solution ideas to further justify our current design. Specifically, in practice, a common way to update the deployed model for applications with streaming data is to set a schedule and periodically re-train the model using iterative optimization methods. For comparison, we propose and rigorously analyze a federated GLB algorithm designed based on this idea, as well as a variant that further enables online updates on the clients. We also consider another solution idea motivated by distributed/batched online convex optimization, which is characterized by lazy online updates over batches of data. Moreover, extensive empirical evaluations on both synthetic and real-world datasets are performed to validate the effectiveness of our algorithm.

%% file: related_work.tex
\section{Related Work}
GLB, as an important extension of linear bandit models, has demonstrated encouraging performance in modeling binary rewards (such as clicks) that are ubiquitous in real-world applications \citep{li2012unbiased}. The study of GLB under a centralized setting dates back to Filippi et al. \cite{filippi2010parametric}, who proposed a UCB-type algorithm that achieved $\tilde{O}(d\sqrt{T})$ regret. Li et al. \cite{li2017provably} later proposed two improvements: a similar UCB-type algorithm that improves the result of \citep{filippi2010parametric} by a factor of $O(\log{T})$, which has been popularly used in practice as it avoids the projection step needed in \citep{filippi2010parametric}; and another impractical algorithm that further improves the result by a factor of $O(\sqrt{d})$ assuming fixed number of arms.
To improve the time and space complexity of the aforementioned GLB algorithms, followup works adopted online regression methods.
In particular, motivated by the online-to-confidence-set conversion technique from \cite{abbasi2012online}, Jun et al. \cite{jun2017scalable} proposed both UCB and Thompsan sampling algorithms with online Newton step, and Ding et al. \cite{ding2021efficient} proposed a Thompson sampling algorithm with online gradient descent, which, however, requires an additional context regularity assumption to obtain a sub-linear regret.



GLB under federated/distributed setting still remains under-explored. The most related works are the federated/distributed linear bandits \citep{korda2016distributed,wang2019distributed,dubey2020differentially,huang2021federated,li2022asynchronous}.
In these works, thanks to the existence of closed-form solution for linear models, the clients only communicate their local sufficient statistics for global model update. 
Korda et al. \cite{korda2016distributed} considered a peer-to-peer (P2P) communication network and assumed the clients form clusters, i.e., each cluster is associated with a unique bandit problem. But as they only focused on reducing \textit{per-round} communication, the communication cost is still linear over time. 
Huang et al. \cite{huang2021federated} considered a star-shaped communication network as in our paper, but their proposed phase-based elimination algorithm only works in fixed arm set setting.
The closest works to ours are \citep{wang2019distributed,dubey2020differentially,li2022asynchronous}, which uses event-triggered communication protocols to obtain sub-linear communication cost over time for federated linear bandit with a time-varying arm set.
Another related line of research is the standard federated learning that considers offline supervised learning problems \citep{kairouz2019advances}.
Since its debut in \citep{mcmahan2017communication}, FedAvg has become the most popularly used algorithm for offline federated learning. However, despite its popularity, several works \citep{li2019convergence,karimireddy2020scaffold,mitra2021achieving} identified that FedAvg suffers from a \emph{client-drift} problem when the clients' data are non-IID (which is an important signature of our case), i.e., local iterates in each client drift towards their local minimum. This leads to a sub-optimal convergence rate of FedAvg: for example, one has to suffer a sub-linear convergence rate for strongly convex and smooth losses, though a linear convergence rate is expected under a centralized setting. To alleviate this, Pathak and Wainwright \cite{pathak2020fedsplit} proposed an operator splitting procedure to guarantee linear convergence to a neighborhood of the global minimum. Later,
Mitra et al. \cite{mitra2021achieving} introduced variance reduction techniques to guarantee exact linear convergence to the global minimum.



%% file: preliminaries.tex
\section{Preliminaries} \label{sec:problem_formulation}
In this section, we first introduce the general problem formulation of federated bandit learning, and discuss the existing solutions under the linear reward assumption. Then we formulate the federated GLB problem considered in this paper, followed by detailed discussions about the new challenges compared with its linear counterpart.

\subsection{Federated Bandit Learning}
Consider a learning system with 1) $N$ clients responsible for taking actions and receiving corresponding reward feedback from the environment, e.g., each client being an edge device directly interacting with a user, and 2) a central server responsible for coordinating the communication between the clients for joint model estimation.

At each time step $t=1,2,...,T$, all $N$ clients interact with the environment in a round-robin manner, i.e., each client $i \in [N]$ chooses an arm $\bx_{t,i}$ from its time-varying candidate set $\cA_{t,i}=\{\bx_{t,i}^{(1)}, \bx_{t,i}^{(2)}, \dots, \bx_{t,i}^{(K)}\}$, where $\bx_{t,i}^{(a)} \in \bR^{d}$ denotes the context vector associated with the $a$-th arm for client $i$ at time $t$. 
Without loss of generality, we assume $||\bx_{t,i}^{(a)}||_{2}\leq 1,\forall i,a,t$. 
Then client $i$ receives the corresponding reward $y_{t,i} \in \bR$ from the environment, which is drawn from the reward distribution governed by an unknown parameter $\theta_{\star} \in \bR^{d}$ (assume $\lVert \theta_{\star} \rVert \leq S$), i.e., $y_{t,i} \sim p_{\theta_{\star}}(y|\bx^{(a)}_{t,i})$.
The interaction between the learning system and the environment repeats itself, and the goal of the learning system is to minimize the cumulative (pseudo) regret over all $N$ clients in the finite time horizon $T$, i.e., $R_{T}=\sum_{t=1}^{T} \sum_{i=1}^{N}r_{t,i}$, where $r_{t,i}=\max_{\bx \in \cA_{t,i}} \bbE[y|\bx]- \bbE[y_{t,i}|\bx_{t,i}]$.

In a federated learning setting, the clients cannot directly communicate with each other, but through the central server, i.e., a star-shaped communication network. Raw data collected by each client $i \in [N]$, i.e., $\{(\bx_{s,i},y_{s,i})\}_{s \in [T]}$, is stored locally and cannot be shared with anyone else. Instead, the clients can only communicate the parameters of the learning algorithm, e.g., models, gradients, or sufficient statistics; and the communication cost is measured by 
the total number of times data being transferred across the system up to time $T$, which is denoted as $C_{T}$.


\subsection{Federated Linear Bandit} \label{subsec:federated_linear}
Prior works have studied communication-efficient federated linear bandit \citep{wang2019distributed,dubey2020differentially}, i.e., the reward function is a linear model $y_{t,i}=\bx_{t,i}^{\top}\theta_{\star} +\eta_{t,i}$, where $\eta_{t,i}$ denotes zero-mean sub-Gaussian noise.
Consider an imaginary centralized agent that has direct access to the data of all clients, so that it can compute the global sufficient statistics $A_{t}=\sum_{i \in [N]} \sum_{s \in [t]} \bx_{s,i} \bx_{s,i}^{\top},b_{t}=\sum_{i \in [N]}\sum_{s\in [t]}\bx_{s,i}y_{s,i}$. 
Then the cumulative regret incurred by this distributed learning system can match that under a centralized setting, if all $N$ clients select arms based on the global sufficient statistics $\{A_{t},b_{t}\}$. 
However, it requires $N^{2}T$ communication cost for the immediate sharing of each client's update to the sufficient statistics
with all other clients, which is expensive for most applications.

To ensure communication efficiency, prior works like DisLinUCB \cite{wang2019distributed} let each client $i$ maintain a local copy $\{A_{t-1,i},b_{t-1,i}\}$ for arm selection, which receives immediate local update using each newly collected data sample, i.e., $A_{t,i}=A_{t-1,i}+\bx_{t,i} \bx_{t,i}^{\top}, b_{t,i}=b_{t-1,i}+\bx_{t,i} y_{t,i}$. Then client $i$ checks whether the event $(t-t_\text{last}) \log(\frac{\det A_{t,i}}{\det A_{t_\text{last}}})>D$ is true, where $t_\text{last}$ denotes the time step of last global update. If true, a new global update is triggered, such that the server will collect all clients' local update since $t_\text{last}$, aggregate them to compute $\{A_{t},b_{t}\}$, and then synchronize the local sufficient statistics of all clients, i.e., set $\{A_{t,i},b_{t,i}\}=\{A_{t},b_{t}\},\forall i \in [N]$.


\subsection{Federated Generalized Linear Bandit} \label{subsec:federated_generalized_linear}
In this paper, we study federated bandit learning with generalized linear models, i.e., the conditional distribution of reward $y$ given context vector $\bx$ is drawn from the exponential family \citep{filippi2010parametric,li2017provably}:
\begin{equation} \label{eq:glm_model}
    p_{\theta_{\star}}(y|\bx) = \exp{\left( \frac{y \bx^{\top} \theta_{\star} - m(\bx^{\top} \theta_{\star})}{g(\tau)} + h(y,\tau) \right)}
\end{equation}
where $\tau \in \bR^{+}$ is a known scale parameter.
Given a function $f:\bR \rightarrow \bR$, we denote its first and second derivatives by $\dot{f}$ and $\ddot{f}$, respectively. 
It is known that $\dot{m}(\bx^{\top} \theta_{\star})=\bE[y|\bx]:=\mu(\bx^{\top}\theta_{\star})$, which is called the inverse link function, and $\ddot{m}(\bx^{\top}\theta_{\star}) =\bV(y|\bx^{\top}\theta_{\star})$.
Based on Eq.\eqref{eq:glm_model}, the reward $y_{t,i}$ observed by client $i$ at time $t$ can be equivalently represented as $y_{t,i}=\mu(\bx_{t,i}^{\top}\theta_{\star}) + \eta_{t,i}$, where $\eta_{t,i}$ denotes the sub-Gaussian noise. 
Then we denote the negative log-likelihood of $y_{i,t}$ given $\bx_{i,t}$ as $l(\bx_{t,i}^{\top}\theta_{\star},y_{t,i})= -\log{p_{\theta_{\star}}(y_{t,i}|\bx_{t,i})}=-y_{t,i} \bx_{t,i}^{\top}\theta_{\star} + m(\bx_{t,i}^{\top}\theta_{\star})$. 
In addition, we adopt the following two assumptions about the reward, which are standard for GLB \citep{filippi2010parametric}.
\begin{assumption}\label{assump:1}
The link function $\mu$ is continuously differentiable on $(-S,S)$, $k_{\mu}$-Lipschitz on $[-S,S]$, and $\inf_{z \in [-S,S]} \dot{\mu}(z) = c_{\mu} >0$.
\end{assumption}
\begin{assumption}\label{assump:2}
$\bE[\eta_{t,i}|\cF_{t,i}]=0, \forall t,i$, where $\cF_{t,i}=\sigma\{ \bx_{t,i}, [\bx_{s,j}, y_{s,j}]_{(s,j):s < t \cap j = i} \}$ denotes the $\sigma$-algebra generated by client $i$'s previously pulled arms and observed rewards, 
and $\max_{t , i} |\eta_{t,i}| \leq R_{\max}$ for some constant $R_{\max} > 0$.
\end{assumption}

\paragraph{New Challenges} Compared with federated linear bandit discussed in Section \ref{subsec:federated_linear}, new challenges arise in designing a communication-efficient algorithm for federated GLB due to the absence of a closed form solution:
\begin{itemize}[nosep]
    \item \emph{Iterative communication for global update:} compared with the global update for federated linear bandit that only requires one round of communication to share the sufficient statistics, now it takes multiple iterations of gradient aggregation to obtain converged global optimization.
    Moreover, as the clients collect more data samples over time during bandit learning, the required number of iterations for convergence also increases.
    \item \emph{Drifting issue with local update:} 
    during local model update, 
    iterative optimization using only local gradient can push the updated model away from the global model, i.e., forget the knowledge gained during previous communications \cite{kirkpatrick2017overcoming}.
\end{itemize}

\begin{figure}[t]
\centering
\includegraphics[width=0.9\textwidth]{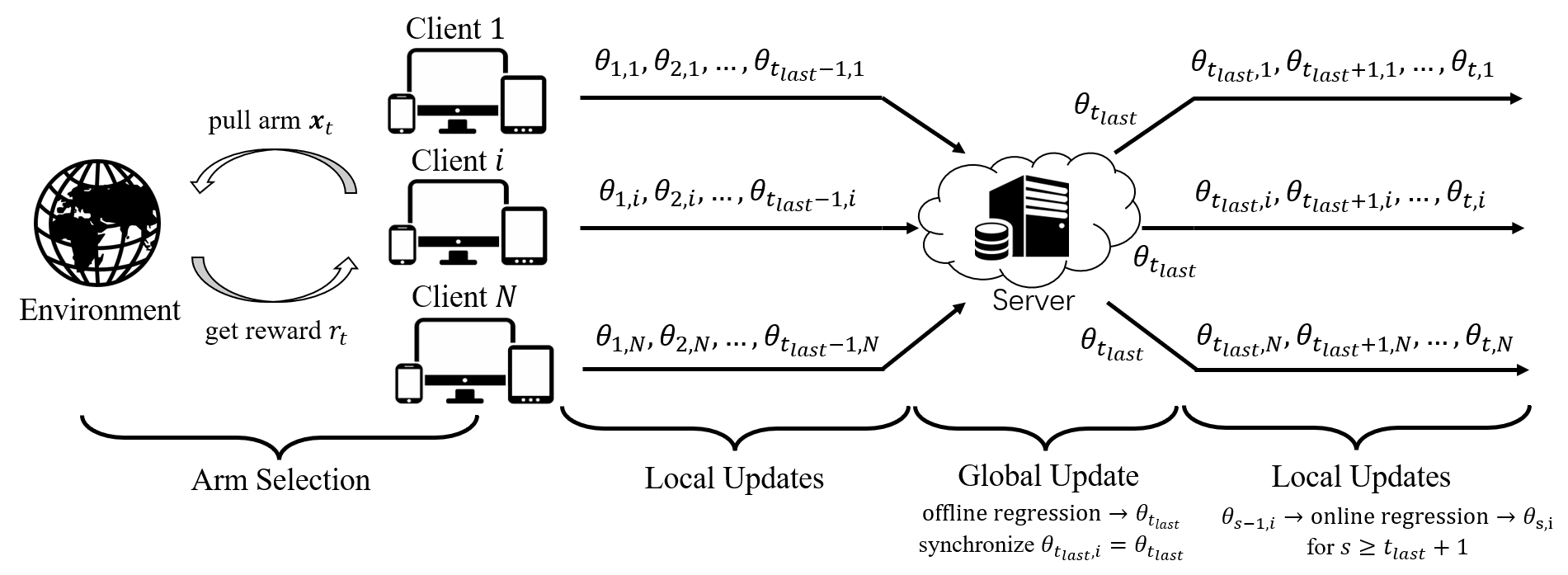}
\caption{Illustration of \algone{} algorithm,
which uses online regression for local update, i.e., immediately update each client's local model $\theta_{t, i}$ using its newly collected data sample, and uses offline regression for global update, i.e., synchronize all $N$ clients to a globally updated model $\theta_{t_\text{last}}$ using all the data samples collected so far.
}
\label{fig:fedglb_illustration}
\end{figure}

%% file: method.tex
\section{Methodology}
In this section we propose the first algorithm for federated GLB that addresses the aforementioned challenges. 
We rigorously prove that it attains sub-linear rate in $T$ for both regret and communication cost. In addition, we propose and analyze different variants of our algorithm to facilitate understanding of our algorithm design.

\subsection{\algone{} Algorithm} \label{subsec:algo_description}
To ensure communication-efficient model updates for federated GLB, 
we propose to use online regression for local update, i.e., update each client's local model only with its newly collected data samples, and use offline regression for global update, i.e., solicit all clients' local gradients for joint model estimation. Based on the resulting sequence of \emph{offline-and-online} model updates, the confidence ellipsoid for $\theta_{\star}$ is constructed for each client to select arms using the OFUL principle. We name this algorithm Federated Generalized Linear Bandit with Upper Confidence Bound, or \algone{} for short. We illustrate its key components in Figure \ref{fig:fedglb_illustration} and describe its procedures in Algorithm \ref{algo:3}. In the following, we discuss about each component of \algone{} in details.

\begin{algorithm}[h]
    \caption{\algone{}} \label{algo:3}
  \begin{algorithmic}[1]
    \STATE \textbf{Input:} threshold $D$, regularization parameter $\lambda>0$, $\delta \in (0,1)$ and $c_{\mu}$.
    \STATE \textbf{Initialize} $\forall i\in[N]$: ${A}_{0,i}=\frac{\lambda}{c_{\mu}} \textbf{I} \in \bR^{d \times d}, b_{0,i}=\textbf{0} \in \bR^{d} ,\theta_{0,i}=\textbf{0} \in \bR^{d}, \Delta {A}_{0,i}= \textbf{0} \in \bR^{d \times d}$;
    ${A}_{0}=\frac{\lambda}{c_{\mu}} \textbf{I} \in \bR^{d \times d}, b_{0}=\textbf{0} \in \bR^{d}, \theta_{0}= \textbf{0} \in \bR^{d}$, $t_{\text{last}}=0$
    \FOR{ $t=1,2,...,T$}
        \FOR{client $i=1,2,...,N$}
            \STATE Observe arm set $\mathcal{A}_{t,i}$ for client $i$
            \STATE Select arm $\bx_{t,i}\in\mathcal{A}_{t,i}$ by Eq.\eqref{eq:UCB}, and observe reward $y_{t,i}$
            \STATE Update client $i$: ${A}_{t,i} = {A}_{t-1,i} + \bx_{t,i} \bx_{t,i}^{\top}$, $\Delta {A}_{t,i} = \Delta {A}_{t-1,i} + \bx_{t,i} \bx_{t,i}^{\top}$ 
            \IF{$(t-t_{\text{last}})\log{\frac{\det(A_{t,i})}{\det(A_{t,i}-\Delta A_{t,i})}} < D$}
                \STATE \textbf{Client} $i$: perform local update $\theta_{t,i}=\text{ONS-Update}(\theta_{t-1,i},A_{t,i},\nabla l(\bx_{t,i}^{\top} \theta_{t-1,i},y_{t,i}))$, $b_{t,i}=b_{t-1,i}+ \bx_{t,i}\bx_{t,i}^{\top}\theta_{t-1,i}$
            \ELSE
                \STATE \textbf{Clients} $\forall i \in [N]$: send $\Delta A_{t,i}$ to server, and reset $\Delta A_{t,i}=\textbf{0}$
                \STATE \textbf{Server}: compute $A_{t}=A_{t_{\text{last}}}+\sum_{i=1}^{N}\Delta A_{t,i}$
                \STATE \textbf{Server}: perform global update $\theta_{t}=\text{AGD-Update}(\theta_{t_{\text{last}}},J_{t})$ (see Eq.\eqref{eq:AGD_J_upper_bound} for the choice of $J_{t}$), $b_{t}=b_{t_{\text{last}}}+\sum_{i=1}^{N}\Delta A_{t,i}\theta_{t}$, and set $t_{\text{last}}=t$
                \STATE \textbf{Clients} $\forall i \in [N]$: set $\theta_{t,i}=\theta_{t}, A_{t,i}=A_{t}, b_{t,i}=b_{t}$
            \ENDIF
        \ENDFOR
    \ENDFOR
  \end{algorithmic}
\end{algorithm}


\noindent\textbf{$\bullet$ Local update.}
As mentioned earlier, iterative optimization over local dataset $\{(\bx_{s,i},y_{s,i})\}_{s \in [t]}$ leads to the drifting issue that pushes the updated model to the local optimum. Due to the small size of this local dataset, the confidence ellipsoid centered at the converged model has increased width,
which leads to increased regret in bandit learning. 
However, as we will prove in Section \ref{subsec:theorectical_analysis}, completely disabling local update and restricting all clients to use the previous globally updated model for arm selection is also a bad choice, because the learning system will then need more frequent global updates to adapt to the growing dataset.

To enable local update while alleviating the drifting issue, we adopt online regression in each client, such that the local model estimation $\theta_{t,i}$ is only updated for one step using the sample $(\bx_{t,i},y_{t,i})$ collected at time $t$.
Prior works \cite{abbasi2012online,jun2017scalable} showed that UCB-type algorithms with online regression can attain comparable cumulative regret to the standard UCB-type algorithms \citep{abbasi2011improved,li2017provably}, as long as the selected online regression method guarantees logarithmic online regret. 
As the negative log-likelihood loss defined in Section \ref{subsec:federated_generalized_linear} is exp-concave and online Newton step (ONS) is known to attain logarithmic online regret in this case \citep{hazan2007logarithmic,jun2017scalable}, ONS is chosen for the local update of \algone{} and its description is given in 
Algorithm \ref{alg:ons_update}.
At time step $t$, after client $i$ pulls an arm $\bx_{t,i} \in \cA_{t,i}$ and observes the reward $y_{t,i}$, its model $\theta_{t-1,i}$ is immediately updated by the ONS update rule (line 9 in Algorithm \ref{algo:3}), 
where $\nabla l(\bx_{t,i}^{\top} \theta_{t-1,i},y_{t,i})$ denotes the gradient w.r.t. $\theta_{t-1,i}$, and $A_{t,i}$
denotes the covariance matrix for client $i$ at time $t$.

\begin{algorithm}[h]
    \caption{$\quad \text{ONS-Update}$} \label{alg:ons_update}
  \begin{algorithmic}[1]
    \STATE \textbf{Input:} $\theta_{t-1,i}, A_{t,i}, \nabla l(\bx_{t,i}^{\top} \theta_{t-1,i},y_{t,i})$
    \STATE $\theta_{t,i}^{\prime} = \theta_{t-1,i} - \frac{1}{c_{\mu}} A_{t,i}^{-1} \nabla l(\bx_{t,i}^{\top} \theta_{t-1,i},y_{t,i})$
    \STATE $\theta_{t,i} = \argmin_{\theta \in \cB_{d}(S)} ||\theta-\theta_{t,i}^{\prime}||^{2}_{A_{t,i}}$
    \STATE \textbf{Output:} $\theta_{t,i}$
  \end{algorithmic}
\end{algorithm}

\noindent\textbf{$\bullet$ Global update}
The global update of \algone{} requires communication among the $N$ clients,
which imposes communication cost in two aspects: 1) each global update for federated GLB requires multiple rounds of communication among $N$ clients, i.e., iterative aggregation of local gradients; and 2) global update needs to be performed for multiple times over time horizon $T$, in order to adapt to the growing dataset collected by each client during bandit learning.
Consider a particular time step $t\in [T]$ when global update happens, 
the distributed optimization objective is:
\begin{equation}\label{eq:objective}
    \min_{\theta \in \Theta} F_{t}(\theta) := \frac{1}{N} \sum_{i=1}^{N} F_{t,i}(\theta)
\end{equation}
where 
$F_{t,i}(\theta)=\frac{1}{t}\sum_{s=1}^{t} l(\bx_{s,i}^{\top}\theta,y_{s,i}) + \frac{\lambda}{2t} ||\theta||_{2}^{2}$ denotes the \textit{average} regularized negative log-likelihood loss for client $i \in [N]$, and $\lambda > 0$ denotes the regularization parameter. 
Based on Assumption \ref{assump:1}, $\{F_{t,i}(\theta)\}_{i \in [N]}$ are $\frac{\lambda}{Nt}$-strongly-convex and $(k_{\mu}+\frac{\lambda}{N t})$-smooth in $\theta$ (proof in Appendix \ref{sec:technical_lemmas}), and we denote the unique minimizer of Eq.\eqref{eq:objective} as $\hat{\theta}^{\text{MLE}}_{t}$.
In this case, it is known that
the number of communication rounds $J_{t}$ required to attain a specified sub-optimality $\epsilon_{t}$, such that $F_{t}(\theta)-\min_{\theta \in \Theta} F_{t}(\theta) \leq \epsilon_{t}$, has a lower bound 
$J_{t}=\Omega\left(\sqrt{(k_{\mu}Nt)/\lambda+1}\log{\frac{1}{\epsilon_{t}}}\right)$
\citep{arjevani2015communication}, which means $J_{t}$ increases at least at the rate of $\sqrt{Nt}$. 
This lower bound is matched by the distributed version of accelerated gradient descent (AGD) \citep{nesterov2003introductory}:
\small
\begin{equation} \label{eq:AGD_J_upper_bound}
\begin{split}
    J_{t} & \leq 1+\sqrt{(k_{\mu}Nt)/\lambda+1}\log{\frac{(k_{\mu}+\frac{2\lambda}{N t})\lVert \theta_{t}^{(1)}-\hat{\theta}_{t}^{\text{MLE}} \rVert_{2}^{2}}{2\epsilon_{t}}} 
\end{split}
\end{equation}
\normalsize
where the superscript $(i)$ denotes the $i$-th iteration of AGD.

In order to minimize the number of communication rounds in one global update, AGD is chosen as the offline regression method for \algone{}, and its description is given in 
Algorithm \ref{alg:agd_update} (subscript $t$ is omitted for simplicity). 
However, other federated/distributed optimization methods can be readily used in place of AGD,
as our analysis only requires the convergence result of the adopted method.
We should note that $\epsilon_{t}$ is essential to the regret-communication trade-off during the global update at time $t$: a larger $\epsilon_{t}$ leads to a wider confidence ellipsoid, which increases regret, while a smaller $\epsilon_{t}$ requires more communication rounds $J_{t}$, which increases communication cost. In Section \ref{subsec:theorectical_analysis}, we will discuss the proper choice of $\epsilon_{t}$ to attain desired trade-off between the two conflicting objectives.

\begin{algorithm}[h]
    \caption{$\text{AGD-Update}$ } \label{alg:agd_update}
  \begin{algorithmic}[1]
    \STATE \textbf{Input} : initial $\theta$, number of inner iterations $J$
    \STATE \textbf{Initialization}: set $\theta^{(1)}=\vartheta^{(1)}=\theta$, and define the sequences $\{\upsilon_{j}:=\frac{1+\sqrt{1+4 \upsilon_{j-1}^{2}}}{2}\}_{j\in[J]}$ (with $\upsilon_{0}=0$), and $\{\gamma_{j}=\frac{1-\upsilon_{j}}{\upsilon_{j+1}}\}_{j\in[J]}$
    \FOR{$j=1,2,\dots,J$}
        \STATE \textbf{Clients} compute and send local gradient $\{\nabla F_{i}(\theta^{(j)})\}_{i \in [N]}$ to the server
        \STATE \textbf{Server} aggregates local gradients $\nabla F(\theta^{(j)})=\frac{1}{N}\sum_{i=1}^{N}\nabla F_{i}(\theta^{(j)})$, and execute the following update rule to get $\theta^{(j+1)}$:
        \begin{itemize}[nosep]
            \item $\vartheta^{(j+1)} = \theta^{(j)} - \frac{1}{k_{\mu}+\frac{\lambda}{N t}} \nabla F(\theta^{(j)})$
            \item $\theta^{(j+1)} = (1-\gamma_{j}) \vartheta^{(j+1)} + \gamma_{j} \vartheta^{(j)}$
        \end{itemize}
    \ENDFOR
    \STATE \textbf{Output:} $\argmin_{\theta \in \cB_{d}(S)} \lVert g_{t}(\theta^{(J+1)}) -g_{t}(\theta) \rVert_{A_{t}^{-1}}$
  \end{algorithmic}
\end{algorithm}

To reduce the total number of global updates over time horizon $T$, we adopt the event-triggered communication from \cite{wang2019distributed}, such that global update is triggered if the following event is true for any client $i \in [N]$ (line 8):
\begin{equation} \label{eq:event_trigger}
    (t-t_{\text{last}})\log{\frac{\det(A_{t,i})}{\det(A_{t,i}-\Delta A_{t,i})}} > D
\end{equation}
where $\Delta A_{t,i}$
denotes client $i$'s local update to its covariance matrix since last global update at $t_{\text{last}}$, and $D>0$ is the chosen threshold for the event-trigger. During the global update, the model estimation $\theta_{t,i}$, covariance matrix $A_{t,i}$ and vector $b_{t,i}$ for all clients $i \in [N]$ will be updated (line 11-14). 
We should note that the LHS of Eq.\eqref{eq:event_trigger} is essentially an upper bound of the cumulative regret that client $i$'s locally updated model has incurred since $t_{\text{last}}$. Therefore, this event-trigger guarantees that a global update only happens when effective regret reduction is possible.

\noindent\textbf{$\bullet$ Arm selection}
To balance exploration and exploitation during bandit learning, \algone{} uses the OFUL principle for arm selection \citep{abbasi2011improved}, which requires the construction of a confidence ellipsoid for each client $i$.
We propose a novel construction of the confidence ellipsoid based on the sequence of model updates that each client $i$ has received up to time $t$: basically, there are 1) one global update at $t_{\text{last}}$, i.e., the joint offline regression across all clients' accumulated data till $t_{\text{last}}$: $\{(\bx_{s,i}, y_{s,i})\}_{s \in [t_{\text{last}}], i \in [N]}$, which resets all clients' local models to $\theta_{t_{\text{last}}}$; and 2) multiple local updates from $t_{\text{last}}+1$ to $t$, i.e., the online regression on client $i$'s own data sequence $\{(\bx_{s,i}, y_{s,i})\}_{s \in [t_{\text{last}}+1,t]}$ to get $\{\theta_{s,i}\}_{s \in [t_{\text{last}}+1,t]}$ step by step.
This can be more easily understood by the illustration in Figure \ref{fig:fedglb_illustration}.
The resulting confidence ellipsoid is centered at the ridge regression estimator $\hat{\theta}_{t,i}=A_{t,i}^{-1} b_{t,i}$  \citep{abbasi2012online,jun2017scalable}, which is computed using the predicted rewards given by the past sequence of model updates $\{\theta_{t_{\text{last}}}\} \cup  \{\theta_{s,i}\}_{s \in [t_{\text{last}}+1,t]}$ (see the update of $b_{t,i}$ in line 9 and 13 of Algorithm \ref{algo:3}).
Then at time step $t$, client $i$ selects the arm that maximizes the UCB score:
\begin{equation}\label{eq:UCB}
    \bx_{t,i}=\argmax_{ \bx \in \cA_{t,i}}{ \bx^{\top}\hat{\theta}_{t-1,i}
    +\alpha_{t-1,i}||\bx||_{A_{t-1,i}^{-1}}}
\end{equation}
where $\alpha_{t-1,i}$ is the parameter of the confidence ellipsoid given in Lemma \ref{lem:confidence_ellipsoid_fedglb}. 
Note that compared with standard federated/distributed learning where clients only need to communicate gradients for joint model estimation, in our problem, due to the time-varying arm set, it is also necessary to communicate the confidence ellipsoid among clients, i.e., $A_{t} \in \bR^{d \times d}$ and $b_{t} \in \bR^{d}$ (line 14 in Algorithm \ref{algo:3}), as the clients need to be prepared for all possible arms $\bx \in \bR^{d}$ that may appear in future for the sake of regret minimization. 

\subsection{Theoretical Analysis} \label{subsec:theorectical_analysis}
In this section, we construct the confidence ellipsoid based on the \emph{offline-and-online} estimators described in Section \ref{subsec:algo_description}. Then we analyze the cumulative regret and communication cost of \algone{}, followed by theoretical comparisons with its different variants.

\noindent\textbf{$\bullet$ Construction of confidence ellipsoid}
Compared with prior works that convert a sequence of online regression estimators to confidence ellipsoid \citep{abbasi2012online,jun2017scalable}, our confidence ellipsoid is built on the combination of an offline regression estimator $\theta_{t_{\text{last}}}$ for global update, and the subsequent online regression estimators $\{\theta_{s-1,i}\}_{s \in [t_{\text{last}}+1,t]}$ for local updates on each client $i$. This construction is new and requires proof techniques unique to our proposed solution. In the following, we highlight the key steps, and refer our readers to the appendix for details.

To simplify the use of notations, we assume without loss of generality that the global update at $t_{\text{last}}$ is triggered by the $N$-th client, such that 
no more new data will be collected at $t_\text{last}$, i.e., the first data sample obtained after the global update has index $t_{\text{last}}+1$.
We start our construction by considering the following loss difference introduced by the global and local model updates: $\sum_{s=1}^{t_{\text{last}}} \sum_{i=1}^{N} \bigl[ l(\bx_{s,i}^{\top} \theta_{t_{\text{last}}}, y_{s,i}) - l(\bx_{s,i}^{\top} \theta_{\star}, y_{s,i}) \bigr] + \sum_{s=t_{\text{last}}+1}^{t} \bigl[ l(\bx_{s,i}^{\top} \theta_{s-1,i}, y_{s,i}) - l(\bx_{s,i}^{\top} \theta_{\star}, y_{s,i}) \bigr]$,
where the first term is the loss difference between the globally updated model $\theta_{t_{\text{last}}}$ and $\theta_{\star}$,
and the second term is between the sequence of locally updated models $\{\theta_{s-1,i}\}_{s \in [t_{\text{last}}+1,t]}$ and $\theta_{\star}$.
This extends the definition of online regret used in the construction in \citep{abbasi2012online,jun2017scalable}; and due to the existence of offline regression, the obtained upper bounds in Lemma \ref{lem:loss_diff} are unique to our solution.


\begin{lemma}[Upper Bound of Loss Difference] \label{lem:loss_diff}
Denote the sub-optimality of the global model update procedure at time step $t_{\text{last}}$ as $\epsilon_{t_{\text{last}}}$, such that $F_{t_{\text{last}}}(\theta)-\min_{\theta \in \cB_{d}(S)} F_{t_{\text{last}}}(\theta) \leq \epsilon_{t_{\text{last}}}$. Then under Assumption \ref{assump:1} and \ref{assump:2}, we have
\begin{equation} \label{eq:loss_diff_1}
    \sum_{s=1}^{t_{\text{last}}} \sum_{i=1}^{N} \bigl[ l(\bx_{s,i}^{\top} \theta_{t_{\text{last}}}, y_{s,i}) - l(\bx_{s,i}^{\top} \theta_{\star}, y_{s,i}) \bigr] \leq B_{1}
\end{equation}
where $B_{1}=N t_{\text{last}} \epsilon_{t_{\text{last}}} + \frac{\lambda}{2} S^{2}$, and with probability at least $1-\delta$,
\begin{equation}\label{eq:loss_diff_2}
    \sum_{s=t_{\text{last}}+1}^{t} \bigl[ l(\bx_{s,i}^{\top} \theta_{s-1,i}, y_{s,i}) - l(\bx_{s,i}^{\top} \theta_{\star}, y_{s,i}) \bigr] \leq B_{2}
\end{equation}
where 
\small
$B_{2}=\frac{1}{2 c_{\mu}} \sum_{s=t_{\text{last}}+1}^{t} \lVert \nabla l(\bx_{s,i}^{\top} \theta_{s-1,i}, y_{s,i}) \rVert^{2}_{A_{s,i}^{-1}}  + \frac{c_{\mu}}{2} \Big[ \frac{1}{c_{\mu}} R_{\max} \sqrt{d\log{(1+{N t_{\text{last}} c_{\mu}}/{d \lambda})}+2\log{({1}/{\delta})}} + 2 N t_{\text{last}} \sqrt{\frac{2 k_{\mu} }{\lambda c_{\mu}} + \frac{2}{ N t_{\text{last}} c_{\mu}}} \sqrt{\epsilon_{t_{\text{last}}}}  + \sqrt{\frac{\lambda}{c_{\mu}}}S  \Big]^{2}$,
\normalsize
respectively.
\end{lemma}
Specifically, $B_{1}$ corresponds to the convergence of the offline (distributed) optimization in previous global update; $B_{2}$ is essentially the online regret upper bound of ONS, with the major difference that it is initialized using the globally updated model $\theta_{t_\text{last}}$, instead of an arbitrary model as in standard ONS. 
Then due to the $c_{\mu}$-strongly-convexity of $l(z,y)$ w.r.t. $z$, i.e., $l(\bx_{s}^{\top}\theta,y_{s}) - l(\bx_{s}^{\top}\theta_{\star},y_{s}) \geq \bigl[ \mu(\bx_{s}^{\top}\theta_{\star})-y_{s} \bigr]\bx_{s}^{\top}(\theta-\theta_{\star}) + \frac{c_{\mu}}{2}
\bigl[\bx_{s}^{\top}(\theta-\theta_{\star})\bigr]^{2}$,  
and by rearranging terms in Eq.\eqref{eq:loss_diff_1} and Eq.\eqref{eq:loss_diff_2}, we have: $\sum_{s=1}^{t_{\text{last}}} \sum_{i=1}^{N} \bigl[\bx_{s,i}^{\top}(\theta_{t_{\text{last}}}-\theta_{\star})\bigr]^{2} \leq  \frac{2}{c_{\mu}} B_{1} +  \frac{2}{c_{\mu}}\sum_{s=1}^{t_{\text{last}}} \sum_{i=1}^{N} \eta_{s,i} \bx_{s,i}^{\top}(\theta_{t_{\text{last}}}-\theta_{\star})$, and $\sum_{s=t_{\text{last}}+1}^{t} \bigl[\bx_{s,i}^{\top}(\theta_{s-1,i}-\theta_{\star})\bigr]^{2} \leq \frac{2}{c_{\mu}} B_{2} + \frac{2}{c_{\mu}}\sum_{s=t_{\text{last}}+1}^{t} \eta_{s,i} \bx_{s,i}^{\top}(\theta_{s-1,i}-\theta_{\star})$, whose LHS is quadratic in $\theta_{\star}$.
To further upper bound the RHS, we should note that the term $\frac{2}{c_{\mu}}\sum_{s=t_{\text{last}}+1}^{t} \eta_{s,i} \bx_{s,i}^{\top}(\theta_{s-1,i}-\theta_{\star})$ is standard in \citep{abbasi2012online,jun2017scalable} as $\bx_{s,i}^{\top}(\theta_{s-1,i}-\theta_{\star})$ is $\cF_{s,i}$-measurable for online estimator $\theta_{s-1,i}$. However, this is not true for the term $\frac{2}{c_{\mu}}\sum_{s=1}^{t_{\text{last}}} \sum_{i=1}^{N} \eta_{s,i} \bx_{s,i}^{\top}(\theta_{t_{\text{last}}}-\theta_{\star})$ as the offline regression estimator $\theta_{t_{\text{last}}}$ depends on all data samples collected till $t_{\text{last}}$; and thus we have to develop a different approach to bound it.
This leads to Lemma \ref{lem:confidence_ellipsoid_fedglb}
below, which provides the confidence ellipsoid for $\theta_{\star}$. 
\begin{lemma}[Confidence Ellipsoid of \algone{}] \label{lem:confidence_ellipsoid_fedglb}
With probability at least $1-2\delta$, for all $t \in [T], i \in [N]$,
\begin{align*}
    \lVert \hat{\theta}_{t,i}- \theta_{\star} \rVert_{A_{t,i}}^{2} & \leq \beta_{t,i} + \frac{\lambda}{c_{\mu}} S^{2} - \lVert \textbf{z}_{t,i} \rVert_{2}^{2} + \hat{\theta}_{t,i}^{\top} b_{t,i} 
    := \alpha_{t,i}^{2}
\end{align*}
where
$\textbf{z}_{t,i}$ denotes the vector of predicted rewards $[\bx_{1,1}^{\top}\theta_{t_{\text{last}}}, \bx_{1,2}^{\top}\theta_{t_{\text{last}}},\dots , \bx_{t_{\text{last}},N-1}^{\top} \theta_{t_{\text{last}}},\bx_{t_{\text{last}},N}^{\top} \theta_{t_{\text{last}}}$, $\bx_{t_{\text{last}}+1,i}^{\top} \theta_{t_{\text{last}},i}, \bx_{t_{\text{last}}+2,i}^{\top} \theta_{t_{\text{last}}+1,i},\dots , \bx_{t,i}^{\top} \theta_{t-1,i} ]^{\top}$,
and
\small
$\beta_{t,i} =  \frac{8R_{\max}^{2}}{c_{\mu}^{2}} \log\bigl(\frac{1}{\delta} \sqrt{\det(I + \sum_{s=1}^{t_{\text{last}}} \sum_{i=1}^{N} \bx_{s,i}\bx_{s,i}^{\top})} \bigr) + B_{1} + \frac{4R_{\max}}{c_{\mu}} \sqrt{
    2 \log\bigl(\frac{1}{\delta} \sqrt{\det(I + \sum_{s=1}^{t_{\text{last}}} \sum_{i=1}^{N} \bx_{s,i}\bx_{s,i}^{\top})} \bigr)} \Big( \lVert \theta_{t_{\text{last}}} \rVert_{2}+ \lVert\theta_{\star} \rVert_{2} + \sqrt{B_{1}} \Big) + \frac{4 B_{2}}{c_{\mu}} +  \frac{8 R_{\max}^{2}}{c_{\mu}^{2}} \log\Big( \frac{N}{\delta} \sqrt{4+\frac{8}{c_{\mu}}B_{2} + \frac{64 R_{\max}^{4}}{c_{\mu}^{4}\cdot 4 \delta^{2}}}  \Big)  + 1$.
\normalsize
\end{lemma}

\noindent\textbf{$\bullet$ Regret and communication cost}
From Lemma \ref{lem:confidence_ellipsoid_fedglb}, we can see that $\alpha_{t,i}$ grows at a rate of $Nt_{\text{last}}\sqrt{\epsilon_{t_{\text{last}}}}$ through its dependence on the $B_{2}$ term. To make sure the growth rate of $\alpha_{t,i}$ matches that in standard GLB algorithms \citep{li2017provably,jun2017scalable}, we set $\epsilon_{t_{\text{last}}}=\frac{1}{N^{2} t_{\text{last}}^{2}}$, which leads to the following corollary.
\begin{corollary}[Order of $\beta_{t,i}$] \label{corollary:order_ellipsoid}
With $\epsilon_{t_{\text{last}}} = \frac{1}{N^{2}t_{\text{last}}^{2}}$, $\beta_{t,i}=O\big(\frac{d\log{NT}}{c_{\mu}^{2}}[k_{\mu}^{2}+R_{\max}^{2}]\big)$.
\end{corollary}

Then using a similar argument as the proof for Theorem 4 of \citep{wang2019distributed}, we obtain the following upper bounds on $R_{T}$ and $C_{T}$ for \algone{} (proof in Appendix \ref{sec:prove_regret_comm}).
\begin{theorem}[Regret and Communication Cost Upper Bound of \algone{}] \label{thm:regret_comm_upper_bound}
Under Assumption \ref{assump:1}, \ref{assump:2}, and by setting $\epsilon_{t} = \frac{1}{N^{2}t^{2}},\forall t$ and $D=\frac{T}{Nd \log(NT)}$, the cumulative regret $R_{T}$ has upper bound 
\begin{align*}
    R_{T}=O\left(\frac{k_{\mu}(k_{\mu}+R_{\max})}{c_{\mu}}d\sqrt{NT}\log(NT/\delta)\right),
\end{align*}
with probability at least $1-2\delta$. The corresponding communication cost $C_{T}$\footnote{This is measured by the \emph{total number of times} data is transferred. Some works \cite{wang2019distributed} measure $C_{T}$ by the \emph{total number of scalars} transferred, in which case, we have $C_{T}=O\Big(d^{3} N^{1.5} \log(NT) + d^{2} N^{2} T^{0.5} \log^{2}(NT)  \Big)$.} has upper bound 
\begin{align*}
    C_{T} =O\Big(d N^{2}\sqrt{T}\log^{2}(NT)\Big).
\end{align*}
\end{theorem}
Theorem \ref{thm:regret_comm_upper_bound} shows that \algone{} recovers the standard $O\big(d \sqrt{NT}\log(NT)\big)$ rate in regret as in the centralized setting, while only incurring a communication cost that is sub-linear in $T$. 
Note that, to obtain $O\big(d \sqrt{NT}\log(NT)\big)$ regret for federated linear bandit,
the DisLinUCB algorithm incurs a communication cost of $O(d N^{1.5} \log(NT))$ \citep{wang2019distributed}, which is smaller than that of \algone{} by a factor of $\sqrt{NT}\log(NT)$. As the frequency of global updates is the same for both algorithms (due to their use of the same event-trigger), this additional communication cost is caused by the iterative optimization procedure for the global update, which is required for GLB model estimation. Moreover, as we mentioned in Section \ref{subsec:algo_description}, there is not much room for improvement here as the use of AGD already matches the lower bound up to a logarithmic factor.

To facilitate the understanding of our algorithm design and investigate the impact of 
different components of \algone{}
on its regret and communication efficiency trade-off, we propose and analyze three variants, which are also of independent interest, and report the results in Table \ref{tb:theoretical_comparison}. Detailed descriptions, as well as proof for these results can be found in Appendix \ref{sec:prove_variants}. 
Note that all three variants perform global update according to a fixed schedule $\cS=\{t_{1}:=\lfloor\frac{T}{B}\rfloor,t_{2}:=2\lfloor \frac{T}{B}\rfloor,\dots,t_{B}:=B\lfloor\frac{T}{B}\rfloor\}$, where $B$ denotes the total number of global updates specified in advance to trade-off between $R_{T}$ and $C_{T}$, and these variants differ in their global and local update strategies.
This comparison demonstrates that our solution is proven to achieve a better regret-communication trade-off against these reasonable alternatives. For example, when using standard federated learning methods (which assume fixed dataset) for streaming data in real-world applications, it is a common practice to set some fixed schedule to periodically retrain the global model to fit the new dataset, and \algone{}$_{1}$ implements such behaviors. The design of \algone{}$_{3}$ is motivated by distributed online convex optimization that also deals with streaming data in a distributed setting.


\begin{table*}[t]
\centering
\caption{Comparison between \algone{} and its variants with different design choices.}
\scriptsize
\begin{tabular}{c c c c c}
\toprule
Global Upd. & Local Upd. & Setting & $R_{T}$ & $C_{T}$ \\
 \hline
AGD & ONS & $D=\frac{T}{ N d \log(NT)}$ & $\frac{k_{\mu}(k_{\mu}+R_{\max})}{c_{\mu}}d\sqrt{NT}\log(NT)$ & $dN^{2}\sqrt{T}\log^{2}(NT)$  \\
AGD & no update & $B=\sqrt{NT}$ & $\frac{k_{\mu}R_{\max}}{c_{\mu}}d\sqrt{NT}\log(NT)$ & $N^{2} T \log(NT)$  \\
AGD & ONS & $B=d^2 N \log(NT)$ & $\frac{k_{\mu}(k_{\mu}+R_{\max})}{c_{\mu}}d\sqrt{NT}\log(NT)\log(T)$ &  $d^{2}N^{2.5} \sqrt{T} \log^{2}(NT)$ \\
ONS & ONS & $B=\sqrt{NT}$ & $\frac{k_{\mu}(k_{\mu}+R_{max})}{c_{\mu}}d (NT)^{3/4}\log(NT)$ & $N^{1.5} \sqrt{T}$ \\
\bottomrule
\end{tabular}
\normalsize
\label{tb:theoretical_comparison}
\end{table*}

%% file: exp.tex

\section{Experiments}  \label{sec:exp}
We performed extensive empirical evaluations of \algone{} on both synthetic and real-world datasets, and the results (averaged over 10 runs) are reported in Figure \ref{fig:real_exp_results}.
We included the three variants of \algone{} (listed in Table \ref{tb:theoretical_comparison}), One-UCB-GLM, N-UCB-GLM \citep{li2017provably} and N-ONS-GLM \citep{jun2017scalable} as baselines, where One-UCB-GLM learns a shared bandit model across all clients, and N-UCB-GLM and N-ONS-GLM learn a separated bandit model for each client with no communication. Additional results and discussions about experiments can be found in Appendix \ref{sec:explain_scatter_plot}.

\begin{figure*}[ht]
\centering     
\subfigure[Synthetic]{\label{fig:a}\includegraphics[width=0.48\textwidth]{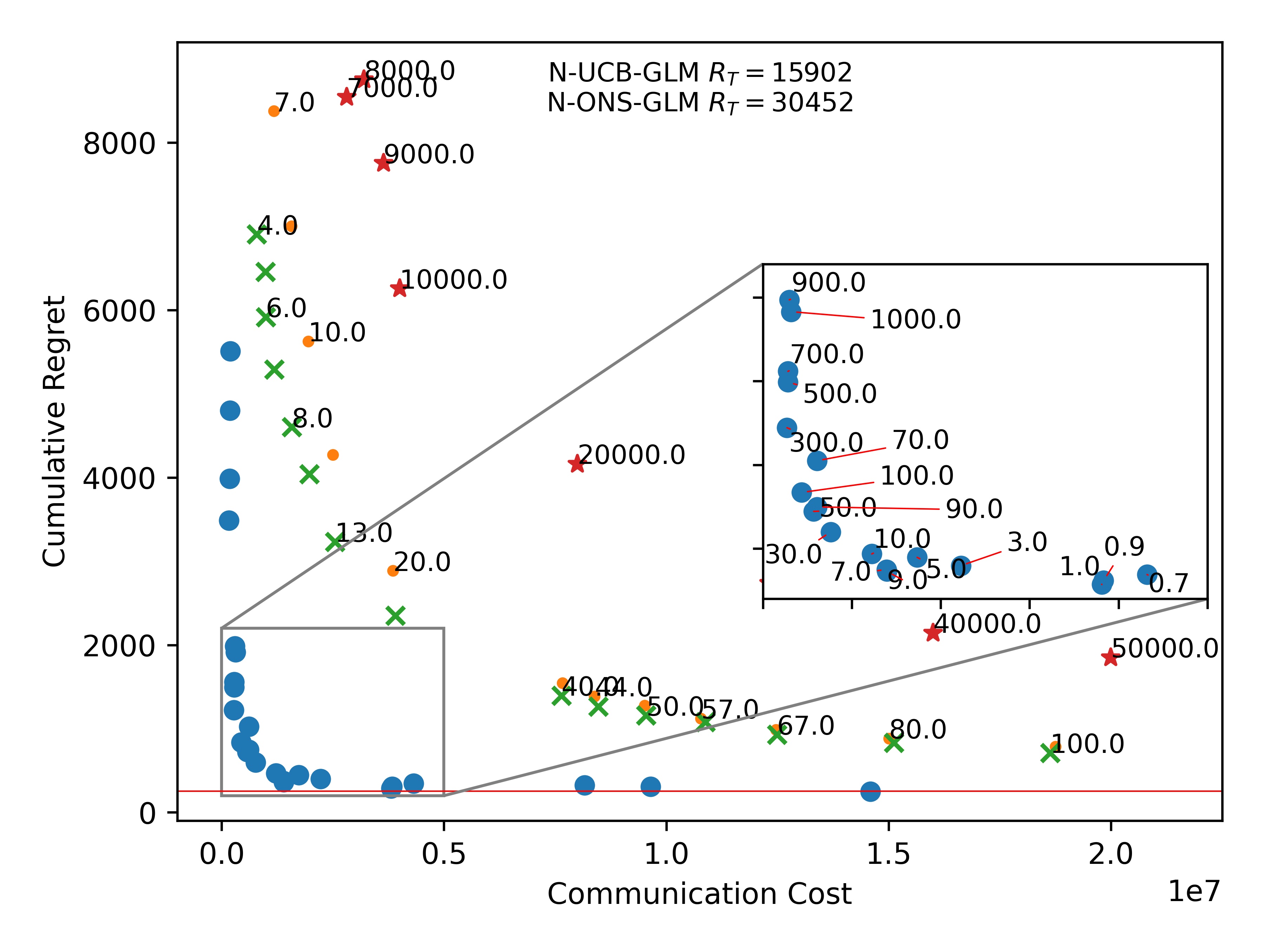}}
\subfigure[CoverType]{\label{fig:b}\includegraphics[width=0.49\textwidth]{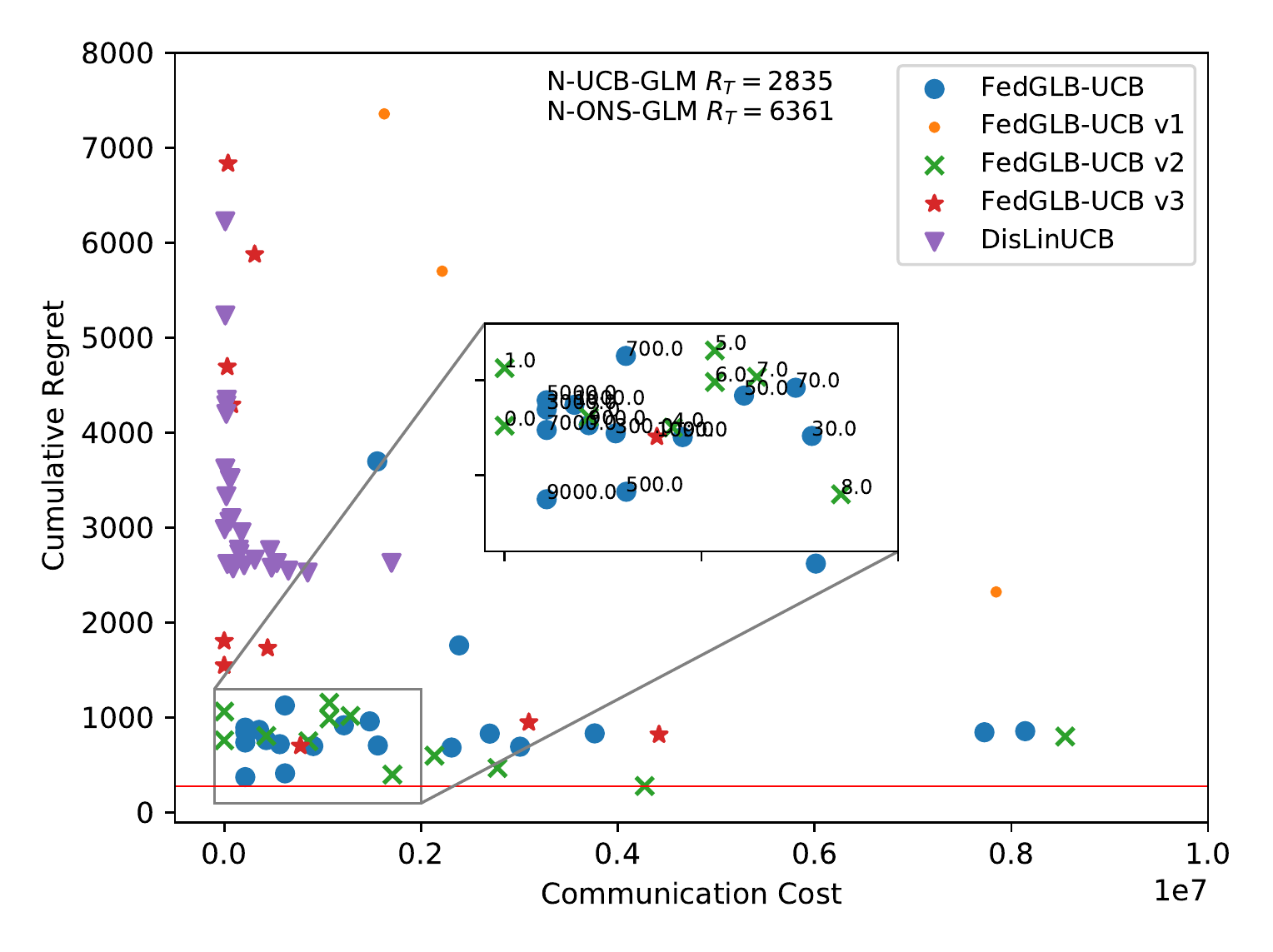}}\vspace{-2mm}
\subfigure[MagicTelescope]{\label{fig:c}\includegraphics[width=0.48\textwidth]{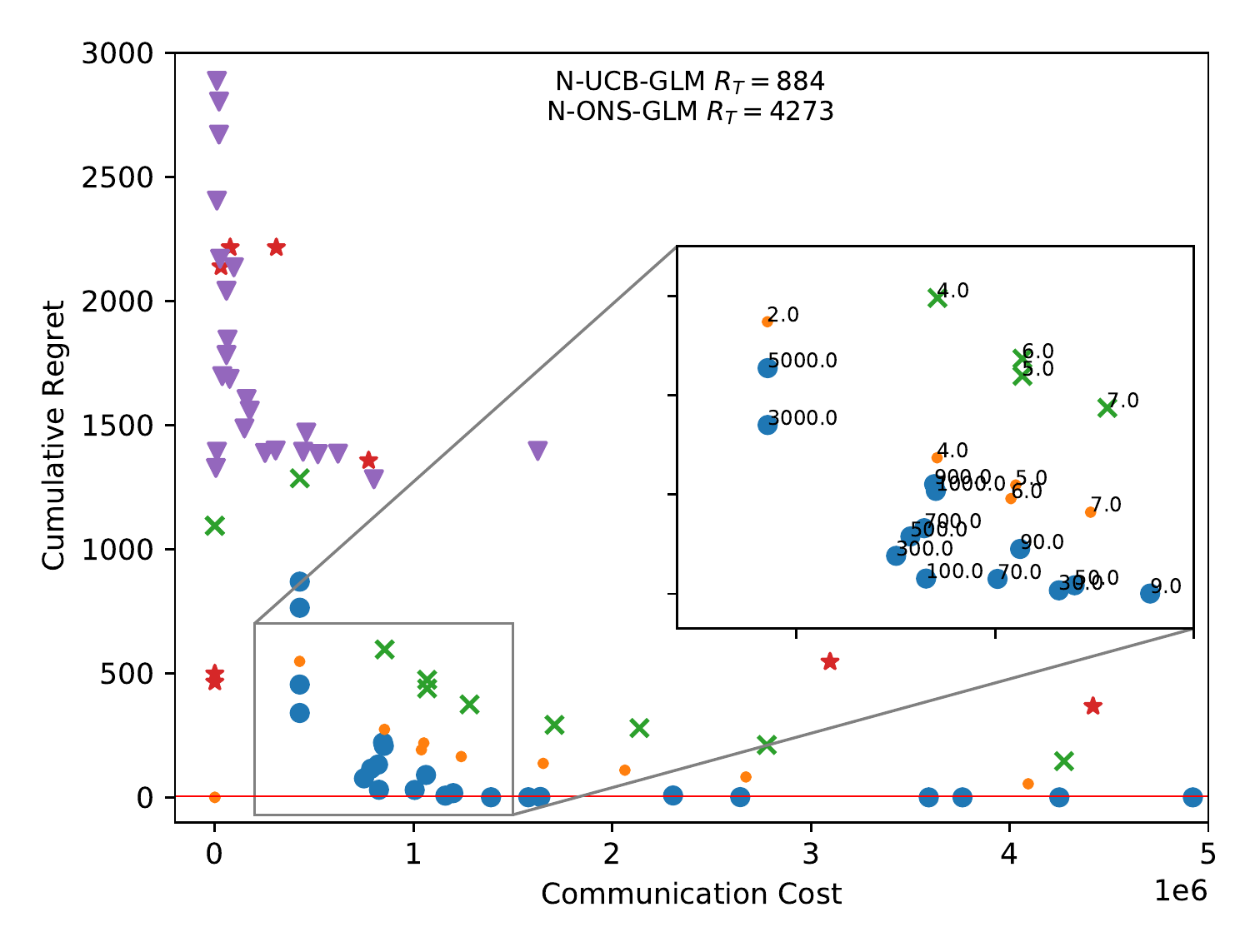}}
\subfigure[Mushroom]{\label{fig:d}\includegraphics[width=0.487\textwidth]{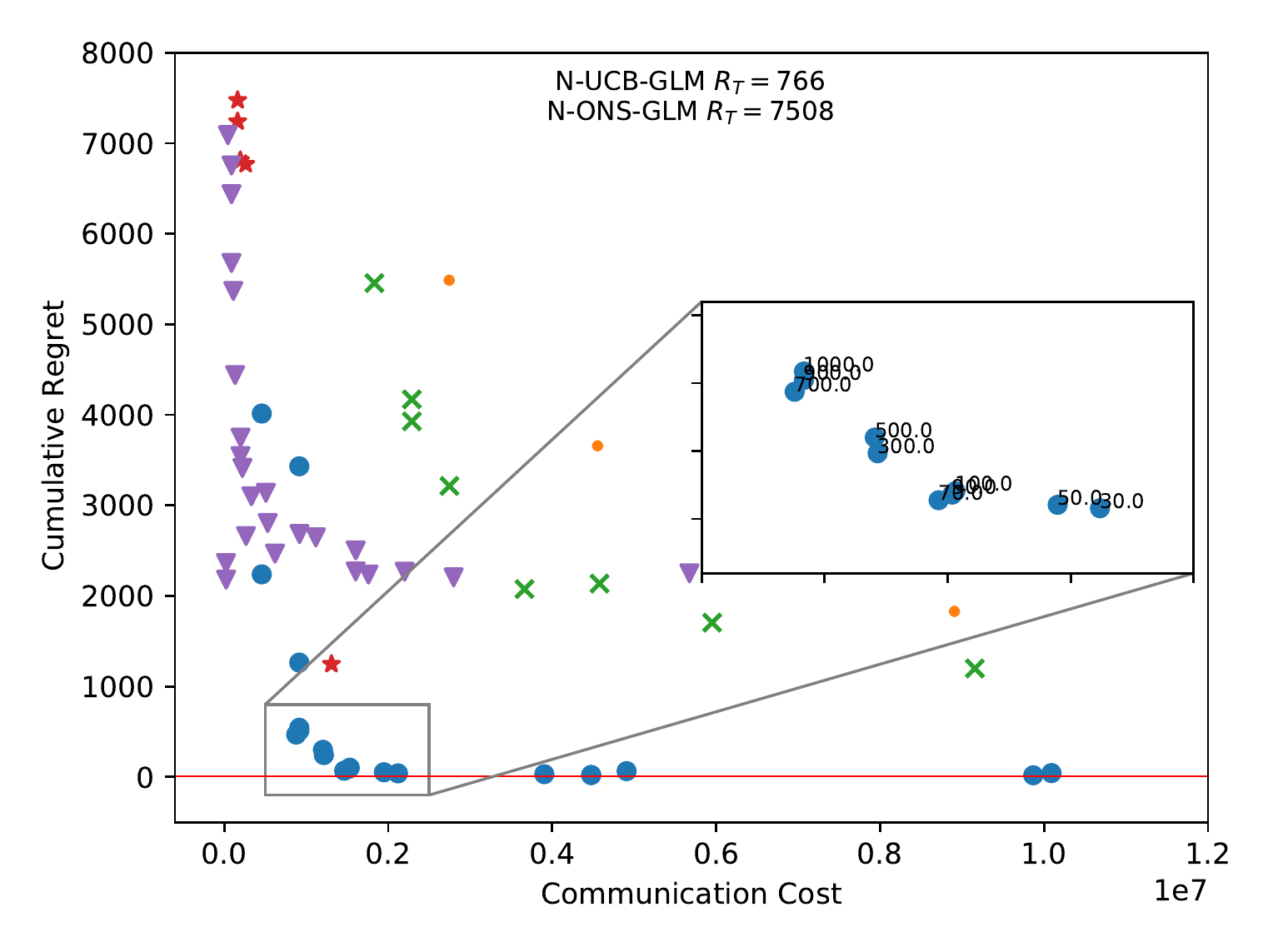}}
\caption{Experiment results on synthetic and real world datasets.}
\label{fig:real_exp_results}
\end{figure*}

\noindent\textbf{$\bullet$ Synthetic Dataset}
We simulated the federated GLB setting defined in Section \ref{subsec:federated_generalized_linear}, with $T=2000, N=200, d=10, S=1$, $\cA_{t}$ ($K=25$) uniformly sampled from a $\ell_2$ unit sphere, and reward $y_{t,,i} \sim \text{Bernoulli}(\mu(\bx_{t,,i}^{\top}\theta_{\star}))$, with $\mu(z)=(1+\exp(-z))^{-1}$. 
To compare the algorithms' $R_{T}$ and $C_{T}$ under different trade-off settings, we run \algone{} with different threshold value $D$ (logarithmically spaced between $10^{-1}$ and $10^{3}$)
and its variants 
with different number of global updates $B$.
Note that each dot in the result figure illustrates the $C_{T}$ (x-axis) and $R_{T}$ (y-axis) that a particular instance of \algone{} or its variants
obtained by time $T$, and the corresponding value for $D$ or $B$ is labeled next to the dot. $R_{T}$ of One-UCB-GLM is illustrated as the red horizontal line, and $R_{T}$ of N-UCB-GLM and N-ONS-GLM are labeled on the top of the figure.
We can observe that for \algone{} and its variants, $R_{T}$ decreases as $C_{T}$ increases, interpolating between the two extreme cases: independently learned bandit models by N-UCB-GLM, N-ONS-GLM; and the jointly learned bandit model by One-UCB-GLM.
\algone{} significantly reduces $C_{T}$, while attaining low $R_{T}$, i.e., its regret is even comparable with One-UCB-GLM that requires at least $C_{T}=N^{2}T$ ($8\times 10^{7}$ in this simulation) for gradient aggregation at each time step.
\noindent\textbf{$\bullet$ Real-world Dataset}
The results above demonstrate the effectiveness of \algone{} when data is generated by a well-specified generalized linear model. To evaluate its performance in a more challenging and practical scenario, we performed experiments using real-world datasets: CoverType, MagicTelescope and Mushroom from the UCI Machine Learning Repository \citep{Dua:2019}. To convert them to contextual bandit problems, we pre-processed these datasets following the steps in prior works \citep{filippi2010parametric}, with $T=2000$ and $N=20$. Moreover, to demonstrate the advantage of GLB over linear model, we included DisLinUCB \citep{wang2019distributed} as an additional baseline. 
Since the parameters being communicated in DisLinUCB and \algone{} are different, to ensure a fair comparison of $C_T$ in this experiment, we measure communication cost (x-axis) by the number of integers or real numbers transferred across the learning system (instead of the frequency of communications). Note that DisLinUCB has no $C_{T} \geq 3\times 10^{6}$in Figure \ref{fig:real_exp_results} because 
its global update is already happening
in every round and cannot be increased further. As mentioned earlier, due to the difference in messages being sent, the communication in DisLinUCB's \textit{per global update} is much smaller than that in FedGLB-UCB.
However, because linear models failed to capture the complicated reward mappings in these three datasets, we can see that DisLinUCB is clearly outperformed by \algone{} and its variants.
This shows that, by offering a larger variety of modeling choices, e.g., linear, Poisson, logistic regression, etc., \algone{} has more potential in dealing with the complicated data in real-world applications.

%% file: appendix.tex
\section{Technical Lemmas} \label{sec:technical_lemmas}
When applying the following self-normalized bound in the analysis of federated bandit algorithm with event-trigger, a subtle difference from the analysis of standard bandit algorithm is that the sequence of data points used to update each client is controlled by the data-dependent event-trigger, e.g. Eq~\eqref{eq:event_trigger}, which introduces dependencies on the future data, and thus breaks the standard argument. 
This problem also exists in prior works of distributed linear bandit, but was not addressed rigorously (see Lemma H.1. of \citep{wang2019distributed}).
Specifically, each client $i$ observes the sequence of data points in a different order, i.e., it first observes each newly collected local data points from the environment, and then observes (in the form of their gradients) the batch of new data points that other clients have collected at the end of the epoch.
Then, if we consider a data point that is contained in the batch of new data collected by other clients, the index of this data point (as observed by client $i$) has dependency all the way to the end of this batch,
i.e., its index is only determined after some client triggers the global update.

Therefore, in order to avoid this dependency on future data points, when constructing the filtration, we should avoid including the $\sigma$-algebra that `cuts a batch in half', but instead only include the $\sigma$-algebra generated by the sequence of data points up to the end of each batch, where we consider each locally observed data point as a batch as well. Denote the sequence of time indices corresponding to these data points as $\{t_{p}\}_{p \in [P]}$ for some $P>0$. Then the constructed filtration $\{\cF_{t_{p}}\}_{p \in [P]}$ is essentially a batched version of the standard $\{\cF_{t}\}_{t=1}^{\infty}$. The self-normalized bound below still holds, i.e., by changing the stopping time construction from $\cup_{t \geq 1} B_{t}(\delta)$ to $\cup_{t \in \{t_{p}\}_{p \in [P]}} B_{t}(\delta)$ in the proof of Theorem 1 in \citep{abbasi2011improved}, where $B_{t}(\delta)$ denotes the bad event that the bound does not hold. Therefore, instead of holding uniformly over all $t$, the self-normalized bound now only holds for all $t \in \{t_{p}\}_{p \in [P]}$, i.e., the sequence of time steps when client $i$ gets updated, which is also what we need.

\begin{lemma}[Vector-valued self-normalized bound (Theorem 1 of \cite{abbasi2011improved})] \label{lem:self_normalized_bound}
Let $\{\cF_{t}\}_{t=1}^{\infty}$ be a filtration. Let $\{\eta_{t}\}_{t=1}^{\infty}$ be a real-valued stochastic process such that $\eta_{t}$ is $\cF_{t+1}$-measurable, and $\eta_{t}$ is conditionally zero mean $R$-sub-Gaussian for some $R \geq 0$.
Let $\{X_{t}\}_{t=1}^{\infty}$ be a $\bR^{d}$-valued stochastic process such that $X_{t}$ is $\cF_{t}$-measurable. Assume that $V$ is a $d \times d$ positive definite matrix. For any $t > 0$, define
\begin{align*}
    V_{t}=V+\sum_{\tau=1}^{t} X_{\tau} X_{\tau}^{\top} \quad \cS_{t}=\sum_{\tau=1}^{t}\eta_{\tau} X_{\tau} 
\end{align*}
Then for any $\delta >0$, with probability at least $1-\delta$,
\begin{align*}
    ||\cS_{t}||_{V_{t}^{-1}} \leq  R \sqrt{2\log{\frac{\det(V_{t})^{1/2}}{\det(V)^{1/2}\delta}}}, \quad \forall t\geq 0
\end{align*}
\end{lemma}

\begin{lemma}[Corollary 8 of \citep{abbasi2012online}] \label{lem:uniform_self_normalized}
Under the same assumptions as Lemma \ref{lem:self_normalized_bound}, consider a sequence of real-valued
variables $\{Z_{t}\}_{t=1}^{\infty}$ such that $Z_{t}$ is $\cF_{t}$-measurable. Then for any $\delta > 0$, with probability at least $1-\delta$, 
\begin{align*}
    |\sum_{\tau=1}^{t} \eta_{\tau} Z_{\tau}| \leq R \sqrt{2 (V + \sum_{\tau=1}^{t} Z_{\tau}^{2}) \log\left(\frac{\sqrt{V + \sum_{\tau=1}^{t} Z_{\tau}^{2}}}{\delta \sqrt{V}}\right)}, \forall t \geq 0
\end{align*}
\end{lemma}



\begin{lemma} \label{lem:smoothness_constant}
Under Assumption \ref{assump:1}, $F_{i,t}(\theta)$ for $i=1,2,\dots,N$ is smooth with constant $k_{\mu} + \frac{\lambda}{Nt}$
\end{lemma}
\begin{proof}
By Assumption 1, $\mu(\cdot)$ is Lipschitz continuous with constant $k_{\mu}$, i.e.,
$|\mu(\bx^{\top} \theta_{1}) - \mu(\bx^{\top} \theta_{2})| \leq k_{\mu}|\bx^{\top} (\theta_{1} - \theta_{2})|$. Then we can show that
\begin{align*}
    & ||\nabla F_{i,t}(\theta_{1})-\nabla F_{i,t}(\theta_{2})|| \\
    & = ||\frac{1}{t}\sum_{s=1}^{t}\bx_{s,i}[\mu(\bx_{s,i}^{\top} \theta_{1})-\mu(\bx_{s,i}^{\top} \theta_{2})] + \frac{\lambda}{Nt}(\theta_{1}-\theta_{2})||  \\
    & \leq \frac{1}{t}\sum_{s=1}^{t}||\bx_{s,i}[\mu(\bx_{s,i}^{\top} \theta_{1})-\mu(\bx_{s,i}^{\top} \theta_{2})]|| + \frac{\lambda}{Nt}||\theta_{1}-\theta_{2}||  \\
    & \leq \frac{1}{t}\sum_{s=1}^{t} |\mu(\bx_{s,i}^{\top} \theta_{1})-\mu(\bx_{s,i}^{\top} \theta_{2})|  + \frac{\lambda}{Nt}||\theta_{1}-\theta_{2}||  \\
    & \leq \frac{ k_{\mu}}{t} \sum_{s=1}^{t}|\bx_{s,i}^{\top} (\theta_{1} - \theta_{2})| + \frac{\lambda}{Nt}||\theta_{1}-\theta_{2}||  \leq (k_{\mu} + \frac{\lambda}{Nt}) ||\theta_{1} - \theta_{2}||
\end{align*}
Therefore, $\nabla F_{i,t}(\theta)$ is Lipschitz continuous with constant $k_{\mu} + \frac{\lambda}{Nt}$, and $\nabla F_{t}(\theta)=\frac{1}{N}\sum_{i=1}^{N} \nabla F_{i,t}(\theta)$ is Lipschitz continuous with constant $k_{\mu} + \frac{\lambda}{Nt}$ as well. 
\end{proof}

\begin{lemma}[Matrix Weighted Cauchy-Schwarz] \label{lem:matrix_weighted_cauchy_schwartz}
If $A \in \mathbb{R}^{d \times d}$ is a PSD matrix, then $x^{T}A y \leq \sqrt{x^{T}A x \cdot y^{T}A y}$ holds for any vectors $x,y \in \mathbb{R}^{d}$.
\end{lemma}
\begin{proof}
Consider a quadratic function $(x+ty)^{T}A(x+ty)=x^{T}A x + 2 (x^{T}A y)t + (y^{T}A y)t^{2}$ for some variable $t \in \bR$, where $x,y \in \bR^{d}$ are arbitrary vectors. Since $A$ is PSD, the value of this quadratic function $(x+ty)^{T}A(x+ty)=x^{T}A x + 2 (x^{T}A y)t + (y^{T}A y)t^{2} \geq 0, \forall t$, which means there can be at most one root. This is equivalent to saying the discriminant of this quadratic function $4(x^{T} A y)^{2}-4 x^{T}A x \cdot y^{T}A y \leq 0$, which finishes the proof.
\end{proof}

\begin{lemma}[Confidence Ellipsoid Centered at Global Model]\label{lem:global_model_confidence_ellipsoid}
Consider time step $t \in [T]$ when a global update happens, such that the distributed optimization over $N$ clients is performed to get the globally updated model $\theta_{t}$. Denote the sub-optimality of the final iteration as $\epsilon_{t}$, such that $F_{t}(\theta_{t})-F_{t}(\hat{\theta}^{\text{MLE}}_{t}) \leq \epsilon_{t}$; then with probability at least $1-\delta$, for all $t \in [T]$,
\begin{equation*}
    ||\theta_{t} - \theta_{\star}||_{A_{t}} \leq \alpha_{t}
\end{equation*}
where $\alpha_{t} = N t \sqrt{\frac{2 k_{\mu} }{\lambda c_{\mu}} + \frac{2}{N t c_{\mu}}} \sqrt{\epsilon_{t}} + \frac{R_{max}}{c_{\mu}} \sqrt{d\log{(1+{N t c_{\mu}}/{(d \lambda)})}+2\log{({1}/{\delta})}} + \sqrt{\frac{\lambda}{c_{\mu}}}S$, and $A_{t}=\frac{\lambda}{c_{\mu}} I + \sum_{i=1}^{N}\sum_{s=1}^{t}\bx_{s,i}\bx_{s,i}^{\top}$.
\end{lemma}
\begin{proof}
Recall that the unique minimizer of Eq.\eqref{eq:objective} is denoted as $\hat{\theta}_{t}^{\text{MLE}}$, so by taking gradient w.r.t. $\theta$ we have, $g_{t}(\hat{\theta}_{t}^{\text{MLE}})=\sum_{i=1}^{N} \sum_{s=1}^{t} \bx_{s,i}y_{s,i}$,
where we define $g_{t}(\theta)=\lambda \theta + \sum_{i=1}^{N}\sum_{s=1}^{t}\mu(\bx_{s,i}^{\top}\theta)\bx_{s,i}$.
First, we start with standard arguments \citep{filippi2010parametric,li2017provably} to show that $\lVert \theta_{t} - \theta_{\star} \rVert_{A_{t}} \leq \frac{1}{c_{\mu}} \lVert g_{t}(\theta_{t})-g_{t}(\theta_{\star}) \rVert_{A_{t}^{-1}}$.
Specifically, by Assumption \ref{assump:1} and the Fundamental Theorem of Calculus, we have
$$g_{t}(\theta_{t})-g_{t}(\theta_{\star})=G_{t}(\theta_{t} - \theta_{\star})$$
where $G_{t}=\int_{0}^{1} \nabla g_{t}(a \theta_{t} + (1-a) \theta_{\star}) da$. Again by Assumption \ref{assump:1}, $\nabla g_{t}(\theta)=\lambda I + \sum_{s=1}^{t} \sum_{i=1}^{N} \bx_{s,i}\bx_{s,i}^{\top}\dot{\mu}(\bx_{s,i}^{\top}\theta)$ is continuous, and $\nabla g_{t}(\theta) \succcurlyeq \lambda I + c_{\mu} \sum_{s=1}^{t} \sum_{i=1}^{N}\bx_{s,i}\bx_{s,i}^{\top} \succ 0$ for $\theta \in \cB_{d}(S)$, so $G_{t} \succ 0$, i.e., $G_{t}$ is invertible. Therefore, we have
\begin{align*}
    \theta_{t} - \theta_{\star} = G_{t}^{-1}[g_{t}(\theta_{t})-g_{t}(\theta_{\star})]
\end{align*}
Note that $G_{t} \succcurlyeq \lambda I + c_{\mu} \sum_{s=1}^{t} \sum_{i = 1}^{N} \bx_{s,i} \bx_{s,i}^{\top}=c_{\mu} A_{t}$, so $G_{t}^{-1} \preccurlyeq \frac{1}{c_{\mu}} A_{t}^{-1}$. Hence,
\begin{align} \label{eq:confidence_ellipsoid_global_triangle}
     \lVert \theta_{t} - \theta_{\star} \rVert_{A_{t}} & = \lVert G_{t}^{-1}[g_{t}(\theta_{t})-g_{t}(\theta_{\star})]  \rVert_{A_{t}}  \leq \lVert \frac{1}{c_{\mu}} A_{t}^{-1}[g_{t}(\theta_{t})-g_{t}(\theta_{\star})]  \rVert_{A_{t}} = \frac{1}{c_{\mu}} \lVert g_{t}(\theta_{t})-g_{t}(\theta_{\star}) \rVert_{A_{t}^{-1}} \nonumber\\
    & \leq \frac{1}{c_{\mu}} \lVert g_{t}(\theta_{t})-g_{t}(\hat{\theta}_{t}^{\text{MLE}}) \rVert_{A_{t}^{-1}} + \frac{1}{c_{\mu}} \lVert g_{t}(\hat{\theta}_{t}^{\text{MLE}}) - g_{t}(\theta_{\star})\rVert_{A_{t}^{-1}}
\end{align}
where the first term depends on the sub-optimality of the offline regression estimator $\theta_{t}$ to the unique minimizer $\hat{\theta}^{(\text{MLE})}_{t}$, and the second term is standard for GLB \citep{li2017provably}.

Recall from Algorithm \ref{alg:agd_update} that $\theta_{t}=\argmin_{\theta \in \cB_{d}(S)} \lVert g_{t}(\tilde{\theta}_{t}) -g_{t}(\theta) \rVert_{A_{t}^{-1}}$, where $\tilde{\theta}_{t}$ denotes the AGD estimator before projection.
Therefore, for the first term, using triangle inequality and the definition of $g_{t}(\cdot)$, we have
\begin{align*}
    & ||g_{t}(\theta_{t})- g_{t}(\hat{\theta}^{(\text{MLE})}_{t})||_{A_{t}^{-1}} \leq ||g_{t}(\theta_{t})- g_{t}(\tilde{\theta}_{t})||_{A_{t}^{-1}} + ||g_{t}(\tilde{\theta}_{t})- g_{t}(\hat{\theta}^{(\text{MLE})}_{t})||_{A_{t}^{-1}} \\
    & \leq 2 ||g_{t}(\tilde{\theta}_{t})- g_{t}(\hat{\theta}^{(\text{MLE})}_{t})||_{A_{t}^{-1}} = 2||\lambda \theta_{t} + \sum_{s=1}^{t}\sum_{i=1}^{N}\bx_{s,i}\mu(\bx_{s,i}^{\top}\theta_{t})-\sum_{s=1}^{t}\sum_{i=1}^{N}\bx_{s,i}y_{s,i}||_{A_{t}^{-1}} \\
    & = 2|| \sum_{s=1}^{t}\sum_{i=1}^{N} \bx_{s,i}[-y_{i,s}+\mu(\bx_{s,i}^{\top}\theta_{t})] + \lambda \theta_{t}||_{A_{t}^{-1}} = 2||N t \nabla F_{t}(\theta_{t})||_{A_{t}^{-1}}
\end{align*}
where the last equality is due to the definition of $F_{t}(\theta)$ in Eq.\eqref{eq:objective}.
We can further bound it using the property of Rayleigh quotient and the fact that $A_{t} \succcurlyeq \frac{\lambda}{c_{\mu}} I$, which gives us
\begin{align*}
    ||g_{t}(\theta_{t})- g_{t}(\hat{\theta}^{(\text{MLE})}_{t})||_{A_{t}^{-1}} & \leq \frac{2 N t ||\nabla F_{t}(\theta_{t})||_{2}}{\sqrt{\lambda_{\text{min}}(A_{t})}} \leq \frac{2 N t||\nabla F_{t}(\theta_{t})||_{2}}{\sqrt{\lambda/c_{\mu}}}
\end{align*}
Based on Lemma \ref{lem:smoothness_constant}, $F_{t}(\theta)$ is $(k_{\mu}+\frac{\lambda}{N t})$-smooth, which means
\begin{align*}
    \frac{1}{2 k_{\mu}+{2\lambda}/(N t)}  \lVert \nabla F_{t}(\theta_{t})\rVert^{2} \leq F_{t}(\theta_{t}) - F_{t}(\hat{\theta}^{\text{MLE}}_{t})  \leq \epsilon_{t}
\end{align*}
where the second inequality is by definition of $\epsilon_{t}$.
Putting everything together, we have the following bound for the first term
\begin{align*}
    \frac{1}{c_{\mu}} \lVert g_{t}(\theta_{t})-g_{t}(\hat{\theta}_{t}^{\text{MLE}}) \rVert_{A_{t}^{-1}} \leq 2 N t \sqrt{\frac{2 k_{\mu} }{\lambda c_{\mu}} + \frac{2}{N t c_{\mu}}} \sqrt{\epsilon_{t}}
\end{align*}
For the second term, similarly, based on the definition of $g_{t}(\cdot)$, we have
\begin{align*}
    & \frac{1}{c_{\mu}}||g_{t}(\hat{\theta}_{t}^{\text{MLE}})-g_{t}(\theta_{\star})||_{A_{t}^{-1}}  \\
    & = \frac{1}{c_{\mu}}||\sum_{s=1}^{t} \sum_{i=1}^{N} \bx_{s,i}y_{s,i} - \sum_{s=1}^{t} \sum_{i=1}^{N}\mu(\bx_{s,i}^{\top}\theta_{\star})\bx_{s,i} -\lambda \theta_{\star}||_{A_{t}^{-1}} \\
    & \leq \frac{1}{c_{\mu}} ||\sum_{s=1}^{t} \sum_{i=1}^{N} \bx_{s,i} \eta_{s,i}||_{A_{t}^{-1}} + \sqrt{\frac{\lambda}{c_{\mu}}}S
\end{align*}
Then based on the self-normalized bound in Lemma \ref{lem:self_normalized_bound} (Theorem 1 of \citep{abbasi2011improved}), we have $||\sum_{s=1}^{t} \sum_{i=1}^{N} \bx_{s,i} \eta_{s,i}||_{A_{t}^{-1}} \leq R_{max} \sqrt{d\log{(1+{N t c_{\mu}}/{d \lambda})}+2\log{({1}/{\delta})}}, \forall t$, with probability at least $1-\delta$.

Substituting the upper bounds for these two terms back into Eq.\eqref{eq:confidence_ellipsoid_global_triangle}, we have, with probability at least $1-\delta$,
\begin{equation*}
\begin{split}
     \lVert \theta_{t} - \theta_{\star} \rVert_{A_{t}} & \leq \frac{1}{c_{\mu}} \lVert g_{t}(\theta_{t})-g_{t}(\hat{\theta}_{t}^{\text{(MLE)}}) \rVert_{A_{t}^{-1}} + \frac{1}{c_{\mu}} \lVert g_{t}(\hat{\theta}_{t}^{\text{(MLE)}}) - g_{t}(\theta_{\star})\rVert_{A_{t}^{-1}} \\
    & \leq 2 N t \sqrt{\frac{2 k_{\mu} }{\lambda c_{\mu}} + \frac{2}{N t c_{\mu}}} \sqrt{\epsilon_{t}} + \frac{R_{max}}{c_{\mu}} \sqrt{d\log{(1+{N t c_{\mu}}/{(d \lambda)})}+2\log{({1}/{\delta})}} + \sqrt{\frac{\lambda}{c_{\mu}}}S
\end{split}
\end{equation*}
which finishes the proof.
\end{proof}

\section{Proof of Lemma \ref{lem:loss_diff}} \label{sec:prove_loss_diff_bound}
\begin{proof}
Denote the two terms for loss difference as $A_{1}=\sum_{s=1}^{t_\text{last}} \sum_{i=1}^{N} \bigl[ l(\bx_{s,i}^{\top} \theta_{t_\text{last}}, y_{s,i}) - l(\bx_{s,i}^{\top} \theta_{\star}, y_{s,i}) \bigr]$, and $A_{2}=\sum_{s=t_\text{last}+1}^{t} \bigl[ l(\bx_{s,i}^{\top} \theta_{s-1,i}, y_{s,i}) - l(\bx_{s,i}^{\top} \theta_{\star}, y_{s,i}) \bigr]$.
We can upper bound the term $A_{1}$ by
\begin{align*}
    A_{1} & = \sum_{s=1}^{t_\text{last}} \sum_{i=1}^{N} \bigl[ l(\bx_{s,i}^{\top} \theta_{t_\text{last}}, y_{s,i}) - l(\bx_{s,i}^{\top} \theta_{\star}, y_{s,i}) \bigr]  - \frac{\lambda}{2} \lVert \theta_{\star} \rVert_{2}^{2} + \frac{\lambda}{2} \lVert \theta_{\star} \rVert_{2}^{2} \\
    &  \leq \sum_{s=1}^{t_\text{last}} \sum_{i=1}^{N} l(\bx_{s,i}^{\top} \theta_{t_\text{last}}, y_{s,i}) - \sum_{s=1}^{t_\text{last}} \sum_{i=1}^{N} l(\bx_{s,i}^{\top} \hat{\theta}^{\text{MLE}}_{t_\text{last}}, y_{s,i})  - \frac{\lambda}{2} \lVert \hat{\theta}^{\text{MLE}}_{t_\text{last}} \rVert_{2}^{2} + \frac{\lambda}{2} \lVert \theta_{\star} \rVert_{2}^{2}  \\
    & \leq \sum_{s=1}^{t_\text{last}} \sum_{i=1}^{N} l(\bx_{s,i}^{\top} \theta_{t_\text{last}}, y_{s,i}) + \frac{\lambda}{2} \lVert \theta_{t_\text{last}} \rVert_{2}^{2} - \sum_{s=1}^{t_\text{last}} \sum_{i=1}^{N} l(\bx_{s,i}^{\top} \hat{\theta}^{\text{MLE}}_{t_\text{last}}, y_{s,i}) - \frac{\lambda}{2} \lVert \hat{\theta}^{\text{MLE}}_{t_\text{last}} \rVert_{2}^{2} + \frac{\lambda}{2} S^{2}\\
    & \leq N t_\text{last} \epsilon_{t_\text{last}} + \frac{\lambda}{2} S^{2}  := B_{1}
\end{align*}
where the first inequality is because $\hat{\theta}^{\text{MLE}}_{t_\text{last}}$ minimizes Eq.\eqref{eq:objective}, such that $\sum_{s=1}^{t_\text{last}}l(\bx_{s,i}^{\top} \hat{\theta}^{\text{MLE}}_{t_\text{last}}, y_{s,i}) + \frac{\lambda}{2}\lVert \hat{\theta}^{\text{MLE}}_{t_\text{last}} \rVert_{2}^{2} \leq \sum_{s=1}^{t_\text{last}}l(\bx_{s,i}^{\top} \theta, y_{s,i}) + \frac{\lambda}{2}\lVert \theta \rVert_{2}^{2}$ for any $\theta \in \cB_{d}(S)$, and the last inequality is because $F_{t_\text{last}}(\theta_{t_\text{last}})-  F_{t_\text{last}}(\hat{\theta}^{\text{MLE}}_{t_\text{last}}) \leq \epsilon_{t_\text{last}}$ by definition.


Now we start with standard arguments \citep{jun2017scalable,zhang2016online} in order to bound the term $A_{2}$, which is essentially the online regret of ONS, except that its initial model is the globally updated model $\theta_{t_\text{last}}$.
First, since $l(z,y)$ is $c_{\mu}$-strongly-convex w.r.t. $z$, we have
\begin{equation} \label{eq:strong_convex_ineq}
    l(\bx_{s,i}^{\top}\theta_{s-1,i},y_{s,i}) - l(\bx_{s,i}^{\top}\theta_{\star},y_{s,i}) \leq [\mu(\bx_{s,i}^{\top}\theta_{s-1,i})-y_{s,i}] \bx_{s,i}^{\top} (\theta_{s-1,i}-\theta_{\star}) -\frac{c_{\mu}}{2} ||\theta_{s-1,i} - \theta_{\star}||^{2}_{\bx_{s,i}\bx_{s,i}^{\top}}
\end{equation}
To further bound the RHS of Eq.\eqref{eq:strong_convex_ineq}, recall from the ONS local update rule in Algorithm \ref{alg:ons_update} that, for each client $i\in[N]$ at the end of each time step $s \in [t_\text{last}+1,t]$,
\begin{align*}
    & \theta_{s,i}^{\prime} = \theta_{s-1,i} - \frac{1}{c_{\mu}} A_{s,i}^{-1} \nabla l(\bx_{s,i}^{\top} \theta_{s-1,i}, y_{s,i})  \\
    & \theta_{s,i} = \argmin_{\theta \in \cB_{d}(S)} ||\theta_{s,i}^{\prime} - \theta||^{2}_{A_{s,i}}
\end{align*}
Then due to the property of generalized projection (Lemma 8 of \citep{hazan2007logarithmic}), we have
\begin{align*}
    & \lVert \theta_{s,i} - \theta_{\star} \rVert^{2}_{A_{s,i}} \\
    & \leq \lVert \theta_{s-1,i} - \theta_{\star} - \frac{1}{c_{\mu}} A_{s,i}^{-1} \nabla l(\bx_{s,i}^{\top} \theta_{s-1,i}, y_{s,i}) \rVert^{2}_{A_{s,i}} \\
    & \leq \lVert \theta_{s-1,i} - \theta_{\star} \rVert^{2}_{A_{s,i}} - \frac{2}{c_{\mu}} \nabla l(\bx_{s,i}^{\top} \theta_{s-1,i}, y_{s,i})^{\top} (\theta_{s-1,i}-\theta_{\star}) + \frac{1}{c_{\mu}^{2}} \lVert \nabla l(\bx_{s,i}^{\top} \theta_{s-1,i}, y_{s,i}) \rVert^{2}_{A_{s,i}^{-1}}
\end{align*}
By rearranging terms, we have
\small
\begin{align*}
    & \nabla l(\bx_{s,i}^{\top} \theta_{s-1,i}, y_{s,i})^{\top} (\theta_{s-1,i}-\theta_{\star}) \\
    & \leq \frac{1}{2c_{\mu}} \lVert \nabla l(\bx_{s,i}^{\top} \theta_{s-1,i}, y_{s,i}) \rVert^{2}_{A_{s,j}^{-1}} + \frac{c_{\mu}}{2} \bigl( \lVert \theta_{s-1,i} - \theta_{\star} \rVert^{2}_{A_{s,i}} -  \lVert \theta_{s,i} - \theta_{\star} \rVert^{2}_{A_{s,i}} \bigr)  \\
    & = \frac{1}{2c_{\mu}} \lVert \nabla l(\bx_{s,i}^{\top} \theta_{s-1,i}, y_{s,i}) \rVert^{2}_{A_{s,i}^{-1}} + \frac{c_{\mu}}{2} \lVert \theta_{s-1,i} - \theta_{\star} \rVert^{2}_{A_{s-1,i}} \\
    & \quad + \frac{c_{\mu}}{2} \bigl( \lVert \theta_{s-1,i} - \theta_{\star} \rVert^{2}_{A_{s,i}} - \lVert \theta_{s-1,i} - \theta_{\star} \rVert^{2}_{A_{s-1,i}}  \bigr)  - \frac{c_{\mu}}{2} \lVert \theta_{s,i}- \theta_{\star} \rVert^{2}_{A_{s,i}}  \\
    & = \frac{1}{2c_{\mu}} \lVert \nabla l(\bx_{s,i}^{\top} \theta_{s-1,i}, y_{s,i}) \rVert^{2}_{A_{s,i}^{-1}} + \frac{c_{\mu}}{2} \lVert \theta_{s-1,i} - \theta_{\star} \rVert^{2}_{A_{s-1,i}} + \frac{c_{\mu}}{2} \lVert \theta_{s-1,i} - \theta_{\star} \rVert^{2}_{\bx_{s,i} \bx_{s,i}^{\top}} - \frac{c_{\mu}}{2} \lVert \theta_{s,i} - \theta_{\star} \rVert^{2}_{A_{s,i}}
\end{align*}
\normalsize
Note that $\nabla l(\bx_{s,i}^{\top} \theta_{s-1,i}, y_{s,i})=\bx_{s,i}[\mu(\bx_{s,i}^{\top}\theta_{s-1,i})-y_{s,i}]$, so with the inequality above, we can further bound the RHS of Eq.\eqref{eq:strong_convex_ineq}:
\begin{align*}
  & l(\bx_{s,i}^{\top}\theta_{s-1,i},y_{s,i}) - l(\bx_{s,i}^{\top}\theta_{\star},y_{s,i}) \leq [\mu(\bx_{s,i}^{\top}\theta_{s-1,i})-y_{s,i}] \bx_{s,i}^{\top} (\theta_{s-1,i}-\theta_{\star}) -\frac{c_{\mu}}{2} ||\theta_{s-1,i} - \theta_{\star}||^{2}_{\bx_{s,i}\bx_{s,i}^{\top}} \\
    & \leq \frac{1}{2c_{\mu}} \lVert \nabla l(\bx_{s,i}^{\top} \theta_{s-1,i}, y_{s,i}) \rVert^{2}_{A_{s,i}^{-1}} + \frac{c_{\mu}}{2} \lVert \theta_{s-1,i} - \theta_{\star} \rVert^{2}_{A_{s-1,i}} - \frac{c_{\mu}}{2} \lVert \theta_{s,i} - \theta_{\star} \rVert^{2}_{A_{s,i}}
\end{align*}
Then summing over $s \in [t_\text{last}+1,t]$, we have
\begin{align*}
    A_{2} \leq \frac{1}{2 c_{\mu}} \sum_{s=t_\text{last}+1}^{t} \lVert \nabla l(\bx_{s,i}^{\top} \theta_{s-1,i}, y_{s,i}) \rVert^{2}_{A_{s,i}^{-1}} + \frac{c_{\mu}}{2} \lVert \theta_{t_\text{last},i} - \theta_{\star} \rVert^{2}_{A_{t_\text{last},i}} - \frac{c_{\mu}}{2} \lVert \theta_{t,i} - \theta_{\star} \rVert^{2}_{A_{t,i}} 
\end{align*}
where $A_{t_\text{last},i}=A_{t_\text{last}}, \theta_{t_\text{last},i}=\theta_{t_\text{last}}, \forall i \in [N]$ due to the global update (line 15 in Algorithm \ref{algo:3}).

We should note that the second term above itself essentially corresponds to a confidence ellipsoid centered at the globally updated model $\theta_{t_\text{last}}$, and its appearance in the upper bound for the loss difference (online regret) of local updates is because the local update is initialized by $\theta_{t_\text{last}}$. And based on Lemma \ref{lem:global_model_confidence_ellipsoid}, with probability at least $1-\delta$,
\begin{align*}
  \lVert \theta_{t_\text{last},i} - \theta_{\star} \rVert_{A_{t_\text{last},i}}  & \leq  2 N t_\text{last} \sqrt{\frac{2 k_{\mu} }{\lambda c_{\mu}} + \frac{2}{N t_\text{last} c_{\mu}}} \sqrt{\epsilon_{t_\text{last}}} \\
  & \quad + \frac{1}{c_{\mu}} R_{\max} \sqrt{d\log{(1+{N t_\text{last} c_{\mu}}/{d \lambda})}+2\log{({1}/{\delta})}} + \sqrt{\frac{\lambda}{c_{\mu}}}S
\end{align*}
Therefore, with probability at least $1-\delta$,
\begin{align*}
    A_{2} & \leq \frac{1}{2 c_{\mu}} \sum_{s=t_\text{last}+1}^{t} \lVert \nabla l(\bx_{s,i}^{\top} \theta_{s-1,i}, y_{s,i}) \rVert^{2}_{A_{s,i}^{-1}}  + \frac{c_{\mu}}{2} \bigl[ 2 N t_\text{last} \sqrt{\frac{2 k_{\mu} }{\lambda c_{\mu}} + \frac{2}{N t_\text{last} c_{\mu}}} \sqrt{\epsilon_{t_\text{last}}} \\
    & \quad + \frac{1}{c_{\mu}} R_{\max} \sqrt{d\log{(1+{N t_\text{last} c_{\mu}}/{d \lambda})}+2\log{({1}/{\delta})}} + \sqrt{\frac{\lambda}{c_{\mu}}}S  \bigr]^{2} :=B_{2} 
\end{align*}
which finishes the proof for Lemma \ref{lem:loss_diff}.

\end{proof}

\section{Proof of Lemma \ref{lem:confidence_ellipsoid_fedglb} and Corollary \ref{corollary:order_ellipsoid}} \label{sec:prove_confidence_ellipsoid}
\begin{proof}[Proof of Lemma \ref{lem:confidence_ellipsoid_fedglb}]
Due to $c_{\mu}$-strongly convexity of $l(z,y)$ w.r.t. $z$, we have $l(\bx_{s,i}^{\top}\theta,y_{s,i}) - l(\bx_{s,i}^{\top}\theta_{\star},y_{s,i}) \geq \bigl[ \mu(\bx_{s,i}^{\top}\theta_{\star})-y_{s,i} \bigr]\bx_{s,i}^{\top}(\theta-\theta_{\star}) + \frac{c_{\mu}}{2}
\bigl[\bx_{s,i}^{\top}(\theta-\theta_{\star})\bigr]^{2}$. Substituting this to the LHS of Eq.\eqref{eq:loss_diff_1} and Eq.\eqref{eq:loss_diff_2}, we have
\begin{align*}
 B_{1} & \geq \sum_{s=1}^{t_\text{last}} \sum_{i=1}^{N} \bigl[ l(\bx_{s,i}^{\top} \theta_{t_\text{last}}, y_{s,i}) - l(\bx_{s,i}^{\top} \theta_{\star}, y_{s,i}) \bigr] \\
    & \geq \sum_{s=1}^{t_\text{last}} \sum_{i=1}^{N} \bigl[ \mu(\bx_{s,i}^{\top}\theta_{\star})-y_{s} \bigr]\bx_{s,i}^{\top}(\theta_{t_\text{last}}-\theta_{\star}) + \frac{c_{\mu}}{2} \sum_{s=1}^{t_\text{last}} \sum_{i=1}^{N} \bigl[\bx_{s,i}^{\top}(\theta_{t_\text{last}}-\theta_{\star})\bigr]^{2} \\
 B_{2}  & \geq  \sum_{s=t_\text{last}+1}^{t} \bigl[ l(\bx_{s,i}^{\top} \theta_{s-1,i}, y_{s,i}) - l(\bx_{s,i}^{\top} \theta_{\star}, y_{s,i}) \bigr] \\
    & \geq \sum_{s=t_\text{last}+1}^{t} \bigl[ \mu(\bx_{s,i}^{\top}\theta_{\star})-y_{s} \bigr]\bx_{s,i}^{\top}(\theta_{s-1,i}-\theta_{\star}) + \frac{c_{\mu}}{2} \sum_{s=t_\text{last}+1}^{t} \bigl[\bx_{s,i}^{\top}(\theta_{s-1,i}-\theta_{\star})\bigr]^{2}
\end{align*}
By rearranging the terms, we have
\begin{align*}
     \sum_{s=1}^{t_\text{last}} \sum_{i=1}^{N} \bigl[\bx_{s,i}^{\top}(\theta_{t_\text{last}}-\theta_{\star})\bigr]^{2} &\leq  \frac{2}{c_{\mu}} B_{1} + \frac{2}{c_{\mu}}\sum_{s=1}^{t_\text{last}} \sum_{i=1}^{N} \eta_{s,i} \bx_{s,i}^{\top}(\theta_{t_\text{last}}-\theta_{\star})  \\
     \sum_{s=t_\text{last}+1}^{t} \bigl[\bx_{s,i}^{\top}(\theta_{s-1,i}-\theta_{\star})\bigr]^{2} & \leq  \frac{2}{c_{\mu}} B_{2} + \frac{2}{c_{\mu}}\sum_{s=t_\text{last}+1}^{t} \eta_{s,i} \bx_{s,i}^{\top}(\theta_{s-1,i}-\theta_{\star})
\end{align*}
where the LHS is quadratic in $\theta_{\star}$. For the RHS, we will further upper bound the second term as shown below. 

\noindent\textbf{$\bullet$ Upper Bound for $\sum_{s=t_\text{last}+1}^{t} \bigl[\bx_{s,i}^{\top}(\theta_{s-1,i}-\theta_{\star})\bigr]^{2}$}
Note that $\bx_{s,i}^{\top}(\theta_{s-1,i}-\theta_{\star})$ is $\cF_{s,i}$-measurable, and $\eta_{s,i}$ is $\cF_{s+1,i}$-measurable and conditionally $R_{max}$-sub-Gaussian. By applying Lemma \ref{lem:uniform_self_normalized} (Corollary 8 of \citep{abbasi2012online}) w.r.t. client $i$'s filtration $\{\cF_{s,i}\}_{s=t_\text{last}+1}^{\infty}$, where $\cF_{s,i}=\sigma\bigl( [\bx_{k,j}, \eta_{k,j}]_{k,j: k\leq t_\text{last} \cap j \leq N}, [\bx_{k,j}, \eta_{k,j}]_{k,j: t_\text{last}+1 \leq k \leq s-1 \cap j=i}, \bx_{s,i} \bigr)$, and taking union bound over all $i \in [N]$, with probability at least $1-\delta$, for all $t \in [T], i\in [N]$,
\small
\begin{align*}
     \sum_{s=t_\text{last}+1}^{t} \eta_{s,i} \bx_{s,i}^{\top}&(\theta_{s-1,i}-\theta_{\star})  \leq \\
    & \quad R_{\max} \sqrt{2\bigl(1+ \sum_{s=t_\text{last}+1}^{t} \bigl[\bx_{s,i}^{\top}(\theta_{s-1,i}-\theta_{\star})\bigr]^{2}  \bigr) \cdot \log\bigl( \frac{N}{\delta} \sqrt{1+ \sum_{s=t_\text{last}+1}^{t} \bigl[\bx_{s,i}^{\top}(\theta_{s-1,i}-\theta_{\star})\bigr]^{2}} \bigr)}
\end{align*}
\normalsize
Therefore, 
\small
\begin{equation}
\begin{split}
    & 1+\sum_{s=t_\text{last}+1}^{t} \bigl[\bx_{s,i}^{\top}(\theta_{s-1,i}-\theta_{\star})\bigr]^{2} \leq  1+\frac{2}{c_{\mu}} B_{2} \\
    & + \frac{2R_{\max}}{c_{\mu}} \sqrt{2\bigl(1+ \sum_{s=t_\text{last}+1}^{t} \bigl[\bx_{s,i}^{\top}(\theta_{s-1,i}-\theta_{\star})\bigr]^{2}  \bigr) \cdot \log\bigl( \frac{N}{\delta} \sqrt{1+ \sum_{s=t_\text{last}+1}^{t} \bigl[\bx_{s,i}^{\top}(\theta_{s-1,i}-\theta_{\star})\bigr]^{2}} \bigr)}
\end{split}
\end{equation}
\normalsize
Then by applying Lemma 2 of \citep{jun2017scalable}, i.e., if $q^{2} \leq a + fq \sqrt{\log(\frac{q}{\delta/N})}$ then $q^{2} \leq 2a + f^{2} \log(\frac{\sqrt{4a+f^{4}/(4\delta^{2})}}{\delta/N})$ (for $a,f \geq 0, q\geq 1$). And by setting $q=\sqrt{1+\sum_{s=t_\text{last}+1}^{t} \bigl[\bx_{s,i}^{\top}(\theta_{s-1,i}-\theta_{\star})\bigr]^{2}}$, $a=1+\frac{2}{c_{\mu}}B_{2},f=\frac{2\sqrt{2}R_{\max}}{c_{\mu}}$, we have
\begin{equation}
    \sum_{s=t_\text{last}+1}^{t} \bigl[\bx_{s,i}^{\top}(\theta_{s-1,i}-\theta_{\star})\bigr]^{2} \leq 1+ \frac{4 B_{2}}{c_{\mu}} + \frac{8 R_{\max}^{2}}{c_{\mu}^{2}} \log\Bigg( \frac{N}{\delta} \sqrt{4+\frac{8}{c_{\mu}}B_{2} + \frac{64 R_{\max}^{4}}{c_{\mu}^{4}\cdot 4 \delta^{2}}}  \Bigg), \forall t,i
\end{equation}
with probability at least $1-\delta$.

\noindent\textbf{$\bullet$ Upper Bound for $\sum_{s=1}^{t_\text{last}} \sum_{i=1}^{N} \bigl[\bx_{s,i}^{\top}(\theta_{t_\text{last}}-\theta_{\star})\bigr]^{2}$}
Note that $\theta_{t_\text{last}}$ depends on all data samples in $\{(\bx_{s,i}, y_{s,i})\}_{s \in [t_\text{last}]}$ as a result of the offline regression method, and therefore $\bx_{s,i}^{\top}(\theta_{t_\text{last}}-\theta_{\star})$ is no longer $\cF_{s,i}$-measurable for $s \in [1,t_\text{last})$. Hence, we cannot use Lemma \ref{lem:uniform_self_normalized} as before. Instead, we have
\small
\begin{align*}
    & \sum_{s=1}^{t_\text{last}} \sum_{i=1}^{N} \eta_{s,i} \bx_{s,i}^{\top}(\theta_{t_\text{last}}-\theta_{\star}) = \bigl( \sum_{s=1}^{t_\text{last}} \sum_{i=1}^{N} \eta_{s,i} \bx_{s,i} \bigr)^{\top} (\theta_{t_\text{last}}-\theta_{\star}) \\
    & = \bigl( \sum_{s=1}^{t_\text{last}} \sum_{i=1}^{N} \eta_{s,i} \bx_{s,i} \bigr)^{\top} (I + \sum_{s=1}^{t_\text{last}} \sum_{i=1}^{N} \bx_{s,i}\bx_{s,i}^{\top})^{-1}(I + \sum_{s=1}^{t_\text{last}} \sum_{i=1}^{N} \bx_{s,i}\bx_{s,i}^{\top}) (\theta_{t_\text{last}}-\theta_{\star})  \\
    & \leq \sqrt{\bigl( \sum_{s=1}^{t_\text{last}} \sum_{i=1}^{N} \eta_{s,i} \bx_{s,i} \bigr)^{\top} (I + \sum_{s=1}^{t_\text{last}} \sum_{i=1}^{N} \bx_{s,i}\bx_{s,i}^{\top})^{-1} \bigl( \sum_{s=1}^{t_\text{last}} \sum_{i=1}^{N} \eta_{s,i} \bx_{s,i} \bigr) \cdot (\theta_{t_\text{last}}-\theta_{\star})^{\top} (I + \sum_{s=1}^{t_\text{last}} \sum_{i=1}^{N} \bx_{s,i}\bx_{s,i}^{\top}) (\theta_{t_\text{last}}-\theta_{\star})}  \\
    & = \sqrt{\lVert \sum_{s=1}^{t_\text{last}} \sum_{i=1}^{N} \eta_{s,i} \bx_{s,i} \rVert_{(I + \sum_{s=1}^{t_\text{last}} \sum_{i=1}^{N} \bx_{s,i}\bx_{s,i}^{\top})^{-1}}^{2} \cdot \lVert \theta_{t_\text{last}}-\theta_{\star} \rVert_{(I + \sum_{s=1}^{t_\text{last}} \sum_{i=1}^{N} \bx_{s,i}\bx_{s,i}^{\top})}^{2} }  \\
    & \leq R_{\max} \sqrt{ 2 \log\bigl(\frac{1}{\delta} \sqrt{\det(I + \sum_{s=1}^{t_\text{last}} \sum_{i=1}^{N} \bx_{s,i}\bx_{s,i}^{\top})} \bigr) \cdot \lVert \theta_{t_\text{last}}-\theta_{\star} \rVert_{(I + \sum_{s=1}^{t_\text{last}} \sum_{i=1}^{N} \bx_{s,i}\bx_{s,i}^{\top})}^{2} },
\end{align*}
\normalsize
with probability at least $1-\delta$, where the first inequality is due to the matrix-weighted Cauchy-Schwarz inequality in Lemma \ref{lem:matrix_weighted_cauchy_schwartz}, such that $x^{\top} A^{-1} A y \leq \sqrt{x^{\top}A^{-1}x \cdot y^{\top} A^{\top} A^{-1} A y} = \sqrt{x^{\top}A^{-1}x \cdot y^{\top} A y}$ for symmetric PD matrix $A$, and the second inequality is obtained by applying the self-normalized bound in Lemma \ref{lem:self_normalized_bound} w.r.t. the filtration $\{\cF_{s}\}_{s \in \{t_{p}\}_{p=1}^{B}}$, where $\cF_{s}=\sigma\bigl( [\bx_{k,j}, \eta_{k,j}]_{k,j:k \leq s-1 \cap j \leq N}, [\bx_{k,j}, \eta_{k,j}]_{k,j:k = s \cap j \leq N-1}, \bx_{s,N} \bigr)$ and $\{t_{p}\}_{p=1}^{B}$ denotes the sequence of time steps when global update happens, and $B$ denotes the total number of global updates.

By substituting it back, we have
\small
\begin{equation}
\begin{split}
    & \sum_{s=1}^{t_\text{last}} \sum_{i=1}^{N} \bigl[\bx_{s,i}^{\top}(\theta_{t_\text{last}}-\theta_{\star})\bigr]^{2}  \\
    & \leq  \frac{2}{c_{\mu}} B_{1} + \frac{2R_{\max}}{c_{\mu}} \sqrt{ 2 \log\bigl(\frac{1}{\delta} \sqrt{\det(I + \sum_{s=1}^{t_\text{last}} \sum_{i=1}^{N} \bx_{s,i}\bx_{s,i}^{\top})} \bigr) \cdot \lVert \theta_{t_\text{last}}-\theta_{\star} \rVert_{I + \sum_{s=1}^{t_\text{last}} \sum_{i=1}^{N} \bx_{s,i}\bx_{s,i}^{\top}}^{2} }  \\
    & \leq \frac{2}{c_{\mu}} B_{1} + \frac{2R_{\max}}{c_{\mu}} \sqrt{ 2 \log\bigl(\frac{1}{\delta} \sqrt{\det(I + \sum_{s=1}^{t_\text{last}} \sum_{i=1}^{N} \bx_{s,i}\bx_{s,i}^{\top})} \bigr) \cdot \bigl( \sum_{s=1}^{t_\text{last}} \sum_{i=1}^{N} \bigl[\bx_{s,i}^{\top}(\theta_{t_\text{last}}-\theta_{\star})\bigr]^{2} + \lVert \theta_{t_\text{last}}-\theta_{\star} \rVert_{2}^{2}  \bigr) }
\end{split}
\end{equation}
\normalsize
Then by applying the Proposition 9 of \citep{abbasi2012online}, i.e. if $z^{2} \leq a + bz$ then $z \leq b + \sqrt{a}$ (for $a,b \geq 0$), and setting $z=\sqrt{\sum_{s=1}^{t_\text{last}} \sum_{i=1}^{N} \bigl[\bx_{s,i}^{\top}(\theta_{t_\text{last}}-\theta_{\star})\bigr]^{2} + \lVert \theta_{t_\text{last}}-\theta_{\star} \rVert_{2}^{2} }, a=\lVert \theta_{t_\text{last}}-\theta_{\star} \rVert_{2}^{2}+\frac{2}{c_{\mu}}B_{1},b=\frac{2R_{\max}}{c_{\mu}} \sqrt{ 2 \log\bigl(\frac{1}{\delta} \sqrt{\det(I + \sum_{s=1}^{t_\text{last}} \sum_{i=1}^{N} \bx_{s,i}\bx_{s,i}^{\top})} \bigr)}$,we have
\begin{equation}
\begin{split}
    &\sqrt{\sum_{s=1}^{t_\text{last}} \sum_{i=1}^{N} \bigl[\bx_{s,i}^{\top}(\theta_{t_\text{last}}-\theta_{\star})\bigr]^{2} + \lVert \theta_{t_\text{last}}-\theta_{\star} \rVert_{2}^{2} }  \\
    \leq &\frac{2R_{\max}}{c_{\mu}} \sqrt{2 \log\bigl(\frac{1}{\delta} \sqrt{\det(I + \sum_{s=1}^{t_\text{last}} \sum_{i=1}^{N} \bx_{s,i}\bx_{s,i}^{\top})} \bigr)} 
    + \sqrt{\lVert \theta_{t_\text{last}}-\theta_{\star} \rVert_{2}^{2} + B_{1}}
\end{split}
\end{equation}
Taking square on both sides, and rearranging terms, we have
\small
\begin{equation}
\begin{split}
     &\sum_{s=1}^{t_\text{last}} \sum_{i=1}^{N} \bigl[\bx_{s,i}^{\top}(\theta_{t_\text{last}}-\theta_{\star})\bigr]^{2} \\
      \leq&  \frac{8R_{\max}^{2}}{c_{\mu}^{2}} \log\bigl(\frac{1}{\delta} \sqrt{\det(I + \sum_{s=1}^{t_\text{last}} \sum_{i=1}^{N} \bx_{s,i}\bx_{s,i}^{\top})} \bigr) + B_{1} \\
    & + \frac{4R_{\max}}{c_{\mu}} \sqrt{2 \log\bigl(\frac{1}{\delta} \sqrt{\det(I + \sum_{s=1}^{t_\text{last}} \sum_{i=1}^{N} \bx_{s,i}\bx_{s,i}^{\top})} \bigr)}\sqrt{\lVert \theta_{t_\text{last}}-\theta_{\star} \rVert_{2}^{2} + B_{1}} \\
     \leq& \frac{8R_{\max}^{2}}{c_{\mu}^{2}} \log\bigl(\frac{1}{\delta} \sqrt{\det(I + \sum_{s=1}^{t_\text{last}} \sum_{i=1}^{N} \bx_{s,i}\bx_{s,i}^{\top})} \bigr) + B_{1} \\
    & + \frac{4R_{\max}}{c_{\mu}} \sqrt{2 \log\bigl(\frac{1}{\delta} \sqrt{\det(I + \sum_{s=1}^{t_\text{last}} \sum_{i=1}^{N} \bx_{s,i}\bx_{s,i}^{\top})} \bigr)} ( \lVert \theta_{t_\text{last}}-\theta_{\star} \rVert_{2} + \sqrt{B_{1}} ) \\
     \leq &\frac{8R_{\max}^{2}}{c_{\mu}^{2}} \log\bigl(\frac{1}{\delta} \sqrt{\det(I + \sum_{s=1}^{t_\text{last}} \sum_{i=1}^{N} \bx_{s,i}\bx_{s,i}^{\top})} \bigr) + B_{1} \\
    & + \frac{4R_{\max}}{c_{\mu}} \sqrt{2 \log\bigl(\frac{1}{\delta} \sqrt{\det(I + \sum_{s=1}^{t_\text{last}} \sum_{i=1}^{N} \bx_{s,i}\bx_{s,i}^{\top})} \bigr)} ( \lVert \theta_{t_\text{last}} \rVert_{2}+ \lVert\theta_{\star} \rVert_{2} + \sqrt{B_{1}} )
\end{split}
\end{equation}
\normalsize

Now putting everything together, we have the following confidence region for $\theta_{\star}$,
\begin{equation}
\begin{split}
    & P\bigl(\forall t,i, \sum_{s=1}^{t_\text{last}} \sum_{i=1}^{N} \bigl[\bx_{s,i}^{\top}(\theta_{t_\text{last}}-\theta_{\star})\bigr]^{2} + \sum_{s=t_\text{last}+1}^{t} \bigl[\bx_{s,i}^{\top}(\theta_{s-1,i}-\theta_{\star})\bigr]^{2}  \leq \beta_{t,i} \bigr) \geq 1-2\delta
\end{split}
\end{equation}
where $\beta_{t,i} = \frac{8R_{\max}^{2}}{c_{\mu}^{2}} \log\bigl(\frac{1}{\delta} \sqrt{\det(I + \sum_{s=1}^{t_\text{last}} \sum_{i=1}^{N} \bx_{s,i}\bx_{s,i}^{\top})} \bigr) + B_{1} + \frac{4R_{\max}}{c_{\mu}} \sqrt{2 \log\bigl(\frac{1}{\delta} \sqrt{\det(I + \sum_{s=1}^{t_\text{last}} \sum_{i=1}^{N} \bx_{s,i}\bx_{s,i}^{\top})} \bigr)} ( \lVert \theta_{t_\text{last}} \rVert_{2}+ \lVert\theta_{\star} \rVert_{2} + \sqrt{B_{1}} )  + 1+ \frac{4 B_{2}}{c_{\mu}} + \frac{8 R_{\max}^{2}}{c_{\mu}^{2}} \log\bigl( \frac{N}{\delta} \sqrt{4+\frac{8}{c_{\mu}}B_{2} + \frac{64 R_{\max}^{4}}{c_{\mu}^{4}\cdot 4 \delta^{2}}}  \bigr)$.
Denote $\textbf{X}_{t,i} = \begin{bmatrix} \bx_{1,1}^{\top} \\ \dots \\ \bx_{t_\text{last},N}^{\top} \\ \bx_{i,t_\text{last}+1}^{\top}  \\ \dots \\ \bx_{i,t}^{\top} \end{bmatrix} \in \bR^{(Nt_\text{last}+t-t_\text{last}) \times d}$, and $\textbf{z}_{t,i} = \begin{bmatrix} \bx_{1,1}^{\top}\theta_{t_\text{last}} \\ \dots \\ \bx_{t_\text{last},N}^{\top} \theta_{t_\text{last}} \\ \bx_{i,t_\text{last}+1}^{\top} \theta_{t_\text{last},i}  \\ \dots \\ \bx_{i,t}^{\top} \theta_{t-1,i} \end{bmatrix} \in \bR^{Nt_\text{last}+t-t_\text{last}}$.
We can rewrite the inequality above as 
\begin{align*}
    & \lVert \textbf{z}_{t,i} - \textbf{X}_{t,i} \theta_{\star} \rVert_{2}^{2} + \frac{\lambda}{c_{\mu}} \lVert \theta_{\star} \rVert_{2}^{2} \leq \beta_{t,i} + \frac{\lambda}{c_{\mu}} \lVert \theta_{\star} \rVert_{2}^{2} \leq \beta_{t,i} + \frac{\lambda}{c_{\mu}} S^{2}  \\
    \Leftrightarrow	 & \lVert \textbf{z}_{t,i} - \textbf{X}_{t,i} \theta_{\star} \rVert_{2}^{2} + \frac{\lambda}{c_{\mu}} \lVert \theta_{\star} \rVert_{2}^{2} - \lVert \textbf{z}_{t,i} - \textbf{X}_{t,i} \hat{\theta}_{t,i} \rVert_{2}^{2} - \frac{\lambda}{c_{\mu}} \lVert \hat{\theta}_{t,i} \rVert_{2}^{2} + \lVert \textbf{z}_{t,i} - \textbf{X}_{t,i} \hat{\theta}_{t,i} \rVert_{2}^{2} + \frac{\lambda}{c_{\mu}} \lVert \hat{\theta}_{t,i} \rVert_{2}^{2} \\
    & \leq \beta_{t,i} + \frac{\lambda}{c_{\mu}} S^{2} 
\end{align*}
where $\hat{\theta}_{t,i}=A_{t,i}^{-1} \textbf{X}_{t,i}^{\top} \textbf{z}_{t,i}$ denotes the Ridge regression estimator based on the predicted rewards given by the past sequence of model updates, and the regularization parameter is $\frac{\lambda}{c_{\mu}}$.
Note that by expanding $\hat{\theta}_{t,i}$, we can show $\hat{\theta}_{t,i}^{\top} A_{i,t} \hat{\theta}_{t,i}=\textbf{z}_{i,t}^{\top} \textbf{X}_{i,t} \hat{\theta}_{t,i}$, and $\hat{\theta}_{t,i}^{\top} A_{i,t} \theta_{\star}=\textbf{z}_{i,t}^{\top} \textbf{X}_{i,t} \theta_{\star}$.
Therefore, we have 
\begin{align*}
    & \lVert \hat{\theta}_{t,i}- \theta_{\star} \rVert_{A_{t,i}}^{2} \leq \beta_{t,i} + \frac{\lambda}{c_{\mu}} S^{2} - (\lVert \textbf{z}_{t,i} \rVert_{2}^{2} - 
    \hat{\theta}_{t,i}^{\top} \textbf{X}_{t,i}^{\top} \textbf{z}_{t,i})
\end{align*}
which finishes the proof of Lemma \ref{lem:confidence_ellipsoid_fedglb}.
\end{proof}

\begin{proof}[Proof of Corollary \ref{corollary:order_ellipsoid}]
Under the condition that $\epsilon_{t_\text{last}} \leq \frac{1}{N^{2}t_\text{last}^{2}}$, 
\begin{align*}
    & B_{1} \leq \frac{1}{N t_\text{last}} + \frac{\lambda}{2}S^{2} =O(1) \\
    & B_{2} \leq \frac{1}{2 c_{\mu}} \sum_{s=t_\text{last}+1}^{t} \lVert \nabla l(\bx_{s,i}^{\top} \theta_{s-1,i}, y_{s,i}) \rVert^{2}_{A_{s,i}^{-1}}  \\
    & \quad \quad + \frac{c_{\mu}}{2} \bigl[ 2 \sqrt{\frac{2 k_{\mu} }{\lambda c_{\mu}} + \frac{2}{N t_\text{last} c_{\mu}}} + \frac{1}{c_{\mu}} R_{\max} \sqrt{d\log{(1+{N t_\text{last} c_{\mu}}/{d \lambda})}+2\log{({1}/{\delta})}} + \sqrt{\frac{\lambda}{c_{\mu}}}S  \bigr]^{2}
\end{align*}
Note that $\nabla l(\bx_{s,i}^{\top} \theta_{s-1,i}, y_{s,i}) = \bx_{s,i}[\mu(\bx_{s,i}^{\top} \theta_{s-1,i}) - y_{s,i}]$. We can upper bound the squared prediction error by
\begin{align*}
     &\bigl[ \mu(\bx_{s,i}^{\top}\theta_{s-1,i})-y_{s,i} \bigr]^{2} \\
     &= \bigl[ \mu(\bx_{s,i}^{\top}\theta_{s-1,i})-\mu(\bx_{s,i}^{\top}\theta_{\star})- \eta_{s,i} \bigr]^{2} \\
    & \leq 2\bigl[ \mu(\bx_{s,i}^{\top}\theta_{s-1,i})-\mu(\bx_{s,i}^{\top}\theta_{\star})\bigr]^{2} + 2\eta_{s,i}^{2} \\
    & \leq 2k_{\mu}^{2}\bigl[ \bx_{s,i}^{\top} (\theta_{s-1,i}-\theta_{\star}) \bigr]^{2} + 2\eta_{s,i}^{2} \\
    &\leq 8k_{\mu}^{2} S^{2} + 2\eta_{s,i}^{2}
\end{align*}
where the first inequality is due to AM-QM inequality, and the second inequality is due to the $k_{\mu}$-Lipschitz continuity of $\mu(\cdot)$ according to Assumption \ref{assump:1}.
Since $|\eta_{s,i}| \leq R_{\max}$, $\bigl[ \mu(\bx_{s,i}^{\top}\theta_{s-1,i})-y_{s,i} \bigr]^{2} \leq k_{\mu}^{2} S^{2}+R_{\max}^{2}$. In addition, due to Lemma 11 of \citep{abbasi2011improved}, i.e., $\sum_{s=t_\text{last}+1}^{t} \lVert \bx_{s,i} \rVert^{2}_{A_{s,i}^{-1}} \leq 2 \log(\frac{\det(A_{t,i})}{\det(\lambda I)})$
Therefore,
\begin{align*}
    \frac{1}{2c_{\mu}}\sum_{s=t_\text{last}+1}^{t} \lVert \nabla l(\bx_{s,i}^{\top} \theta_{s-1,i}, y_{s,i}) \rVert^{2}_{A_{s,i}^{-1}} =  O\bigl(\frac{d\log{NT}}{c_{\mu}}[k_{\mu}^{2} S^{2}+R_{\max}^{2}] \bigr)
\end{align*}
so $B_{2}=O\bigl(\frac{d\log{NT}}{c_{\mu}}[k_{\mu}^{2} S^{2}+R_{\max}^{2}] \bigr)$.
Hence,
\begin{align*}
    \beta_{t,i} = O(d\frac{R_{\max}^{2}}{c_{\mu}^{2}} \log{NT} + d\frac{k_{\mu}^{2}}{c_{\mu}^{2}}\log{NT} + d\frac{R_{\max}^{2}}{c_{\mu}^{2}}\log{NT})=O(\frac{d\log{NT}}{c_{\mu}^{2}}[k_{\mu}^{2}+R_{\max}^{2}])
\end{align*}
which finishes the proof.
\end{proof}

\section{Proof of Theorem \ref{thm:regret_comm_upper_bound}} \label{sec:prove_regret_comm}
\begin{proof}
Since $\mu(\cdot)$ is $k_{\mu}$-Lipschitz continuous, we have $\mu(\bx_{t,\star}^{\top} \theta_{\star}) - \mu(\bx_{t,i}^{\top} \theta_{\star}) \leq k_{\mu} (\bx_{t,\star}^{\top} \theta_{\star}- \bx_{t,i}^{\top} \theta_{\star})$. Then we have the following upper bound on the instantaneous regret,
\begin{align*}
    \frac{r_{t,i}}{k_{\mu}} & \leq \bx_{t,\star}^{\top} \theta_{\star}- \bx_{t,i}^{\top} \theta_{\star} \leq \bx_{t,i}^{\top} \tilde{\theta}_{t-1,i}- \bx_{t,i}^{\top} \theta_{\star} \\
    & = \bx_{t,i}^{\top} (\tilde{\theta}_{t-1,i} - \hat{\theta}_{t-1,i}) + \bx_{t,i}^{\top} (\hat{\theta}_{t-1,i} - \theta_{\star})  \\
    & \leq \lVert \bx_{t,i} \rVert_{A_{t-1,i}^{-1}} \lVert \tilde{\theta}_{t-1,i} - \hat{\theta}_{t-1,i} \rVert_{A_{t-1,i}} + \lVert \bx_{t,i} \rVert_{A_{t-1,i}^{-1}} \lVert \hat{\theta}_{t-1,i} - \theta_{\star} \rVert_{A_{t-1,i}} \\
    & \leq 2 \alpha_{t-1,i} \cdot \lVert \bx_{t,i} \rVert_{A_{t-1,i}^{-1}}
\end{align*}
which holds for all $i \in [N], t \in [T]$, with probability at least $1-2\delta$. And $\tilde{\theta}_{t-1.i}$ denotes the optimistic estimate in the confidence ellipsoid that maximizes the UCB score when client $i$ selects arm at time step $t$.

Now consider an imaginary centralized agent that has direct access to all clients' data, and we denote its covariance matrix as $\tilde{A}_{t,i}=\frac{\lambda}{c_{\mu}} I + \sum_{s=1}^{t-1} \sum_{j=1}^{N} \bx_{s,j}\bx_{s,j} + \sum_{j=1}^{i} \bx_{t,j} \bx_{t,j}^{\top}$, i.e., $\tilde{A}_{t,i}$ is immediately updated after any client obtains a new data sample from the environment. Then we can obtain the following upper bound for $r_{t,i}$, which is dependent on the determinant ratio between the covariance matrix of the imaginary centralized agent and that of client $i$, i.e., $\det(\tilde{A}_{t-1,i})/\det(A_{t-1,i})$.
\begin{align*}
    r_{t,i} \leq 2 k_{\mu}\alpha_{t-1,i} \sqrt{\bx_{t,i}^{\top} A^{-1}_{t-1,i} \bx_{t,i}} \leq 2 k_{\mu}\alpha_{t-1,i} \sqrt{\bx_{t,i}^{\top} \tilde{A}_{t-1,i}^{-1} \bx_{t,i} \cdot \frac{\det(\tilde{A}_{t-1,i})}{\det(A_{t-1,i})}} 
\end{align*}
We refer to the time period in-between two consecutive global updates as an epoch, and denote the total number of epochs as $B \in \bR$, i.e., the $p$-th epoch refers to the period from $t_{p-1}+1$ to $t_{p}$, for $p \in [B]$, where $t_{p}$ denotes the time step when the $p$-th global update happens.
Then the $p$-th epoch is called a `good' epoch if the determinant ratio $\frac{\det(A_{t_{p}})}{\det(A_{t_{p-1}})} \leq 2$, where $A_{t_{p}}$ is the aggregated sufficient statistics computed at the $p$-th global update.
Otherwise, it is called a `bad' epoch. In the following, we bound the cumulative regret in `good' and `bad' epochs separately.

Suppose the $p$-th epoch is a good epoch, then for any client $i \in [N]$, and time step $t \in [t_{p-1}+1, t_{p}]$, we have $\frac{\det(\tilde{A}_{t-1,i})}{\det(A_{t-1,i})} \leq \frac{\det(A_{t_{p}})}{\det(A_{t_{p-1}})} \leq 2$, because $A_{t-1,i} \succcurlyeq A_{t_{p-1}}$ and $\tilde{A}_{t-1,i} \preccurlyeq A_{t_{p}}$. Therefore, the instantaneous regret incurred by any client $i$ at any time step $t$ of a good epoch can be bounded by 
\begin{align*}
    r_{t,i} \leq 2 \sqrt{2} k_{\mu}\alpha_{t-1,i} \sqrt{\bx_{t,i}^{\top} \tilde{A}_{t-1,i}^{-1} \bx_{t,i}} 
\end{align*} 
with probability at least $1-2\delta$.
Therefore, using standard arguments for UCB-type algorithms, e.g., Theorem 2 in \citep{li2017provably}, the cumulative regret for all the `good epochs' is 
\begin{align*}
    REG_{good} & \leq 2 \sqrt{2} k_{\mu}\alpha_{t-1,i} \sum_{t=1}^{T} \sum_{i=1}^{N} \lVert \bx_{t,i} \rVert_{\tilde{A}_{t-1,i}^{-1}} \\
    & = O\left(\frac{k_{\mu}(k_{\mu}+R_{\max})}{c_{\mu}}d\sqrt{NT}\log{NT}\right)
\end{align*}
which matches the regret upper bound of GLOC \citep{jun2017scalable}. 

Now suppose the $p$-th epoch is bad. Then the cumulative regret incurred by all $N$ clients during this `bad epoch' can be upper bounded by:
\begin{align*}
    &\sum_{t=t_{p-1}+1}^{t_{p}} \sum_{i=1}^{N} r_{t,i} \\
    & \leq O(\frac{k_{\mu}(k_{\mu}+R_{\max})}{c_{\mu}}\sqrt{d\log(NT)})\sum_{t=t_{p-1}+1}^{t_{p}}\sum_{i=1}^{N}\min(1,||\bx_{t,i}||_{A_{t-1,i}^{-1}})  \\
    & \leq  O(\frac{k_{\mu}(k_{\mu}+R_{\max})}{c_{\mu}}\sqrt{d\log(NT)}) \sum_{i=1}^{N} \sqrt{(t_{p}-t_{p-1})\log{\frac{\det(A_{t_{p}-1,i})}{\det(A_{t_{p}-1,i}-\Delta{A}_{t_{p}-1,i})}}} \\
    & \leq O(\frac{k_{\mu}(k_{\mu}+R_{\max})}{c_{\mu}}N\sqrt{d\log{(NT)}D})
\end{align*}
where the last inequality is due to the event-trigger design in Algorithm \ref{algo:3}.
Following the same argument as \citep{wang2019distributed}, there can be at most $R = \lceil d \log{(1+\frac{NT c_{\mu}}{\lambda d})} \rceil = O\bigl( d\log(NT) \bigr)$ `bad epochs', because $\det(A_{t_{B}}) \leq \det(\tilde{A}_{T,N}) \leq (\frac{\lambda}{c_{\mu}} + \frac{NT}{d})^{d}$. Therefore, the cumulative regret for all the `bad epochs' is
\begin{align*}
    REG_{bad}=O\left(\frac{k_{\mu}(k_{\mu}+R_{\max})}{c_{\mu}}d^{1.5}\log^{1.5}{(NT)}ND^{0.5}\right)
\end{align*}
Combining the regret upper bound for `good' and `bad' epochs, the cumulative regret $$R_{T}=O\left(\frac{k_{\mu}(k_{\mu}+R_{\max})}{c_{\mu}}(d\sqrt{NT}\log(NT)+d^{1.5}\log^{1.5}{(NT)}ND^{0.5})\right).$$
To obtain upper bound for the communication cost $C_{T}$, we first upper bound the total number of epochs $B$.
Denote the length of an epoch, i.e., the number of time steps between two consecutive global updates, as $\alpha > 0$, so that there can be at most $\lceil \frac{T}{\alpha}\rceil$ epochs with length longer than $\alpha$.
For a particular epoch $p$ with less than $\alpha$ time steps, we have
$t_{p}-t_{p-1} < \alpha$. 
Moreover, due to the event-trigger design in Algorithm \ref{algo:3}, we have $(t_{p}-t_{p-1})\log{\frac{\det(A_{t_{p}})}{\det(A_{t_{p-1}})}} > D$, which means $\log{\frac{\det(A_{t_{p}})}{\det(A_{t_{p-1}})}} > \frac{D}{\alpha}$. Since $\sum_{p=1}^{B} \log{\frac{\det(A_{t_{p}})}{\det(A_{t_{p-1}})}} \leq R$, the number of epochs with less than $\alpha$ time steps is at most $\lceil \frac{R \alpha}{D}\rceil$. Therefore, the total number of epochs. 
\begin{align*}
    B \leq \lceil \frac{T}{\alpha}\rceil+\lceil \frac{R \alpha}{D}\rceil
\end{align*}
which is minimized it by choosing $\alpha=\sqrt{\frac{D T}{R}}$, so $B \leq \sqrt{\frac{T R}{D}}=O(d^{0.5}\log^{0.5}(NT)T^{0.5}D^{-0.5})$. 

At the end of each epoch, \algone{} has a global update step that executes AGD among all $N$ clients. 
As mentioned in Section \ref{subsec:algo_description}, the number of iterations required by AGD has upper bound
\begin{align*}
    J_{t} \leq 1+\sqrt{\frac{k_{\mu}}{\lambda}Nt+1}\log{\frac{(k_{\mu}+\frac{2\lambda}{N t})\lVert \theta_{t}^{(1)}-\hat{\theta}_{t}^{\text{MLE}} \rVert_{2}^{2}}{2\epsilon_{t}}},
\end{align*}
and under the condition that $\epsilon_{t}=\frac{1}{N^{2}t^{2}},\forall t\in [T]$, we have $J_{t}=O\bigl( \sqrt{NT} \log(NT) \bigr), \forall t \in [T]$. Moreover, each iteration of AGD involves communication with $N$ clients,
so the communication cost
\begin{align*}
    C_{T}=O(d^{0.5}\log^{1.5}(NT)TN^{1.5}D^{-0.5})
\end{align*}

In order to match the regret under centralized setting, we set the threshold $D=\frac{T}{Nd\log(NT)}$, which gives us $R_{T}=O(\frac{k_{\mu}(k_{\mu}+R_{\max})}{c_{\mu}}d\sqrt{NT}\log(NT))$, and $C_{T}=O(dN^{2}\sqrt{T}\log^{2}(NT))$. 
\end{proof}

\section{Theoretical Analysis for Variants of \algone{}} \label{sec:prove_variants}
In this section, we describe and analyze the variants of \algone{} listed in Table \ref{tb:theoretical_comparison}.
The first variant, \algone{}$_{1}$, completely disables local update, and
we can see that it requires a linear communication cost in $T$ to attain the $O(d\sqrt{NT}\log(NT))$ regret. As we mentioned in Section \ref{subsec:algo_description}, this is because in the absence of local update, \algone{}$_{1}$ requires more frequent global updates, i.e., $\sqrt{NT}$ in total, to control the sub-optimality of the employed bandit model w.r.t the growing training set.
The second variant, denoted as \algone{}$_{2}$, is exactly the same as \algone{}, except for its fixed communication schedule. This leads to additional $d \sqrt{N}$ global updates, as fixed update schedule cannot adapt to the actual quality of collected data.
The third variant, denoted as \algone{}$_{3}$, uses ONS for both local and global update, such that only one round of gradient aggregation among $N$ clients is performed for each global update, i.e., lazy ONS update over batched data. 
It incurs the least communication cost among all variants, but its regret grows at a rate of $(NT)^{3/4}$ due to the inferior quality of its lazy ONS update.

\subsection{\algone{}$_{1}$: scheduled communication + no local update}
Though many real-world applications are online problems in nature, i.e., the clients continuously collect new data samples from the users, standard federated/distributed learning methods do not provide a principled solution to adapt to the growing datasets.
A common practice is to manually set a fixed global update schedule in advance, i.e., periodically update and deploy the model. 

To demonstrate the advantage of \algone{} over this straightforward solution, we present and analyze the first variant \algone{}$_{1}$, which completely disables local update, and
performs global update
according to a fixed schedule 
$\cS=\{t_{1}:=\lfloor\frac{T}{B}\rfloor,t_{2}:=2\lfloor \frac{T}{B}\rfloor,\dots,t_{B}:=B\lfloor\frac{T}{B}\rfloor\}$, where $B$ is the total number of global updates up to time step $T$. 
The description of \algone{}$_{1}$ is presented in Algorithm \ref{algo:fedglb_variant_1}.

\begin{algorithm}[h]
    \caption{\algone{}$_{1}$} \label{algo:fedglb_variant_1}
  \begin{algorithmic}[1]
    \STATE \textbf{Input:} communication schedule $\cS$, regularization parameter $\lambda>0$, $\delta \in (0,1)$ and $c_{\mu}$.
    \STATE \textbf{Initialize} $\forall i\in [N]$: $\theta_{0,i}=\textbf{0}\in \mathbb{R}^{d}, {A}_{0,i}=\frac{\lambda}{c_{\mu}} \textbf{I} \in \mathbb{R}^{d \times d}, \textbf{X}_{0,i}=\textbf{0} \in \mathbb{R}^{0 \times d}, \textbf{y}_{0,i}=\textbf{0} \in \mathbb{R}^{0}$, $t_\text{last}=0$
    \FOR{ $t=1,2,...,T$}
        \FOR{client $i=1,2,...,N$}
            \STATE Observe arm set $\mathcal{A}_{t,i}$ for client $i$
            \STATE Select arm $x_{t,i}\in\mathcal{A}_{t,i}$ according to Eq. \eqref{eq:UCB_1} and observe reward $y_{t,i}$
            \STATE Update client $i$: $\textbf{X}_{t,i}=\begin{bmatrix} \textbf{X}_{t-1,i} \\ \bx_{t,i}^{\top} \end{bmatrix}$ ,$ \textbf{y}_{t,i}=\begin{bmatrix} \textbf{y}_{t-1,i} \\ y_{t,i} \end{bmatrix}$
        \ENDFOR 
        \IF{$t \notin \cS$}
            \STATE \textbf{Clients}: set $\theta_{t,i}=\theta_{t-1,i}, A_{t,i}=A_{t-1,i}, \forall i \in [N]$
        \ELSE 
            \STATE \textbf{Clients}: send $\{\textbf{X}_{t,i}^{\top}\textbf{X}_{t,i}\}_{i \in [N]}$ to server
            \STATE \textbf{Server} compute $A_{t}=\frac{\lambda}{c_{\mu}}\textbf{I}+\sum_{i=1}^{N}\textbf{X}_{t,i}^{\top}\textbf{X}_{t,i}$ and send $A_{t}$ to all clients. 
            \STATE \textbf{Clients}: set $A_{t,i}=A_{t}$, for $i \in [N]$
            \STATE \textbf{Server} update global model $\theta_{t}=\text{AGD-Update}(\theta_{t_\text{last}},J_{t})$, and set $t_\text{last}=t$
            \STATE \textbf{Clients} set local models $\theta_{t,i}=\theta_{t}, \forall i\in[N]$
        \ENDIF
    \ENDFOR
  \end{algorithmic}
\end{algorithm}

In \algone{}$_{1}$, each client stores a local model $\theta_{t-1,i}$, and the corresponding covariance matrix $A_{t-1,i}$. Note that $\{\theta_{t-1,i}, A_{t-1,i}\}_{i \in [N]}$ are only updated at time steps $t \in \cS$, and remain unchanged for $t \notin \cS$. At time $t$, client $i$ selects the arm that maximizes the following UCB score:
\begin{equation}\label{eq:UCB_1}
    \bx_{t,i}=\argmax_{ \bx \in \cA_{t,i}}{ \bx^{\top}\theta_{t-1,i}+\alpha_{t-1,i}||\bx||_{A_{t-1,i}^{-1}}}
\end{equation}
where $\alpha_{t-1,i}$ is given in Lemma \ref{lem:global_model_confidence_ellipsoid}.
The regret and communication cost of \algone{}$_{1}$ is given in the following theorem.
\begin{theorem}[Regret and Communication Cost Upper Bound of \algone{}$_{1}$]
Under the condition that $\epsilon_{t} = \frac{1}{N^{2}t^{2}}$, and the total number of global synchronizations $B=\sqrt{NT}$, the cumulative regret $R_{T}$ has upper bound
\begin{align*}
    R_{T}=O\left(\frac{k_{\mu}R_{\max}d}{c_{\mu}}\sqrt{NT}\log(NT/\delta)\right)
\end{align*}
with probability at least $1-\delta$. The cumulative communication cost has upper bound
\begin{align*}
    C_{T} =O(N^{2}T \log(NT))
\end{align*}
\end{theorem}
\begin{proof}
First, based on Lemma \ref{lem:global_model_confidence_ellipsoid} and under the condition that $\epsilon_{t} = \frac{1}{N^{2}t^{2}}$, we have
\begin{align*}
    & \lVert \theta_{t} - \theta_{\star} \rVert_{A_{t}} \leq \alpha_{t}
\end{align*}
holds $\forall t$, where $\alpha_{t}=\sqrt{\frac{2 k_{\mu} }{\lambda c_{\mu}} + \frac{2}{N t c_{\mu}}} + \frac{R_{max}}{c_{\mu}} \sqrt{d\log{(1+{N t c_{\mu}}/{(d \lambda)})}+2\log{({1}/{\delta})}} + \sqrt{\frac{\lambda}{c_{\mu}}}S  = O(\frac{R_{max}}{c_{\mu}}\sqrt{d \log(Nt)})$, which matches the order in \citep{li2017provably}.

Similar to the proof of Theorem \ref{thm:regret_comm_upper_bound}, we decompose all $B$ epochs into `good' and `bad' epochs according to the log-determinant ratio:
the $p$-th epoch, for $p \in [B]$, is a `good' epoch if the determinant ratio $\frac{\det(A_{t_{p}})}{\det(A_{t_{p-1}})} \leq 2$. Otherwise, it is a `bad' epoch. In the following, we bound the cumulative regret in `good' and `bad' epochs separately.

Suppose epoch $p$ is a good epoch, then for any client $i \in [N]$, and time step $t \in [t_{p-1}+1, t_{p}]$, we have $\frac{\det(\tilde{A}_{t-1,i})}{\det(A_{t-1,i})} \leq \frac{\det(A_{t_{p}})}{\det(A_{t_{p-1}})} \leq 2$, because $A_{t-1,i}=A_{t_{p-1}}$ and $\tilde{A}_{t-1,i} \preccurlyeq A_{t_{p}}$. Therefore, the instantaneous regret incurred by any client $i$ at any time step $t$ of a good epoch $p$ can be bounded by 
\begin{align*}
    r_{t,i} & \leq 2 k_{\mu}\alpha_{t_{p-1}} \sqrt{\bx_{t,i}^{\top} A_{t-1,i} \bx_{t,i}} \leq 2 k_{\mu}\alpha_{t_{p-1}} \sqrt{\bx_{t,i}^{\top} A_{t-1}^{-1} \bx_{t,i} \cdot \frac{\det(\tilde{A}_{t-1,i})}{\det(A_{t-1,i})}} \\
    & \leq 2\sqrt{2} k_{\mu} \alpha_{T} \sqrt{\bx_{t,i}^{\top} A_{t-1}^{-1} \bx_{t,i}}
\end{align*} 
By standard arguments \citep{abbasi2011improved,li2017provably}, the cumulative regret incurred in all good epochs can be upper bounded by $O(\frac{k_{\mu}R_{\max}}{c_{\mu}}d\sqrt{NT}\log(NT/\delta))$ with probability at least $1-\delta$.

By Assumption 1, $\mu(\cdot)$ is Lipschitz continuous with constant $k_{\mu}$, i.e.,
$|\mu(\bx^{\top} \theta_{1}) - \mu(\bx^{\top} \theta_{2})| \leq k_{\mu}|\bx^{\top} (\theta_{1} - \theta_{2})|$, so the instantaneous regret $r_{t,i}$ is uniformly bounded $\forall t \in [T], i \in [N]$ by $2 k_{\mu} S$.
Now suppose epoch $p$ is bad, then we can upper bound the cumulative regret in this bad epoch by $2 k_{\mu} S \frac{NT}{B}$, where $\frac{NT}{B}$ is the number of time steps in each epoch. Since there can be at most $O(d \log{NT})$ bad epochs, the cumulative regret incurred in all bad epochs can be upper bounded by $O(\frac{NT}{B} k_{\mu} S d \log(NT))$.
Combining both parts together, the cumulative regret upper bound is
\begin{align*}
    R_{T} = O\left(\frac{NT}{B} k_{\mu} S d \log(NT) + \frac{k_{\mu}R_{\max}d}{c_{\mu}}\sqrt{NT}\log(NT)\right)
\end{align*}
To recover the regret under centralized setting, we set $B=\sqrt{NT}$, so
\begin{align*}
    R_{T}=O\left(\frac{ k_{\mu} R_{max}}{c_{\mu}}  d\sqrt{NT}\log(NT)\right)
\end{align*}
Note that \algone{}$_{1}$ has $B=\sqrt{NT}$ global updates in total, and during each global update, there are $J_{t}$ rounds of communications, for $t\in\cS$. 
As mentioned earlier, for AGD to attain $\epsilon_{t}=\frac{1}{N^{2}t^{2}}$ sub-optimality, the required number of inner iterations
\begin{align*}
    J_{t} & \leq 1+\sqrt{\frac{k_{\mu}+\frac{\lambda}{N t}}{\frac{\lambda}{N t}}}\log{\frac{(k_{\mu}+\frac{\lambda}{N t}+\frac{\lambda}{Nt})\lVert \theta_{t}^{(0)}-\hat{\theta}_{t}^{\text{MLE}} \rVert_{2}^{2}}{2\epsilon_{t}}}  = O\left(\sqrt{Nt}\log(Nt)\right)
\end{align*}
Therefore, the communication cost over time horizon $T$ is
\small
\begin{align*}
    C_{T} & = N \cdot \sum_{t \in \cS} J_{t}  \\
    & = N\cdot \bigl[\sqrt{\sqrt{NT}} \log(\sqrt{NT}) + \sqrt{2\sqrt{NT}} \log(2\sqrt{NT}) + \dots + \sqrt{\sqrt{NT}\cdot \sqrt{NT}} \log(\sqrt{NT}\cdot \sqrt{NT})\bigr]  \\
    & \leq N^{5/4}T^{1/4} \log(NT) \bigl[ \sqrt{1} + \sqrt{2} + \dots + \sqrt{\sqrt{NT}} \bigr]  \\
    & \leq N^{5/4}T^{1/4} \log(NT) \cdot\frac{3}{2}(\sqrt{NT} + \frac{1}{2})^{3/2} \\
    & = O(N^{2}T \log(NT))
\end{align*}
\normalsize
which finishes the proof.
\end{proof}

\subsection{\algone{}$_{2}$: scheduled communication}
For the second variant \algone{}$_{2}$, we enabled local update on top of \algone{}$_{1}$. Therefore, compared with the original algorithm \algone{}, the only difference is that \algone{}$_{2}$ uses scheduled communication instead of event-triggered communication. Its description is given in Algorithm \ref{algo:fedglb_variant_2}.

\begin{algorithm}[h]
    \caption{\algone{}$_{2}$} \label{algo:fedglb_variant_2}
  \begin{algorithmic}[1]
    \STATE \textbf{Input:} communication schedule $\cS$, regularization parameter $\lambda>0$, $\delta \in (0,1)$ and $c_{\mu}$.
    \STATE \textbf{Initialize} $\forall i\in[N]$: ${A}_{0,i}=\frac{\lambda}{c_{\mu}} \textbf{I} \in \bR^{d \times d}, b_{0,i}=\textbf{0} \in \bR^{d} ,\theta_{0,i}=\textbf{0} \in \bR^{d}, \Delta {A}_{0,i}= \textbf{0} \in \bR^{d \times d}$;
    ${A}_{0}=\frac{\lambda}{c_{\mu}} \textbf{I} \in \bR^{d \times d}, b_{0}=\textbf{0} \in \bR^{d}, \theta_{0}= \textbf{0} \in \bR^{d}$, $t_{\text{last}}=0$
    \FOR{ $t=1,2,...,T$}
        \FOR{client $i=1,2,...,N$}
            \STATE Observe arm set $\mathcal{A}_{t,i}$ for client $i$
            \STATE Select arm $\bx_{t,i}\in\mathcal{A}_{t,i}$ by Eq.\eqref{eq:UCB}, and observe reward $y_{t,i}$
            \STATE Update client $i$: ${A}_{t,i} = {A}_{t-1,i} + \bx_{t,i} \bx_{t,i}^{\top}$, $\Delta {A}_{t,i} = \Delta {A}_{t-1,i} + \bx_{t,i} \bx_{t,i}^{\top}$
        \ENDFOR
        \IF{$t \notin \cS$}
            \STATE \textbf{Clients} $\forall i\in [N]$: $\theta_{t,i}=\text{ONS-Update}(\theta_{t-1,i},A_{t,i},\nabla l(\bx_{t,i}^{\top} \theta_{t-1,i},y_{t,i}))$, $b_{t,i}=b_{t-1,i}+ \bx_{t,i}\bx_{t,i}^{\top}\theta_{t-1,i}$
        \ELSE
            \STATE \textbf{Clients} $\forall i \in [N]$: send $\Delta A_{t,i}$ to server, and reset $\Delta A_{t,i}=\textbf{0}$
            \STATE \textbf{Server} compute $A_{t}=A_{t_\text{last}}+\sum_{i=1}^{N}\Delta A_{t,i}$
            \STATE \textbf{Server} perform global model update $\theta_{t}=\text{AGD-Update}(\theta_{t_\text{last}},J_{t})$ (see Eq.\eqref{eq:AGD_J_upper_bound} for choice of $J_{t}$), $b_{t}=b_{t_\text{last}}+\sum_{i=1}^{N}\Delta A_{t,i}\theta_{t}$, and set $t_\text{last}=t$
            \STATE \textbf{Clients} $\forall i \in [N]$: set $\theta_{t,i}=\theta_{t}, A_{t,i}=A_{t}, b_{t,i}=b_{t}$
        \ENDIF
    \ENDFOR
  \end{algorithmic}
\end{algorithm}

The regret and communication cost of \algone{}$_{2}$ is given in the following theorem.
\begin{theorem}[Regret and Communication Cost Upper Bound of \algone{}$_{2}$]
Under the condition that $\epsilon_{t} = \frac{1}{N^{2}t^{2}}$, and the total number of global synchronizations $B=d^{2}N\log(NT)$, the cumulative regret $R_{T}$ has upper bound
\begin{align*}
    R_{T} = O\left(\frac{k_{\mu}(k_{\mu}+R_{\max})}{c_{\mu}}d\sqrt{NT}\log(NT/\delta) \sqrt{\log\frac{T}{d^{2}N \log{NT}}}\right)
\end{align*}
with probability at least $1-\delta$. The cumulative communication cost has upper bound
\begin{align*}
    C_{T} = O(d^{2}N^{2.5} \sqrt{T} \log^{2}(NT))
\end{align*}
\end{theorem}
\begin{proof}
Compared with the analysis for \algone{}, the main difference in the analysis for \algone{}$_{2}$ is how we bound the regret incurred in `bad epochs'.
Using the same argument, the cumulative regret for the `good epochs' is $REG_{good}=O(\frac{k_{\mu}(k_{\mu}+R_{\max})}{c_{\mu}}d\sqrt{NT}\log{NT/\delta})$. 

Now consider a particular bad epoch $p \in [B]$. Then the cumulative regret incurred by all $N$ clients during this `bad epoch' can be upper bounded by:
\begin{align*}
    & \sum_{t=t_{p-1}+1}^{t_{p}} \sum_{i=1}^{N} r_{t,i} \\
    & \leq O(\frac{k_{\mu}(k_{\mu}+R_{\max})}{c_{\mu}}\sqrt{d\log(NT/\delta)})\sum_{t=t_{p-1}+1}^{t_{p}}\sum_{i=1}^{N}\min(1,||\bx_{t,i}||_{A_{t-1,i}^{-1}})  \\
    & \leq  O(\frac{k_{\mu}(k_{\mu}+R_{\max})}{c_{\mu}}\sqrt{d\log(NT/\delta)}) \sum_{i=1}^{N} \sqrt{(t_{p}-t_{p-1})\log{\frac{\det(A_{t_{p}-1,i})}{\det(A_{t_{p}-1,i}-\Delta{A}_{t_{p}-1,i})}}} \\
    & \leq O(\frac{k_{\mu}(k_{\mu}+R_{\max})}{c_{\mu}}d N\sqrt{\log{(NT/\delta)}} \sqrt{\frac{T}{B}\log(\frac{T}{B})})
\end{align*}
where the last inequality is because all epochs has length $\frac{T}{B}$ as defined by $\cS$.
Again, since there can be at most $O(d \log{NT})$ `bad epochs', the cumulative regret for the `bad epochs' is upper bounded by $$REG_{bad}=O(\frac{k_{\mu}(k_{\mu}+R_{\max})}{c_{\mu}}d^{2}\log^{1.5}{(NT/\delta)}N\sqrt{\frac{T}{B}\log(\frac{T}{B})}).$$

Combining the cumulative regret for both `good' and `bad' epochs, and setting $B=d^{2}N\log(NT)$, we have
\begin{align*}
    R_{T} = O\left(\frac{k_{\mu}(k_{\mu}+R_{\max})}{c_{\mu}}d\sqrt{NT}\log(NT/\delta) \sqrt{\log(\frac{T}{d^{2}N \log{NT}})}\right)
\end{align*}


Now that \algone{}$_{2}$ has $B=d^{2}N\log(NT)$ global updates in total, and during each global update, there are $J_{t}=O(\sqrt{NT}\log(NT))$ rounds of communications, for $t\in\cS$. Therefore, the communication cost over time horizon $T$ is
\begin{align*}
    C_{T} & = N \cdot \sum_{t \in \cS} J_{t}  = O( N \cdot d^{2}N\log(NT) \cdot \sqrt{NT}\log(NT))  \\
    & = O(d^{2}N^{2.5} \sqrt{T} \log^{2}(NT))
\end{align*}
which finishes the proof.
\end{proof}

\subsection{\algone{}$_{3}$: scheduled communication + ONS for global update}
The previous two variants both adopt iterative optimization method, i.e., AGD, for the global update, which introduces a $\sqrt{NT}\log(NT)$ factor in the communication cost. In this section, we try to avoid this by studying the third variant \algone{}$_{3}$ that adopts ONS for both local and global update, such that only one step of ONS is performed (based on all new data samples $N$ clients collected in this epoch). It can be viewed as the ONS-GLM algorithm \citep{jun2017scalable} with lazy batch update.
\begin{algorithm}[h]
    \caption{\algone{}$_{3}$}
  \begin{algorithmic}[1]
    \STATE \textbf{Input:} communication schedule $\cS$, regularization parameter $\lambda>0$, $\delta \in (0,1)$ and $c_{\mu}$
    \STATE \textbf{Initialize} $\forall i\in[N]$: $\theta_{0,i}=\textbf{0} \in \bR^{d}, {A}_{0,i}=\lambda \textbf{I} \in \bR^{d \times d}, V_{0,i}= \lambda \textbf{I} \in \bR^{d \times d}, b_{0,i}=\textbf{0} \in \bR^{d}$;
    $\theta_{0}= \textbf{0} \in \bR^{d}, A_{0}=\lambda \textbf{I} \in \bR^{d \times d}, V_{0}=\lambda \textbf{I} \in \bR^{d \times d}, b_{0}=\textbf{0} \in \bR^{d}$, $t_{\text{last}}=0$
    \FOR{ $t=1,2,...,T$}
        \FOR{client $i=1,2,...,N$}
            \STATE Observe arm set $\mathcal{A}_{t,i}$ for client $i \in [N]$
            \STATE Select arm $\bx_{t,i} = \argmax_{\bx \in \mathcal{A}_{t,i}} \bx^{\top} \hat{\theta}_{t-1,i} + \alpha_{t-1,i} \lVert \bx \rVert_{V_{t-1,i}^{-1}}$, where $\hat{\theta}_{t-1,i}=V_{t-1,i}^{-1}b_{t-1,i}$ and $\alpha_{t-1,i}$ is given in Lemma \ref{lem:confidence_ellipsoid_variant_3}; and then observe reward $y_{t,i}$
            \STATE Compute loss $l(z_{t,i},y_{t,i})$, where $z_{t,i}=\bx_{t,i}^{\top}\theta_{t-1,i}$
            \STATE Update client $i$: $A_{t,i} = {A}_{t-1,i} + \nabla l(z_{t,i},y_{t,i}) \nabla l(z_{t,i},y_{t,i})^{\top}$, $V_{t,i} = V_{t-1,i} + \bx_{t,i} \bx_{t,i}^{\top}$
        \ENDFOR
        \IF{$t \notin \cS$}
            \STATE \textbf{Clients} $\forall i \in [N]$: $\theta_{t,i}=\text{ONS-Update}(\theta_{t-1,i},A_{t,i},\nabla l(z_{t,i},y_{t,i}))$, $b_{t,i}=b_{t-1,i}+ \bx_{t,i} z_{t,i}$
        \ELSE
            \STATE \textbf{Clients} $\forall i \in [N]$: send gradient $\nabla F_{t,i}(\theta_{t_\text{last}})=\sum_{s=t_\text{last}+1}^{t} \nabla l(\bx_{s,i}^{\top} \theta_{t_\text{last}},y_{s,i})$ and $\Delta V_{t,i}=V_{t,i}-V_{t_\text{last},i}$ to server
            \STATE \textbf{Server} $A_{t}=A_{t_\text{last}} + (\sum_{i=1}^{N}\nabla F_{t,i}(\theta_{t_\text{last}})) (\sum_{i=1}^{N}\nabla F_{t,i}(\theta_{t_\text{last}}))^{\top}$, $V_{t}=V_{t_\text{last}}+\sum_{i=1}^{N}\Delta V_{t,i}$, $b_{t}=b_{t_\text{last}}+\sum_{i=1}^{N}\Delta V_{t,i} \theta_{t_\text{last}}$, $\theta_{t}=\text{ONS-Update}(\theta_{t_\text{last}},A_{t},\sum_{i=1}^{N}\nabla F_{t,i}(\theta_{t_\text{last}}) )$
            \STATE \textbf{Clients} $\forall i \in [N]$: $\theta_{t,i}=\theta_{t}, A_{t,i}=A_{t}, V_{t}=V_{t}, b_{t,i}=b_{t}$
        \STATE Set $t_{\text{last}}=t$
        \ENDIF
    \ENDFOR
  \end{algorithmic}
\end{algorithm}

Recall that the update schedule is denoted as $\cS=\{t_{1}:=\lfloor\frac{T}{B}\rfloor,t_{2}:=2\lfloor \frac{T}{B}\rfloor,\dots, t_{q}:=q\lfloor \frac{T}{B}\rfloor,\dots,t_{B}:=B\lfloor\frac{T}{B}\rfloor\}$,  where $B$ denotes the total number of global updates up to $T$.
Compared with \citep{jun2017scalable}, the main difference in our construction is that the loss function in the online regression problem may contain multiple data samples, i.e., for global update, or one single data sample, i.e., for local update.
Then for a client $i\in[N]$ at time step $t$ (suppose $t$ is in the $(q+1)$-th epoch, so $t \in [t_{q}+1, t_{q+1}]$), the sequence of loss functions observed by the online regression estimator till time $t$ is:
\scriptsize
\begin{align*}
    \underbrace{\sum_{s=1}^{t_{1}} \sum_{i=1}^{N} l(\bx_{s,i}^{\top}\theta_{0},y_{s,i}) , \sum_{s=t_{1}+1}^{t_{2}}\sum_{i=1}^{N} l(\bx_{s,i}^{\top}\theta_{t_{1}},y_{s,i}), \dots, \sum_{s=t_{q-1}+1}^{t_{q}}\sum_{i=1}^{N} l(\bx_{s,i}^{\top}\theta_{t_{q-1}},y_{s,i})}_{\text{global updates at $t_{1},t_{2},\dots,t_{q}$}}, \underbrace{l(\bx_{t_{q}+1,i}^{\top}\theta_{t_{q}},y_{t_{q}+1,i}), \dots, l(\bx_{t,i}^{\top} \theta_{t-1,i},y_{t,i})}_{\text{local updates at $t_{q}+1,\dots,t$}}
\end{align*}
\normalsize
We can see that the first $q$ terms correspond to the global ONS updates that are computed using the whole batch of data collected by $N$ clients in each epoch, and the remaining $t-t_{q}$ terms are local ONS updates that are computed using each new data sample collected by client $i$ in the $(q+1)$-th epoch.

To facilitate further analysis, we introduce a new set of indices for the data samples, so that we can unify the notations for the loss functions above. Imagine all the arm pulls are performed by an imaginary centralized agent, such that, in each time step $t\in[T]$, it pulls an arm for clients $1,2,\dots,N$ one by one. 
Therefore, the sequence of data sample obtained by this imaginary agent can be denoted as $(\bx_{1},y_{1}),(\bx_{2},y_{2}),\dots,(\bx_{s},y_{s}),\dots,(\bx_{NT},y_{NT})$.
Moreover, we denote $n_{p}$ as the total number of data samples collected by all $N$ clients till the $p$-th ONS update (including both global and local ONS update), and denote the updated model as $\theta_{p}$, for $p \in [P]$. Note that $P$ denotes the total number of updates up to time $t$ (total number of terms in the sequence above), such that $P=q+t-t_{q}$.
Then this sequence of loss functions can be rewritten as:
\begin{align*}
    \underbrace{F_{1}(\theta_{0}),F_{2}(\theta_{1}),\dots, F_{q}(\theta_{q-1})}_{\text{global updates}}, \underbrace{F_{q+1}(\theta_{q}),\dots, F_{P}(\theta_{P-1})}_{\text{local updates}}
\end{align*}
where $F_{p}(\theta_{p-1})=\sum_{s=n_{p-1}+1}^{n_{p}}l(\bx_{s}^{\top}\theta_{p-1},y_{s})$, for $p \in [P]$.

\noindent\textbf{$\bullet$ Online regret upper bound for lazily-updated ONS}
To construct the confidence ellipsoid based on this sequence of global and local ONS updates, we first need to upper bound the online regret that ONS incurs on this sequence of loss functions, which is given in Lemma \ref{lem:third_variant_online_regret}.
\begin{lemma}[Online regret upper bound] \label{lem:third_variant_online_regret}
Under the condition that the learning rate of ONS is set to $\gamma=\frac{1}{2} \min(\frac{1}{4S\sqrt{k_{\mu}^{2} S^{2}+R_{\max}^{2}}}, \frac{c_{\mu}}{(k_{\mu}^{2} S^{2}+R_{\max}^{2})\max_{p\in[P]}(n_{p}-n_{p-1})})$, then the cumulative online regret over $P$ steps
\begin{align*}
    & \sum_{p=1}^{P} F_{p}(\theta_{p-1}) - F_{p}(\theta_{\star}) \leq B_{P}
\end{align*}
where $B_{P}=\frac{1}{2 \gamma} \sum_{p=1}^{P}||\nabla F_{p}(\theta_{p-1})||_{A_{p}^{-1}}^{2} + 2 \gamma \lambda S^{2}$.

\end{lemma}

\begin{proof}[Proof of Lemma \ref{lem:third_variant_online_regret}]
Recall from the proof of Corollary \ref{corollary:order_ellipsoid} that $|\mu(\bx_{s}^{\top}\theta_{p-1})-y_{s}| \leq \sqrt{k_{\mu}^{2} S^{2}+R_{\max}^{2}}:=G,\forall s$.
First, we need to show that $F_{p}(\theta_{p-1})=\sum_{s=n_{p-1}+1}^{n_{p}}l(\bx_{s}^{\top}\theta_{p-1},y_{s})$ is $\frac{c_{\mu}}{(n_{p}-n_{p-1})G^{2}}$-exp-concave, or equivalently, $\nabla^{2} F_{p}(\theta_{p-1}) \succcurlyeq \frac{c_{\mu}}{(n_{p}-n_{p-1})G^{2}} \nabla F_{p}(\theta_{p-1}) \nabla F_{p}(\theta_{p-1})^{\top}$ (Lemma 4.2 in \citep{hazan2019introduction}). Taking first and second order derivative of $F_{p}(\theta_{p-1})$ w.r.t. $\theta_{p-1}$, we have
\begin{align*}
    & \nabla F_{p}(\theta_{p-1}) = \sum_{s=n_{p-1}+1}^{n_{p}} \bx_{s}[-y_{s}+\mu(\bx_{s}^{\top}\theta_{p-1})]=\bX_{p}^{\top}[\mu(\bX_{p}\theta_{p-1})-\by_{p}],  \\
    & \nabla^{2} F_{p}(\theta_{p-1}) = \sum_{s=n_{p-1}+1}^{n_{p}} \bx_{s} \bx_{s}^{\top} \dot{\mu}(\bx_{s}^{\top}\theta_{p-1}) 
\end{align*}
where $\bX_{p}=[\bx_{n_{p-1}+1}, \bx_{n_{p-1}+2},\dots,\bx_{n_{p}}]^{\top} \in \bR^{(n_{p}-n_{p-1})\times d}$, and $\by_{p}=[y_{n_{p-1}+1}, y_{n_{p-1}+2}, \dots, y_{n_{p}}]^{\top} \in \bR^{n_{p}-n_{p-1}}$.
Then due to Assumption 1, we have $\nabla^{2} F_{p}(\theta_{p-1}) \succcurlyeq c_{\mu} \sum_{s=n_{p-1}+1}^{n_{p}} \bx_{s} \bx_{s}^{\top}=c_{\mu} \bX_{p}^{\top}\bX_{p}$. 
For any vector $u\in \bR^{d}$, we can show that,
\begin{align*}
    & u^{\top}\nabla F_{p}(\theta_{p-1}) \nabla F_{p}(\theta_{p-1})^{\top} u \\
    &= u^{\top} \bX_{p}^{\top}[\mu(\bX_{p}\theta_{p-1})-\by_{p}] [\mu(\bX_{p}\theta_{p-1})-\by_{p}]^{\top}\bX_{p} u  \\
    & = \bigl[(\bX_{p}u)^{\top} [\mu(\bX_{p}\theta_{p-1})-\by_{p}] \bigr]^{2} \\
    &\leq \lVert \bX_{p}u \rVert_{2}^{2} \cdot \lVert \mu(\bX_{p}\theta_{p-1})-\by_{p} \rVert_{2}^{2} \\
    & \leq u^{\top} \bX_{p}^{\top}\bX_{p} u \cdot (n_{p}-n_{p-1})G^{2}
\end{align*}
where the first inequality is due to Cauchy-Schwarz inequality, and the second inequality is because $\lVert \mu(\bX_{p}\theta_{p-1})-\by_{p} \rVert_{2}^{2}=\sum_{s=n_{p-1}+1}^{n_{p}} [-y_{s}+\mu(\bx_{s}^{\top}\theta_{p-1})]^{2} \leq (n_{p}-n_{p-1})G^{2}$.
Therefore, $\bX_{p}^{\top}\bX_{p} \succcurlyeq \frac{1}{(n_{p}-n_{p-1})G^{2}} \nabla F_{p}(\theta_{p-1}) \nabla F_{p}(\theta_{p-1})^{\top}$, which gives us
\begin{align*}
    & \nabla^{2} F_{p}(\theta_{p-1}) \succcurlyeq \frac{c_{\mu}}{(n_{p}-n_{p-1})G^{2}} \nabla F_{p}(\theta_{p-1}) \nabla F_{p}(\theta_{p-1})^{\top} 
\end{align*}

Then due to Lemma 4.3 of \citep{hazan2019introduction}, under the condition that $\gamma_{p} \leq \frac{1}{2} \min(\frac{1}{4GS}, \frac{c_{\mu}}{(n_{p}-n_{p-1})G^{2}})$, we have

\begin{align} \label{eq:OCO_instantaneous_regret}
    F_{p}&(\theta_{p-1}) - F_{p}(\theta_{\star}) \nonumber \\
    & \leq \nabla F_{p}(\theta_{p-1})^{\top} (\theta_{p-1}-\theta_{\star}) - \frac{\gamma_{p}}{2}(\theta_{p-1}-\theta_{\star})^{\top} \nabla F_{p}(\theta_{p-1}) \nabla F_{p}(\theta_{p-1})^{\top} (\theta_{p-1}-\theta_{\star})
\end{align}
Then we start to upper bound the RHS of the inequality above. 
Recall that the ONS update rule is:
\begin{align*}
    & \theta_{p}^{\prime} = \theta_{p-1} - \frac{1}{\gamma} A_{p}^{-1} \nabla F_{p}(\theta_{p-1})  \\
    & \theta_{p} = \argmin_{\theta \in \Theta} ||\theta_{p}^{\prime} - \theta||^{2}_{A_{p}}
\end{align*}
where $A_{p}=\sum_{\rho=1}^{p} \nabla F_{\rho}(\theta_{\rho-1})\nabla F_{\rho}(\theta_{\rho-1})^{\top} $, and $\gamma$ is set to $\min_{p\in[P]}\gamma_{p}= \frac{1}{2} \min(\frac{1}{4GS}, \frac{c_{\mu}}{G^{2}\max_{p\in[P]}(n_{p}-n_{p-1})})$. So we have
\begin{align*}
    \theta_{p}^{\prime} - \theta_{\star} = \theta_{p-1} - \theta_{\star} - \frac{1}{\gamma} A_{p}^{-1} \nabla F_{p}(\theta_{p-1})
\end{align*}
Then due to the property of the generalized projection, and by substituting into the update rule, we have
\small
\begin{align*}
   & ||\theta_{p}-\theta_{\star}||^{2}_{A_{p}} \leq ||\theta_{p}^{\prime}-\theta_{\star}||^{2}_{A_{p}} \leq ||\theta_{p-1}-\theta_{\star}||_{A_{p}}^{2} - \frac{2}{\gamma} (\theta_{p-1}-\theta_{\star})^{\top}\nabla F_{p}(\theta_{p-1})  + \frac{1}{\gamma^{2}} ||\nabla F_{p}(\theta_{p-1})||_{A_{p}^{-1}}^{2}
\end{align*}
\normalsize
By rearranging terms, 
\begin{align*}
    & \nabla F_{p}(\theta_{p-1})^{\top} (\theta_{p-1}-\theta_{\star})  \leq \frac{1}{2 \gamma} ||\nabla F_{p}(\theta_{p-1})||_{A_{p}^{-1}}^{2} + \frac{\gamma}{2} \bigl( ||\theta_{p-1}-\theta_{\star}||_{A_{p}}^{2} -  ||\theta_{p}-\theta_{\star}||^{2}_{A_{p}} \bigr)
\end{align*}
After summing over $P$ steps, we have
\small
\begin{align*}
    & \sum_{p=1}^{P} \nabla F_{p}(\theta_{p-1})^{\top} (\theta_{p-1}-\theta_{\star}) \leq \frac{1}{2 \gamma} \sum_{p=1}^{P}||\nabla F_{p}(\theta_{p-1})||_{A_{p}^{-1}}^{2} + \frac{\gamma}{2}\sum_{p=1}^{P} \bigl( ||\theta_{p-1}-\theta_{\star}||_{A_{p}}^{2} -  ||\theta_{p}-\theta_{\star}||^{2}_{A_{p}} \bigr)
\end{align*}
\normalsize
The second term can be simplified,
\begin{align*}
    & \sum_{p=1}^{P} \bigl( ||\theta_{p-1}-\theta_{\star}||_{A_{p}}^{2} -  ||\theta_{p}-\theta_{\star}||^{2}_{A_{p}} \bigr) \\
    & = ||\theta_{0}-\theta_{\star}||_{A_{1}}^{2} + \sum_{p=2}^{P} \bigl( ||\theta_{p-1}-\theta_{\star}||_{A_{p}}^{2} -  ||\theta_{p-1}-\theta_{\star}||^{2}_{A_{p-1}} \bigr) - ||\theta_{P}-\theta_{\star}||_{A_{P}}^{2}  \\
    & \leq ||\theta_{0}-\theta_{\star}||_{A_{1}}^{2} + \sum_{p=2}^{P} \bigl( ||\theta_{p-1}-\theta_{\star}||_{A_{p}}^{2} -  ||\theta_{p-1}-\theta_{\star}||^{2}_{A_{p-1}} \bigr) \\
    & = ||\theta_{0}-\theta_{\star}||_{A_{1}}^{2} + \sum_{p=2}^{P}||\theta_{p-1}-\theta_{\star}||_{\nabla F_{p}(\theta_{p-1})\nabla F_{p}(\theta_{p-1})^{\top}}^{2} \\
    & = ||\theta_{0}-\theta_{\star}||_{A_{1}}^{2} + \sum_{p=1}^{P}||\theta_{p-1}-\theta_{\star}||_{\nabla F_{p}(\theta_{p-1})\nabla F_{p}(\theta_{p-1})^{\top}}^{2} - ||\theta_{0}-\theta_{\star}||_{\nabla F_{1}(\theta_{0})\nabla F_{1}(\theta_{0})^{\top}}^{2} \\
    & = 4 \lambda S^{2} + \sum_{p=1}^{P}||\theta_{p-1}-\theta_{\star}||_{\nabla F_{p}(\theta_{p-1})\nabla F_{p}(\theta_{p-1})^{\top}}^{2}
\end{align*}
which leads to
\begin{align*}
    \sum_{p=1}^{P} \nabla F_{p}(\theta_{p-1})^{\top} (\theta_{p-1}-\theta_{\star}) & \leq \frac{1}{2 \gamma} \sum_{p=1}^{P}||\nabla F_{p}(\theta_{p-1})||_{A_{p}^{-1}}^{2} + 2 \gamma \lambda S^{2} \\
    & \quad + \frac{\gamma}{2}\sum_{p=1}^{P}||\theta_{p-1}-\theta_{\star}||_{\nabla F_{p}(\theta_{p-1})\nabla F_{p}(\theta_{p-1})^{\top}}^{2}
\end{align*}
By rearranging terms, we have
\begin{align*}
    & \sum_{p=1}^{P} \bigl[ \nabla F_{p}(\theta_{p-1})^{\top} (\theta_{p-1}-\theta_{\star}) -\frac{\gamma}{2}||\theta_{p-1}-\theta_{\star}||_{\nabla F_{p}(\theta_{p-1})\nabla F_{p}(\theta_{p-1})^{\top}}^{2} \bigr] \\
    & \leq \frac{1}{2 \gamma} \sum_{p=1}^{P}||\nabla F_{p}(\theta_{p-1})||_{A_{p}^{-1}}^{2} + 2 \gamma \lambda S^{2}
\end{align*}
Combining with Eq.\eqref{eq:OCO_instantaneous_regret}, we obtain the following upper bound for the $P$-step online regret
\begin{align*}
    & \sum_{p=1}^{P} F_{p}(\theta_{p-1}) - F_{p}(\theta_{\star}) \leq \frac{1}{2 \gamma} \sum_{p=1}^{P}||\nabla F_{p}(\theta_{p-1})||_{A_{p}^{-1}}^{2} + 2 \gamma \lambda S^{2}
\end{align*}
where $A_{p}=\sum_{\rho=1}^{p} \nabla F_{\rho}(\theta_{\rho-1})\nabla F_{\rho}(\theta_{\rho-1})^{\top} $.
\end{proof}

\begin{corollary}[Order of $B_{P}$] \label{coro:BP_order}
Under the condition that $\gamma=\frac{1}{2} \min(\frac{1}{4S\sqrt{k_{\mu}^{2} S^{2}+R_{\max}^{2}}}, \frac{c_{\mu}}{(k_{\mu}^{2} S^{2}+R_{\max}^{2})\max_{p\in[P]}(n_{p}-n_{p-1})})$, the online regret upper bound $B_{P}=O(\frac{k_{\mu}^{2}+R_{max}^{2}}{c_{\mu}} d\log{(n_{P})} \max_{p \in [P]}(n_{p}-n_{p-1}))$.
\end{corollary}
\begin{proof}[Proof of Corollary \ref{coro:BP_order}]
Recall that $A_{p}=\sum_{\rho=1}^{p} \nabla F_{\rho}(\theta_{\rho-1})\nabla F_{\rho}(\theta_{\rho-1})^{\top} $. Therefore, we have
\begin{align*}
    \sum_{p=1}^{P} ||\nabla F_{p}(\theta_{p-1})||_{A_{p}^{-1}}^{2} &\leq \log\frac{\det(A_{P})}{\det(\lambda I)} = \log\frac{\det(\lambda I + \sum_{p=1}^{P}\nabla F_{p}(\theta_{p-1})\nabla F_{p}(\theta_{p-1})^{\top})}{\det(\lambda I)}  \\
    & \leq d \log{\bigl(1+ \frac{1}{d \lambda}\sum_{p=1}^{P}\lVert \nabla F_{p}(\theta_{p-1}) \rVert_{2}^{2} \bigr)}
\end{align*}
where the first inequality is due to Lemma 11 of \citep{abbasi2011improved}, and the second due to the determinant-trace inequality (Lemma 10 of \citep{abbasi2011improved}), i.e., $\det(\lambda I + \sum_{p=1}^{P}\nabla F_{p}(\theta_{p-1})\nabla F_{p}(\theta_{p-1})^{\top}) \leq \bigl( \frac{tr(\lambda I + \sum_{p=1}^{P}\nabla F_{p}(\theta_{p-1})\nabla F_{p}(\theta_{p-1})^{\top})}{d} \bigr)^{d} =\bigl( \frac{d \lambda + \sum_{p=1}^{P}\lVert \nabla F_{p}(\theta_{p-1}) \rVert_{2}^{2}}{d} \bigr)^{d}$.
Also note that $\nabla F_{p}(\theta_{p-1}) = \sum_{s=n_{p-1}+1}^{n_{p}} \bx_{s} \bigl[ \mu(\bx_{s}^{\top}\theta_{p-1})-y_{s} \bigr]$, so we have
\begin{align*}
     \sum_{p=1}^{P} ||\nabla F_{p}(\theta_{p-1})||_{2}^{2} &=  \sum_{p=1}^{P} ||\sum_{s=n_{p-1}+1}^{n_{p}} \bx_{s} \bigl[ \mu(\bx_{s}^{\top}\theta_{p-1})-y_{s} \bigr]||_{2}^{2} \\
    & \leq G^{2} \sum_{p=1}^{P} ||\sum_{s=n_{p-1}+1}^{n_{p}} \bx_{s} ||_{2}^{2} \leq G^{2} \sum_{p=1}^{P} (n_{p}-n_{p-1})^{2} \leq G^{2} n_{P}^{2}
\end{align*}
where the second inequality is due to Jensen's inequality and the assumption that $\lVert \bx_{s} \rVert \leq 1,\forall s$.
Substituting this back gives us
\begin{align*}
    & \sum_{p=1}^{P} F_{p}(\theta_{p-1}) - F_{p}(\theta_{\star}) \leq \frac{1}{2 \gamma} d \log{\bigl(1+ \frac{1}{d \lambda}G^{2} n_{P}^{2} \bigr)} + 2 \gamma \lambda S^{2}  \\
    & = \frac{(k_{\mu}^{2} S^{2}+R_{\max}^{2})\max_{p\in[P]}(n_{p}-n_{p-1})}{c_{\mu}} d \log{\bigl(1+ \frac{1}{d \lambda} (k_{\mu}^{2} S^{2}+R_{\max}^{2}) n_{P}^{2} \bigr)} \\
    & \quad + \frac{c_{\mu}}{(k_{\mu}^{2} S^{2}+R_{\max}^{2})\max_{p\in[P]}(n_{p}-n_{p-1})} \lambda S^{2}
\end{align*}
where the equality is because $\max_{p\in[P]}(n_{p}-n_{p-1})$ dominates $\gamma=\frac{1}{2} \min(\frac{1}{4GS}, \frac{c_{\mu}}{G^{2}\max_{p\in[P]}(n_{p}-n_{p-1})})$.
\end{proof}

\noindent\textbf{$\bullet$ Construct Confidence Ellipsoid for \algone{}$_{3}$}
With the online regret bound $B_{P}$ in Lemma \ref{lem:third_variant_online_regret}, the steps to construct the confidence ellipsoid largely follows that of Theorem 1 in \citep{jun2017scalable}, with the main difference in our batch update. We include the full proof here for the sake of completeness. 

\begin{lemma}[Confidence Ellipsoid for \algone{}$_{3}$] \label{lem:confidence_ellipsoid_variant_3}
Under the condition that the learning rate of ONS $\gamma=\frac{1}{2} \min(\frac{1}{4S\sqrt{k_{\mu}^{2} S^{2}+R_{\max}^{2}}}, \frac{c_{\mu}}{(k_{\mu}^{2} S^{2}+R_{\max}^{2})\max_{p\in[P]}(n_{p}-n_{p-1})})$, we have $\forall t \in [T], i \in [N]$
\begin{align*}
    \lVert \theta_{\star} - \hat{\theta}_{t,i} \rVert_{V_{t,i}}^{2} & \leq \lambda S^{2} + 1+\frac{4}{c_{\mu}} B_{P} + \frac{8 R_{max}^{2}}{c_{\mu}^{2}} \log{(\frac{N}{\delta}\sqrt{4+\frac{8}{c_{\mu}}B_{P}+\frac{64 R_{max}^{2}}{c_{\mu}^{4} \cdot 4 \delta^{2}}})}\} \\
    & \quad - \hat{\theta}_{t,i}^{\top} b_{t,i} - \sum_{s=1}^{n_{P}} z_{s}^{2}:=\alpha_{t,i}^{2}
\end{align*}
with probability at least $1-\delta$.
\end{lemma}
\begin{proof}[Proof of Lemma \ref{lem:confidence_ellipsoid_variant_3}]
Due to $c_{\mu}$-strongly convexity of $l(z,y)$ w.r.t. $z$, we have $l(\bx_{s}^{\top}\theta_{p-1},y_{s}) - l(\bx_{s}^{\top}\theta_{\star},y_{s}) \geq \bigl[ \mu(\bx_{s}^{\top}\theta_{\star})-y_{s} \bigr]\bx_{s}^{\top}(\theta_{p-1}-\theta_{\star}) + \frac{c_{\mu}}{2}
\bigl[\bx_{s}^{\top}(\theta_{p-1}-\theta_{\star})\bigr]^{2}$. Therefore, 
\begin{align*}
     F_{p}(\theta_{p-1}) - F_{p}(\theta_{\star}) &= \sum_{s=n_{p-1}+1}^{n_{p}} l(\bx_{s}^{\top}\theta_{p-1},y_{s}) - l(\bx_{s}^{\top}\theta_{\star},y_{s}) \\ 
    & \geq \sum_{s=n_{p-1}+1}^{n_{p}} \bigl[ \mu(\bx_{s}^{\top}\theta_{\star})-y_{s} \bigr]\bx_{s}^{\top}(\theta_{p-1}-\theta_{\star}) + \frac{c_{\mu}}{2} \sum_{s=n_{p-1}+1}^{n_{p}}\bigl[\bx_{s}^{\top}(\theta_{p-1}-\theta_{\star})\bigr]^{2} \\
    & = -\sum_{s=n_{p-1}+1}^{n_{p}} \eta_{s} \bx_{s}^{\top}(\theta_{p-1}-\theta_{\star}) +\frac{c_{\mu}}{2} \sum_{s=n_{p-1}+1}^{n_{p}}\bigl[\bx_{s}^{\top}(\theta_{p-1}-\theta_{\star})\bigr]^{2} 
\end{align*}
where $\eta_{s}$ is the $R$-sub-Gaussian noise in the reward $y_{s}$.
Summing over $P$ steps we have
\small
\begin{align*}
    B_{P} & \geq  \sum_{p=1}^{P} F_{p}(\theta_{p-1}) - F_{p}(\theta_{\star}) \geq \sum_{p=1}^{P} \sum_{s=n_{p-1}+1}^{n_{p}} \eta_{s} \bx_{s}^{\top}(\theta_{p-1}-\theta_{\star}) + \frac{c_{\mu}}{2} \sum_{p=1}^{P} \sum_{s=n_{p-1}+1}^{n_{p}}\bigl[\bx_{s}^{\top}(\theta_{p-1}-\theta_{\star})\bigr]^{2} 
\end{align*}
\normalsize
By rearranging terms, we have
\begin{align*}
    & \sum_{p=1}^{P} \sum_{s=n_{p-1}+1}^{n_{p}}\bigl[\bx_{s}^{\top}(\theta_{p-1}-\theta_{\star})\bigr]^{2} \leq \frac{2}{c_{\mu}} \sum_{p=1}^{P} \sum_{s=n_{p-1}+1}^{n_{p}} \eta_{s} \bx_{s}^{\top}(\theta_{p-1}-\theta_{\star}) + \frac{2}{c_{\mu}}B_{P}
\end{align*}
Then as $\bx_{s}^{\top}(\theta_{p-1}-\theta_{\star})$ for $s \in [n_{p-1}+1,n_{p}]$ is $\cF_{s}$-measurable for lazily updated online estimator $\theta_{p-1}$, we can use Corollary 8 from \citep{abbasi2012online}, which leads to
\small
\begin{align*}
    & \sum_{p=1}^{P} \sum_{s=n_{p-1}+1}^{n_{p}} \eta_{s} \bx_{s}^{\top}(\theta_{p-1}-\theta_{\star}) \leq \\
    & \quad R_{max}\sqrt{\bigl( 2+2\sum_{p=1}^{P} \sum_{s=n_{p-1}+1}^{n_{p}} (\bx_{s}^{\top}(\theta_{p-1}-\theta_{\star}))^{2} \bigr) \cdot \log\bigl( \frac{1}{\delta} \sqrt{1+\sum_{p=1}^{P} \sum_{s=n_{p-1}+1}^{n_{p}} (\bx_{s}^{\top}(\theta_{p-1}-\theta_{\star}))^{2}} \bigr)}
\end{align*}
\normalsize
Then we have 
\small
\begin{align*}
    &\sum_{p=1}^{P} \sum_{s=n_{p-1}+1}^{n_{p}}\bigl[\bx_{s}^{\top}(\theta_{p-1}-\theta_{\star})\bigr]^{2}  \leq \frac{2}{c_{\mu}}B_{P} \\
    & + \frac{2R_{max}}{c_{\mu}}\sqrt{\bigl( 2+2\sum_{p=1}^{P} \sum_{s=n_{p-1}+1}^{n_{p}} (\bx_{s}^{\top}(\theta_{p-1}-\theta_{\star}))^{2} \bigr) \cdot \log\bigl( \frac{1}{\delta} \sqrt{1+\sum_{p=1}^{P} \sum_{s=n_{p-1}+1}^{n_{p}} (\bx_{s}^{\top}(\theta_{p-1}-\theta_{\star}))^{2}} \bigr)}
\end{align*}
\normalsize
Then by applying Lemma 2 from \citep{jun2017scalable}, we have
\begin{align*}
    \sum_{p=1}^{P} \sum_{s=n_{p-1}+1}^{n_{p}}\bigl[\bx_{s}^{\top}(\theta_{p-1}-\theta_{\star})\bigr]^{2} \leq 1+\frac{4}{c_{\mu}} B_{P} + \frac{8 R_{max}^{2}}{c_{\mu}^{2}} \log{(\frac{1}{\delta}\sqrt{4+\frac{8}{c_{\mu}}B_{P}+\frac{64 R_{max}^{2}}{c_{\mu}^{4} \cdot 4 \delta^{2}}})}
\end{align*}
Therefore, we have the following confidence ellipsoid (regularized with parameter $\lambda$):
\small
\begin{equation*}
    \{\theta : \sum_{p=1}^{P} \sum_{s=n_{p-1}+1}^{n_{p}}\bigl[\bx_{s}^{\top}(\theta_{p-1}-\theta_{\star})\bigr]^{2} +\lambda \lVert \theta \rVert_{2}^{2} \leq \lambda S^{2} + 1+\frac{4}{c_{\mu}} B_{P} + \frac{8 R_{max}^{2}}{c_{\mu}^{2}} \log{(\frac{1}{\delta}\sqrt{4+\frac{8}{c_{\mu}}B_{P}+\frac{64 R_{max}^{2}}{c_{\mu}^{4} \cdot 4 \delta^{2}}})}\}
\end{equation*}
\normalsize
And this can be rewritten as a ellipsoid centered at ridge regression estimator $\hat{\theta}_{t,i}=V_{t,i}^{-1}b_{t,i}$, where $V_{t,i}=\lambda I + \sum_{p=1}^{P} \sum_{s=n_{p-1}+1}^{n_{p}}\bx_{s}\bx_{s}^{\top}$ and $b_{t,i}=\sum_{p=1}^{P} \sum_{s=n_{p-1}+1}^{n_{p}} \bx_{s} z_{s}$ (recall that ONS's prediction at time $s$ is denoted as $z_{s}=\bx_{s}^{\top} \theta_{p-1}$), i.e., $\forall t \in [T]$
\begin{equation*}
    \lVert \theta_{\star} - \hat{\theta}_{t,i} \rVert_{V_{t,i}}^{2} \leq \lambda S^{2} + 1+\frac{4}{c_{\mu}} B_{P} + \frac{8 R_{max}^{2}}{c_{\mu}^{2}} \log{(\frac{1}{\delta}\sqrt{4+\frac{8}{c_{\mu}}B_{P}+\frac{64 R_{max}^{2}}{c_{\mu}^{4} \cdot 4 \delta^{2}}})}\} + \hat{\theta}_{t,i}^{\top} b_{t,i} - \sum_{s=1}^{n_{P}} z_{s}^{2}
\end{equation*}
with probability at least $1-\delta$. Then taking union bound over all $N$ clients, we have, $\forall t \in [T], i \in [N]$
\begin{align*}
    \lVert \theta_{\star} - \hat{\theta}_{t,i} \rVert_{V_{t,i}}^{2} \leq \lambda S^{2} + 1+\frac{4}{c_{\mu}} B_{P} + \frac{8 R_{max}^{2}}{c_{\mu}^{2}} \log{(\frac{N}{\delta}\sqrt{4+\frac{8}{c_{\mu}}B_{P}+\frac{64 R_{max}^{2}}{c_{\mu}^{4} \cdot 4 \delta^{2}}})}\} + \hat{\theta}_{t,i}^{\top} b_{t,i} - \sum_{s=1}^{n_{P}} z_{s}^{2}
\end{align*}
with probability at least $1-\delta$.
\end{proof}

\noindent\textbf{$\bullet$ Regret and Communication Upper Bounds for \algone{}$_{3}$}

The regret and communication cost of \algone{}$_{3}$ is given in the following theorem.
\begin{theorem}[Regret and Communication Cost Upper Bound of \algone{}$_{3}$]
Under the condition that the learning rate of ONS $\gamma=\frac{1}{2} \min(\frac{1}{4S\sqrt{k_{\mu}^{2} S^{2}+R_{\max}^{2}}}, \frac{c_{\mu}}{(k_{\mu}^{2} S^{2}+R_{\max}^{2})\sqrt{NT}})$, and the total number of global synchronizations $B=\sqrt{NT}$, the cumulative regret $R_{T}$ has upper bound
\begin{align*}
    R_{T} = O\left(\frac{k_{\mu}(k_{\mu}+R_{max})}{c_{\mu}}d N^{3/4}T^{3/4}\log(NT/\delta)\right)
\end{align*}
with probability at least $1-\delta$. The cumulative communication cost has upper bound
\begin{align*}
    C_{T} = O(N^{1.5} \sqrt{T})
\end{align*}
\end{theorem}

\begin{proof}
Similar to the proof for the previous two variants of \algone{}, we divide the epochs into `good' and `bad' ones according to the determinant ratio, and then bound their cumulative regret separately.

Recall that the instantaneous regret $r_{t,i}$ incurred by client $i\in[N]$ at time step $t\in[T]$ has upper bound
\begin{align*}
    \frac{r_{t,i}}{k_{\mu}} & \leq \bx_{t,\star}^{\top} \theta_{\star}- \bx_{t,i}^{\top} \theta_{\star} \leq \bx_{t,i}^{\top} \tilde{\theta}_{i,t}- \bx_{t,i}^{\top} \theta_{\star} \\
    & = \bx_{t,i}^{\top} (\tilde{\theta}_{i,t} - \hat{\theta}_{t,i}) + \bx_{t,i}^{\top} (\hat{\theta}_{t,i} - \theta_{\star})  \\
    & \leq \lVert \bx_{t,i} \rVert_{V_{t,i}^{-1}} \lVert \tilde{\theta}_{i,t} - \hat{\theta}_{t,i} \rVert_{V_{t,i}} + \lVert \bx_{t,i} \rVert_{V_{t,i}^{-1}} \lVert \hat{\theta}_{t,i} - \theta_{\star} \rVert_{V_{t,i}} \\
    & \leq 2 \alpha_{t,i} \lVert \bx_{t,i} \rVert_{V_{t,i}^{-1}}
\end{align*}
Note that due to the update schedule $\cS$, we have $\max_{p \in [P]}(n_{p}-n_{p-1}) = \frac{NT}{B}$. Then based on Corollary \ref{coro:BP_order}, $\alpha_{t,i}= O(\frac{k_{\mu}+R_{max}}{c_{\mu}}\sqrt{d \log(NT)} \sqrt{\frac{NT}{B}})$, so we have, $\forall t \in [T],i \in [N]$,
\begin{equation*}
    r_{t,i} = O(\frac{k_{\mu}(k_{\mu}+R_{max})}{c_{\mu}}\sqrt{d \log(NT)}\sqrt{\frac{NT}{B}}) \lVert \bx_{t,i} \rVert_{A_{t,i}^{-1}}
\end{equation*}
with probability at least $1-\delta$.

Therefore, the cumulative regret for the `good epochs' is $REG_{good}=O(\frac{k_{\mu}(k_{\mu}+R_{max})}{c_{\mu}}d \frac{NT}{\sqrt{B}}\log(NT))$. 

Using the same argument as in the proof for \algone{}$_{1}$, the cumulative regret for each `bad ' epoch is upper bounded by $2 k_{\mu} S \frac{NT}{B}$.
Since there can be at most $O(d \log{NT})$ `bad epochs', the cumulative regret for all the `bad epochs' is upper bounded by
\begin{align*}
    REG_{bad} = O(d NT\log(NT) \cdot \frac{k_{\mu}S}{B}) 
\end{align*}
Combining the regret incurred in both `good' and `bad' epochs, we have
\begin{align*}
    R_{T} = O\big( \frac{k_{\mu}(k_{\mu}+R_{max})}{c_{\mu}}d \frac{NT}{\sqrt{B}}\log(NT) + d NT\log(NT) \cdot \frac{k_{\mu}S}{B} \big)
\end{align*}
To recover the regret in centralized setting, we can $B=NT$, which leads to $R_{T}=O(\frac{k_{\mu}(k_{\mu}+R_{max})}{c_{\mu}}d \sqrt{NT}\log(NT))$. However, this incurs communication cost $C_{T}=N^{2}T$.
Alternatively, if we set $B=\sqrt{NT}$, we have $R_{T} = O(\frac{k_{\mu}(k_{\mu}+R_{max})}{c_{\mu}}d N^{3/4}T^{3/4}\log(NT))$, and $ C_{T} = O(N^{1.5} \sqrt{T})$.
\end{proof}

\section{Additional Explanation about Figure \ref{fig:real_exp_results}} \label{sec:explain_scatter_plot}
In Section \ref{sec:exp}, we used the scatter plots to present the experiment results. Here we provide more explanation about how to interpret these figures.
As mentioned earlier, each dot in Figure \ref{fig:real_exp_results} denotes the cumulative communication cost (x-axis) and regret (y-axis) that an algorithm (\algone{}, its variants, or DisLinUCB) with certain threshold value of $D$ or $B$ (labeled next to the dot) has obtained at iteration $T$.

Here, Figure \ref{fig:linePlot} shows how the cumulative regret/reward and communication cost of five algorithms change over the course of federated bandit learning in our evaluations on synthetic dataset, (their final results at iteration $T$ are used to plot five dots in Figure \ref{fig:a}). By carefully examining the relationship between their regret and communication cost, we can see that in Figure \ref{fig:linePlot}, \algone{} ($D=5.0$), \algone{}$_{1}$ ($B=10.0$), \algone{}$_{2}$ ($B=10.0$), and \algone{}$_{3}$ ($B=5000.0$) incur similar total communication cost, but \algone{} ($D=5.0$) attains much smaller regret than the others. Meanwhile, \algone{} ($D=5000.0$) attains almost the same regret as \algone{}$_{2}$ ($B=10.0$), but its communication cost is much lower. 

\begin{figure*}[h!]
\centering     
\includegraphics[width=0.95\textwidth]{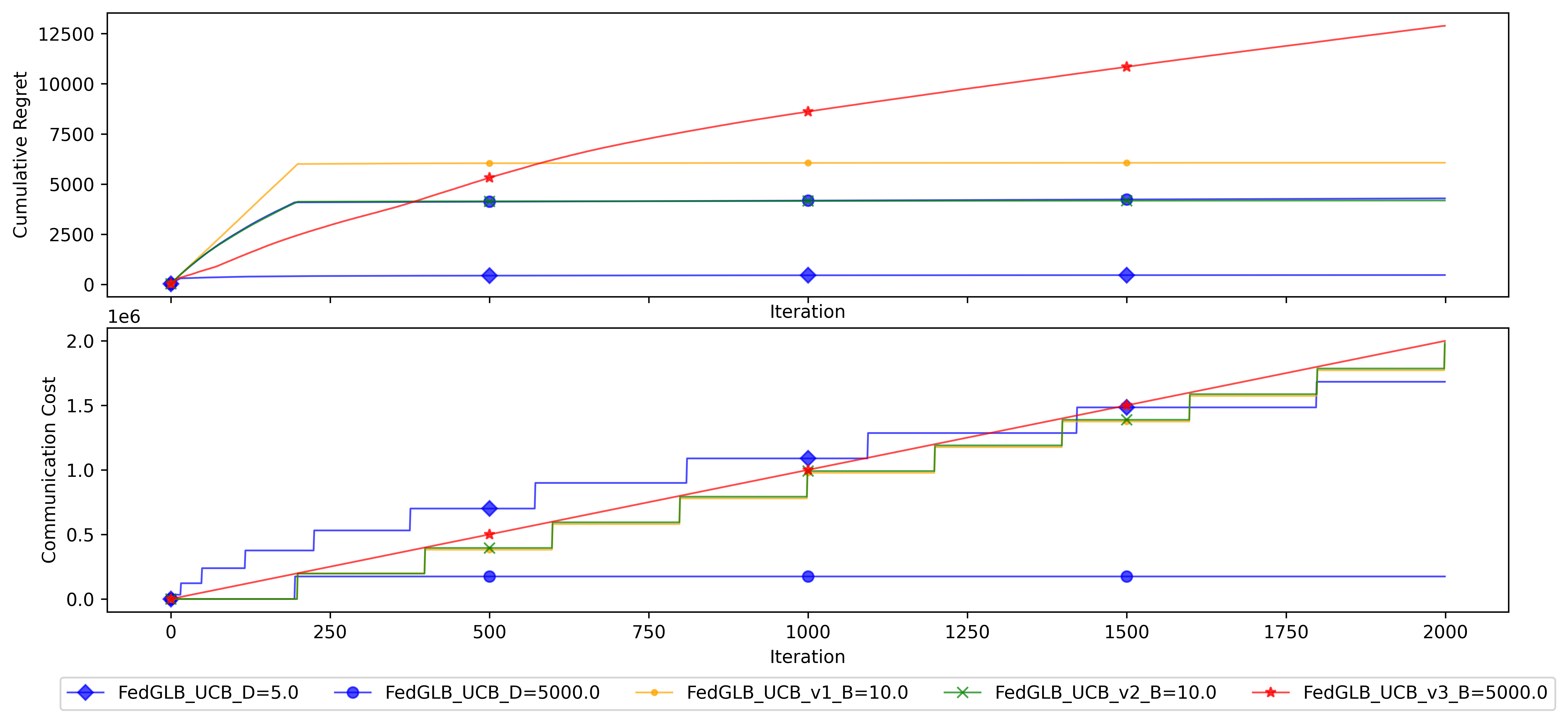}
\caption{Experiment results showing regret and communication cost over time.}
\label{fig:linePlot}
\end{figure*}

Figure \ref{fig:linePlot} also depicts how the communication was controlled in \algone{} under its event triggered protocol (e.g., generally a decreasing frequency of communication comparing to the scheduled updated in its variants). This shows that \algone{} strikes the best regret/reward-communication trade-off among the algorithm instances in comparison. However, this line chart can only accommodate a limited range of trade-off settings for these algorithms, to attain a reasonable visibility.
In comparison, the scatter plots in Figure \ref{fig:a} provide a much more thorough view of how well the algorithms balance regret/reward and communication cost, by covering a large range of trade-off settings.
